\documentclass[lettersize,journal]{IEEEtran}
\usepackage{amsmath,amsfonts}
\usepackage{algorithmic}
\usepackage{algorithm}
\usepackage{array}
\usepackage[caption=false,font=normalsize,labelfont=sf,textfont=sf]{subfig}
\usepackage{textcomp}
\usepackage{stfloats}
\usepackage{url}
\usepackage{verbatim}
\usepackage{graphicx}
\usepackage{cite}
\usepackage{amssymb}

\usepackage{amsthm}  

\theoremstyle{plain}
\newtheorem{theorem}{Theorem}

\newtheorem{proposition}{Proposition}
\newtheorem{lemma}{Lemma}
 \newtheorem{corollary}{Corollary}

\theoremstyle{definition}
\newtheorem{definition}{Definition} 

\newtheorem{assumption}{Assumption} 

\newtheorem{observation}{Observation}

\usepackage{booktabs}

\newtheorem{question}{Question}

\usepackage{etoolbox}
\appto\appendix{\numberwithin{theorem}{section}}

\begin{document}

\title{An Information-Theoretic Analysis for Federated Learning under Concept Drift}

\author{Fu Peng, Meng Zhang, and Ming Tang

\thanks{Fu Peng and Ming Tang are with the Department of Computer Science and Engineering, Southern University of Science and Technology, Shenzhen 518055, China (e-mail: pengf2022@mail.sustech.edu.cn; tangm3@sustech.edu.cn).}
\thanks{
Meng Zhang is with Zhejiang University–University of Illinois Urbana–Champaign Institute, Zhejiang University, Haining 314400, China (e-mail: mengzhang@intl.zju.edu.cn).}
}




\maketitle

\begin{abstract}
Recent studies in federated learning (FL) commonly train models on static datasets. However, real-world data often arrives as streams with shifting distributions, causing performance degradation known as concept drift. This paper analyzes FL performance under concept drift using information theory and proposes an algorithm to mitigate the performance degradation. 
We model concept drift as a Markov chain and introduce the \emph{Stationary Generalization Error} to assess a model's capability to capture characteristics of future unseen data. Its upper bound is derived using KL divergence and mutual information. We study three drift patterns (periodic, gradual, and random) and their impact on FL performance.
Inspired by this, we propose an algorithm that regularizes the empirical risk minimization approach with KL divergence and mutual information, thereby enhancing long-term performance. We also explore the performance-cost tradeoff by identifying a Pareto front.
To validate our approach, we build an FL testbed using Raspberry Pi4 devices. Experimental results corroborate with theoretical findings, confirming that drift patterns significantly affect performance. Our method consistently outperforms existing approaches for these three patterns, demonstrating its effectiveness in adapting concept drift in FL.
\end{abstract}

\begin{IEEEkeywords}
Federated learning, concept drift, information theory.
\end{IEEEkeywords}

\section{Introduction}
FL is a distributed machine learning approach that enables many clients to collaboratively train a neural network model~\cite{DBLP:journals/ftml/KairouzMABBBBCC21}. Recent FL research primarily focuses on the design of incentive mechanisms~\cite{10.1145/3565287.3610269}, improving FL efficiency~\cite{10.1145/3565287.3610264}, utilization of computing sources~\cite{10.1145/3565287.3610273}, etc. In these methods, clients perform training on static datasets to obtain the global model and achieve better utilization of computing and communication resources~\cite{DBLP:journals/ftml/KairouzMABBBBCC21}.

\begin{figure}[t]
\centering
\includegraphics[width=1\linewidth]{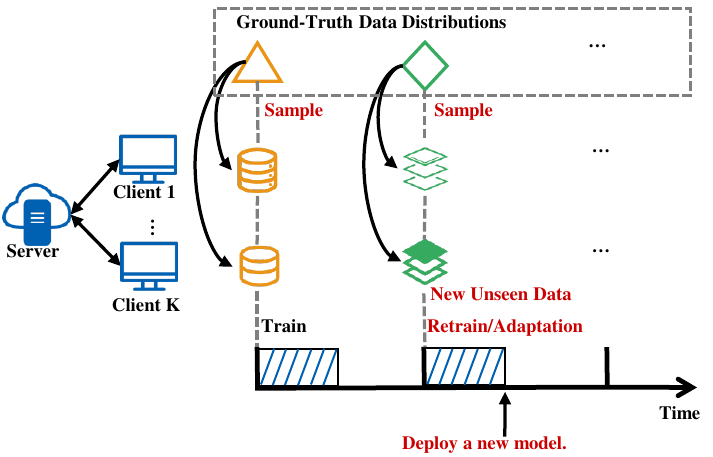}
\caption{An illustration of an FL system under concept drift.}
\label{fig.1}
\end{figure}

\begin{figure}[t]
\centering
\includegraphics[width=1\linewidth]{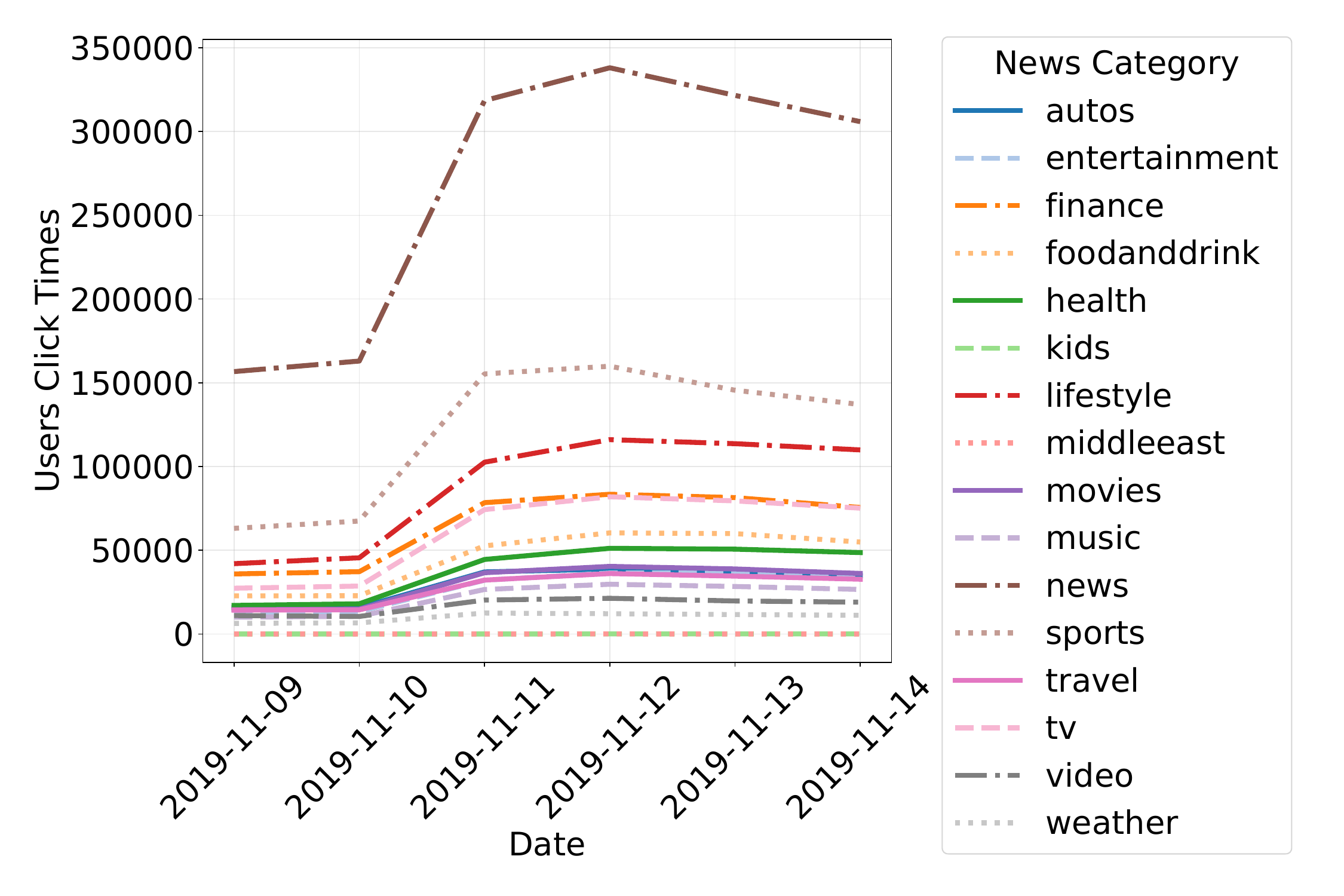}
\caption{Users' interest in news topics changes over time.}
\label{fig:MIND Dataset}
\end{figure}

In practical FL systems, data is typically collected as data streams, whose statistical properties evolve over time. Fig.\ref{fig.1} illustrates an FL system under changing data distributions. For example, in real-world news recommendation scenarios such as the MIND dataset\cite{DBLP:conf/acl/WuQCWQLLXGWZ20}, users’ interests in news topics shift over time (see Fig.~\ref{fig:MIND Dataset}). As user interests change, the characteristics of their behavior data also change, leading to a time-varying local data distribution. This drift requires models to be retrained or adapted to maintain performance on new and unseen data. 
This phenomenon is known as concept drift~\cite{DBLP:journals/tkde/LuLDGGZ19}. Generally, concept drift can typically be classified into three patterns: periodic, gradual, and random~\cite{DBLP:journals/tkde/LuLDGGZ19,gama2014survey}. In periodic pattern, the changes of data distribution exhibit a cyclical nature, e.g., the time-varying traffic flow throughout the day. Gradual pattern characterizes slow shifts in data over time, such as changing trends in popular food or movies. Random pattern characterizes changes without any clear pattern, such as black swan events in the stock market.

Concept drift has been widely studied in centralized machine learning. Early methods retrain models using recent data in a sliding window~\cite{DBLP:conf/icdm/BachM08,DBLP:conf/sdm/BifetG07}, often discarding historical information. Others adapt decision trees via windowed updates or node pruning~\cite{DBLP:conf/kdd/HultenSD01,DBLP:conf/kdd/GamaRM03}, but these are limited to specific models. Neural network-based approaches~\cite{DBLP:journals/ijon/XuW17} adjust model capacity by adding hidden nodes, which is computationally costly in FL. However, these centralized methods are not directly applicable to FL due to its decentralized architecture, limited communication bandwidth, and strict privacy constraints.

Recent studies have addressed concept drift in FL with various strategies. Yoon~\textit{et al.}\cite{pmlr-v139-yoon21b} and Yu~\textit{et al.}\cite{DBLP:journals/tmc/YuCZCZXC23} focused on enhancing local model adaptability, while neglecting global model robustness. Jothimurugesan~\textit{et al.}\cite{DBLP:conf/aistats/JothimurugesanH23} formulated drift as a clustering problem but incurred high client-side computation. Canonaco~\textit{et al.}\cite{DBLP:conf/ijcnn/CanonacoBMR21} and Panchal~\textit{et al.}~\cite{DBLP:conf/icml/PanchalCMMSMG23} proposed adaptive optimizers for faster adaptation, yet overlooked performance degradation. However, none of these works offer a theoretical foundation for analyzing drift in FL.

To bridge this gap, we adopt an information-theoretic framework to analyze concept drift in FL. Compared to classical tools such as VC dimension~\cite{DBLP:books/daglib/0033642} and PAC-Bayesian bounds~\cite{DBLP:journals/ml/McAllester99}, information-theoretic analysis offers a more flexible and insightful way to characterize the generalization error under dynamic and distributed settings. Bu~\textit{etal.}\cite{DBLP:journals/jsait/BuZV20} derived an information-theoretic upper bound on the generalization error of supervised learning algorithms, while Wu~\textit{etal.}\cite{DBLP:conf/isit/WuMAZ20} extended such analysis to transfer learning. In this work, we build upon these results and extend them to FL under concept drift.

Based on this theoretical foundation, it is crucial to quantitatively evaluate the performance of an FL system under concept drift, with a focus on deriving performance bounds using information theory. The performance bound reveals how factors such as drift patterns and the magnitude of distributional changes influence the generalization capability of the system. Moreover, it provides insights into how to design algorithms that mitigate performance degradation. In practice, however, improving performance often comes with increased communication costs, computational overhead, or training efforts. Therefore, it is also important to understand the tradeoffs between system performance and cost. In this work, our objective is to address the following key questions. 
\begin{question}[Performance Metric and Its Bound]\label{question:performance metric and its bound}
How should we design the performance metric and establish its bound to evaluate an FL system under concept drift? 
\end{question}
Designing an appropriate performance metric is challenging because the performance metric should be able to simultaneously evaluate the inference accuracy of the FL global model on the current data distribution and reflect its generalization ability\footnote{The generalization ability of a machine learning model refers to its capability to perform well and make accurate predictions on new and previously unseen data points that were not part of the training dataset~\cite{DBLP:journals/jsait/BuZV20}.} to future unseen data distributions, which are highly dynamic and unpredictable.    
Furthermore, the performance metric should also capture the changes in data distribution for the long term. To address these challenges, we introduce a performance metric termed the \emph{Stationary Generalization Error} by employing a Markov chain to model the evolution of data distributions over time. To establish the bound for the performance metric, we provide an information-theoretic bound expressed in terms of mutual information and KL divergence. This is challenging because the generalization error of the global model must be derived from local models in the FL aggregation process. This difficulty arises due to the decentralized nature of FL, where local models are trained independently and share only partial information. Consequently, quantifying how the aggregation process influences the global model's generalization performance becomes non-trivial.

\begin{question}[Algorithm Design]\label{question:algorithm design}
How can we formulate an algorithm to alleviate the performance decline resulting from concept drift?
\end{question}
To address Question~\ref{question:algorithm design}, we propose an algorithm inspired by the information-theoretic performance bounds derived in our study. This algorithm enhances the ERM approach~\cite{DBLP:conf/nips/Vapnik91} by incorporating regularization terms based on mutual information and KL divergence. The objective is to mitigate performance degradation caused by concept drift and to enhance the system's long-term average performance.

\begin{question}[Performance-cost Tradeoff]\label{question:Performance-cost Teadeoff}
What is the tradeoff (and hence the Pareto front) between FL system performance and cost under concept drift?
\end{question}
Regarding Question~\ref{question:Performance-cost Teadeoff}, we formulate a problem to maximize the FL system performance, subject to a system cost constraint. The optimization problem is nonconvex, requiring specially designed techniques to mathematically characterize the tradeoff between FL performance and cost.


The main contributions of this paper are summarized as follows:
\begin{itemize}
    \item \textit{Analytical Framework for FL under Concept Drift}: We propose to use the Markov chain to characterize the changes in data distributions over time. Meanwhile, we propose a performance metric, called \emph{Stationary Generalization Error}. It represents the generalization ability on both current data and future unseen data. Additionally, it captures the long-term performance of the FL system under concept drift by tracking changes in data distribution.  
    \item \textit{Performance Bound Analysis}: We determine the upper bound of the Stationary Generalization Error in terms of mutual information and KL divergence. 
    Based on the determined bound, we perform case studies under three primary concept drift patterns to elucidate the impact of evolving data distributions on system performance. We prove that the performance bound mainly depend on the transition probability and KL divergence between data distributions. 
    \item \textit{Algorithm Design}: Based on the performance bound analysis, we propose a method that regularizes the ERM algorithm with KL divergence and mutual information striking the right balance between data fit and generalization. 
    \item \textit{Tradeoff between FL Performance and System Cost}: We formulate an Stationary Generalization Error bound minimization problem. Desipite the noncovexity of the problem, we determine the Pareto frontiers that tradeoff the optimal FL system performance and cost. 
    \item \textit{Experimental Results}: We build an FL system testbed using Raspberry Pi4 devices to verify our bound on the Stationary Generalization Error. The theoretical and experimental results show that the concept drift patterns make a major difference in the system performance, and our algorithm outperforms other concept drift adaptation approaches for these three patterns.
\end{itemize}
This paper is organized as follows. Section~\ref{sec:related work} reviews related work on concept drift in both centralized machine learning and FL. Section~\ref{sec:system model} presents the FL system and defines the performance metric. Section~\ref{sec:Bounding System Error via Mutual Information} provides the theoretical bound along with illustrative case studies. In Section~\ref{sec:algorithm design}, we propose an algorithm to handle concept drift in FL. Section~\ref{sec:Trade off between system error and cost} analyzes the trade-off between FL performance and system cost. Section~\ref{sec:Experiments} reports the experimental results. Finally, Section~\ref{sec:conclusion} concludes the paper.


\section{Related Work}\label{sec:related work}
\subsection{Concept Drift in Centralized Machine Learning}
The phenomenon of concept drift has been extensively studied in centralized machine learning. There are various approaches developed to adapt models to new data distributions. Bach\textit{~et~al}.~\cite{DBLP:conf/icdm/BachM08} addressed concept drift by retraining models from scratch using new data, following a window strategy that retains recent data. Bifet\textit{~et~al}.~\cite{DBLP:conf/sdm/BifetG07} proposed an adaptive method to adjust window size. A smaller window captures the recent data distribution better, but the larger one provides more training data for training a new model. However, both methods~\cite{DBLP:conf/icdm/BachM08,DBLP:conf/sdm/BifetG07} overlooked the potential value of historical data. They retrained models solely on new data and discarded previous models.
Hulten\textit{~et~al}.~\cite{DBLP:conf/kdd/HultenSD01} adaptively updated decision tree models~\cite{DBLP:journals/ijids/PriyankaK20} by maintaining a sliding window of recent data. Similarly, Gama\textit{~et~al}.~\cite{DBLP:conf/kdd/GamaRM03} detected drift in decision tree nodes, converted affected nodes into leaves, and pruned their sub-trees. These two methods~\cite{DBLP:conf/kdd/HultenSD01,DBLP:conf/kdd/GamaRM03}, however, are specific to decision tree algorithms, which are not suitable for neural networks. Xu\textit{~et~al}.~\cite{DBLP:journals/ijon/XuW17} tackled concept drift by dynamically increasing the number of hidden layer nodes, thereby enhancing the neural network's learning capacity. However, altering the neural network structure across all clients in FL is computationally expensive. However, directly applying these centralized methods to FL is non-trivial due to several challenges. First, retraining or modifying models across all clients is often impractical due to limited communication and computation resources. Second, most centralized approaches rely on centralized access to data, which contradicts the privacy constraints in FL. Therefore, there is a need for drift adaptation techniques tailored to FL's decentralized setting.

\subsection{Concept Drift in FL}
Several recent studies have recognized the challenges of concept drift in FL and proposed various solutions.
Yoon\textit{~et~al.}~\cite{pmlr-v139-yoon21b} decomposed network weights into global parameters and sparse task-specific parameters, which enables each client to handle concept drift.
Yu\textit{~et~al.}~\cite{DBLP:journals/tmc/YuCZCZXC23} proposed an algorithm for unlabeled gradient integration to overcome the concept drift due to real-time data sensing. 
However, both works in~\cite{pmlr-v139-yoon21b,DBLP:journals/tmc/YuCZCZXC23} mainly focused on improving the performance of local models rather than training a robust global model. 
Jothimurugesan\textit{~et~al.}~\cite{DBLP:conf/aistats/JothimurugesanH23} characterized the issue of drift adaptation in FL as a clustering problem and proposed a multiple-model clustering FL algorithm. However, this method is computationally expensive, as clients must evaluate their data on multiple models at each step to determine which cluster they belong to.
Canonaco\textit{~et~al.}~\cite{DBLP:conf/ijcnn/CanonacoBMR21} proposed an adaptive learning rate approach, taking into account the phenomenon of concept drift. 
Panchal\textit{~et~al.}~\cite{DBLP:conf/icml/PanchalCMMSMG23} introduced an adaptive optimizer that effectively handles concept drift. The key idea is to detect concept drift by analyzing the magnitude of parameter updates needed to fit the global model.
However, the works in~\cite{DBLP:conf/ijcnn/CanonacoBMR21,DBLP:conf/icml/PanchalCMMSMG23} considered how to accelerate the system's adaptation to the new distribution faster, but they did not consider how to mitigate the performance degradation caused by concept drift.
Furthermore, those approaches in~\cite{pmlr-v139-yoon21b,DBLP:journals/tmc/YuCZCZXC23,DBLP:conf/ijcnn/CanonacoBMR21,DBLP:conf/icml/PanchalCMMSMG23,DBLP:conf/aistats/JothimurugesanH23,DBLP:conf/nips/ChenX0LH24,yang2024multi} did not provide a theoretical framework to analyze the phenomenon of concept drift in FL. In our work, we provide a unified theoretical analysis framework based on information theory for concept drift in FL and propose an algorithm to train a global model to mitigate performance degradation effectively. We consider diverse concept drift patterns and analyze their impact on FL performance degradation and performance-cost tradeoffs. 

\section{System Model}\label{sec:system model}
In this section, we first present an overview of the FL system under concept drift and describe the modeling of the time-varying data distribution. Then, we design the performance metric for the FL system under the concept drift setting.

\subsection{FL System under Concept Drift}
We consider an FL system with a set of $K$ clients, denoted by $\mathcal{K}\triangleq \{1,2,\ldots, K\}$, as shown in Fig. \ref{fig.1}. Suppose there is a ground-truth data distribution at any time instant, i.e., the actual data distribution from which clients collect their datasets. The empirical distribution of the dataset of a client may be different from the ground-truth data distribution because the dataset of a client is finite and may be biased due to random data collection. \textbf{Meanwhile, the ground-truth data distribution varies over time.} As in Fig.~\ref{fig:system-model-time-line}, the data distribution changes once within each time slot. We define each time slot as the time period separated by the moment when a model is deployed on the system. That is, the duration of a time slot is defined by the duration between when an old model has been deployed and when a newly retrained or adapted model is deployed. We consider a setting where one time slot corresponds to a relatively long period (e.g., one or multiple days) and contains a large number of training rounds, within which the model retraining or adaptation can be completed. This setting is supported by real-world scenarios where slow concept drift occurs, such as seasonal energy consumption patterns~\cite{shaikh2023new}, customer behavior shifts in e-commerce~\cite{guthrie2021online}, or long-term changes in environmental data~\cite{parr2003detecting}. These cases align with the assumption of slow drift over extended periods, which enables effective retraining or adaptation within each time slot. Note that our following analysis focuses on a general retraining or adaptation algorithm, denoted by $\mathcal{A}$, without restricting to a specific algorithmic choice. In the rest of this paper, we use ``retraining" to refer to the model retraining or adaptation process for presentation simplicity. 
\subsubsection{Time-Varying Data Distribution}
Inspired by concept drift often involving transitions among different distributions, we model the time-varying data distributions as a Markov chain\footnote{While data in the real world changes gradually, in practice, new data is typically collected periodically for retraining, which renders the process effectively discrete.}. This setting is motivated by several reasons. First, the memoryless property of a Markov chain makes it well-suited for capturing concept drift, as the future distribution depends only on the current distribution, which aligns with the nature of concept drift where the most recent changes are often the most relevant. Second, the dynamic adaptation capability of the Markov chain enables effective representation of transitions between different data distributions over time. Third, this approach provides a mathematically tractable framework that simplifies the modeling of temporal changes, making it feasible to predict and mitigate the impact of concept drift. Finally, by representing the uncertainty of distribution changes, a Markov chain offers a structured way to manage the inherent unpredictability of concept drift, facilitating more robust model adaptation. 

Specifically, we consider the time-varying data distributions as a sequence of random variables $\Pi_t,t\in \mathcal{T}=\{1,2,...\}$ over time. The corresponding state space is defined as $\mathcal{M} \triangleq \{\pi_i,i\in \mathcal{D}\}$. Here, $\mathcal{D}\triangleq\{1,2,...,D\}$ denotes the set of indices of the states in $M$, where $D$ is the total number of states. In other words, data distribution $\Pi_t$ can take a specific distribution $\pi_i$ from the state space $\mathcal{M}$ in a time slot $t$. 
Here, we employ the notations \textbf{$\pi_{pre}\in \mathcal{M}$, $\pi_{cur}\in \mathcal{M}$, $\pi_{nxt}\in \mathcal{M}$}, and $p(\pi_{nxt}|\pi_{cur})$. Consider time slot $t\in \mathcal{T}$. $\pi_{cur}$ represents the state taken by distribution $\Pi_{t}$. Accordingly, $\pi_{pre}$ denotes the state in $\mathcal{M}$ taken by $\Pi_{t-1}$ in the previous time slot. $\pi_{nxt}$ signifies the state taken by distribution $\Pi_{t+1}$. We denote the transition probability from state $\pi_{cur}$ to state $\pi_{nxt}$ as $p(\pi_{nxt}|\pi_{cur})$. Let $\overline{\mathbf{S}}_{k,cur}$ and $\overline{\mathbf{S}}_{k,pre}$ denote the $k_{th}$ client's datasets drawn from distribution $\pi_{cur}$ and $\pi_{pre}$. Let $\mathbf{\overline{S}}_{cur}=(\overline{\mathbf{S}}_{1,cur}, \dots,\overline{\mathbf{S}}_{K,cur})$ and $\mathbf{\overline{S}}_{pre}=(\overline{\mathbf{S}}_{1,pre}, \dots ,\overline{\mathbf{S}}_{K,pre})$.  Let $\mathbf{S}_{k,cur}$ represent the dataset drawn from $\overline{\mathbf{S}}_{k,cur}$, and $\mathbf{\mathbf{S}}_{k,pre}$ represent the dataset drawn from  $\overline{\mathbf{S}}_{k,pre}$.



\subsubsection{FL Process in One Time Slot}
In each time slot, there are two stages, with the details shown in Fig.~\ref{fig:system-model-time-line}. Now, we illustrate using time slot $t$. At the beginning of stage I of time slot $t$, the data distribution has not yet changed, remaining $\Pi_{t-1}=\pi_{pre}$. Clients use the global model $\bold{w}_{pre}$ to do the inference with the dataset sampled from $\Pi_{t-1}=\pi_{pre}$. 
Note that the global model $\bold{w}_{pre}$ was trained by an arbitrary algorithm $\mathcal{A}$ with training datasets sampled from $\Pi_{t-1}=\pi_{pre}$ and $\Pi_{t-2}$ in time slot $t-1$.\footnote{We consider such a setting because incorporating datasets from $\Pi_{t-1}$ and $\Pi_{t-2}$ enhances the FL system's generalization capability, which can mitigate the risk of model overfitting.} When a new data distribution $\Pi_{t}=\pi_{cur}$ shows up, stage II begins. Consequently, the inference accuracy decreases on the dataset sampled from the newly arrived unseen distribution $\Pi_{t}=\pi_{cur}$. At the same time, clients start to retrain a new model $\bold{w}_{cur}$ based on $\bold{w}_{pre}$ using training datasets $\mathbf{\overline{S}}_{cur}$ and $\mathbf{\overline{S}}_{pre}$ sampled from $\pi_{cur}$ and $\pi_{pre}$, respectively. Stage II continues until the training process for $\bold{w}_{cur}$ finishes. When the new model $\bold{w}_{cur}$ is deployed on the system to improve the inference accuracy of the FL system, stage II terminates. 

The system's cost modeling in relation to this process is discussed in Section~\ref{sec:Trade off between system error and cost}.

\begin{figure}[t]
\centering
\setlength{\abovecaptionskip}{-0.2pt}
\includegraphics[width=1\linewidth]{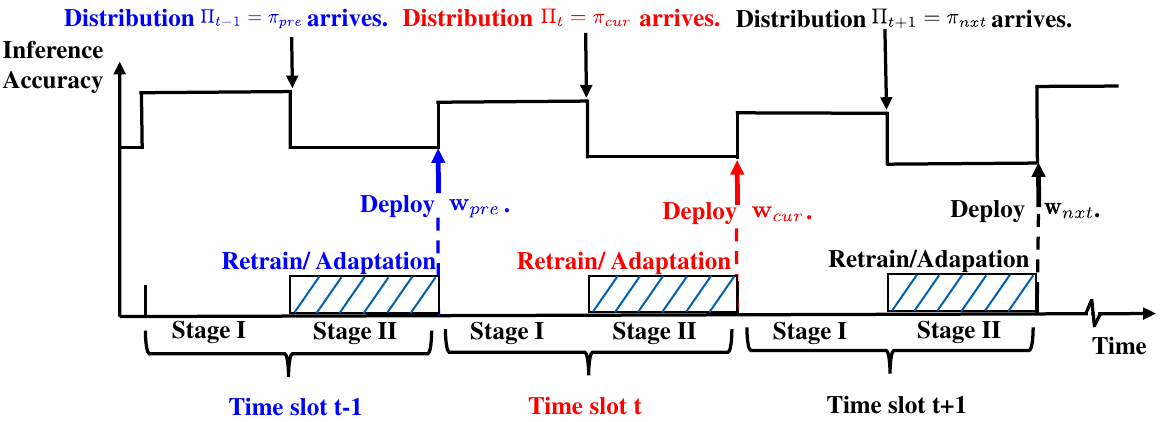}
\caption{Inference accuracy decreases due to concept drift. Each time slot contains two stages: Stage I begins when a new model is deployed; Stage II begins when new data arrive. }
\label{fig:system-model-time-line}
\vspace{-5mm}
\end{figure}

\subsection{Performance Metric}\label{sec:performance metric}
To evaluate the performance of an FL system under concept drift and describe the phenomenon of concept drift, we design a performance metric for the FL process described above. The performance metric should effectively demonstrate the model's inference capacity on the current distribution in stage I and the generalization capacity on future unseen distribution in stage II for each time slot. Further, the performance metric should also capture the long-term evolution of data distribution over time. 
Generalization error is commonly used to measure a machine learning model's generalization capacity to unseen data~\cite{DBLP:journals/jsait/BuZV20,DBLP:conf/isit/WuMAZ20,DBLP:journals/entropy/BarnesDP22}. Inspired by this, we design a performance metric, called \emph{Stationary Generalization Error}, that  
\begin{enumerate}
    \item considers the generalization error from stage I and stage II;
    \item depicts the characteristic of concept drift considering the transition probability of data distributions when the Markov chain reaches the steady state.
\end{enumerate}
In the following, we first define the generalized error for an arbitrary model $\bold{w}_{cur}$ at stage I and stage II in a time slot. Then, we define the Stationary Generalization Error over the state space $\mathcal{M}$ when the Markov chain reaches the steady state. The assumption of a steady state is crucial for our analysis. It reflects the system's long-term behavior, decoupled from transient dynamics that may depend on initial conditions or early adaptation phases. This aligns with real-world scenarios where machine learning systems typically operate over extended periods, and their performance metrics stabilize. Furthermore, steady-state analysis allows us to derive bounds that are generalizable and robust, ensuring applicability across a wide range of deployments.

\subsubsection{Generalization Error for Model $\bold{w}_{cur}$}\label{sec:Generalization Error for Model}
In literature, generalization error is usually defined as the difference between the population risk and the empirical risk~\cite{DBLP:conf/isit/WuMAZ20,DBLP:journals/jsait/BuZV20}. Thus, we first define the population risk and the empirical risk for model $\bold{w}_{cur}$ respectively. Then, we define its generalization error.

Let $\ell(\bold{w}_{cur},Z_{cur})$ be a real-valued loss function, where $\bold{w}_{cur}$ is the model and $Z_{cur}$ is a sample from dataset $\mathbf{S}_{cur}$. The population risk for the global model $\bold{w}_{cur}$ is defined as:
\begin{equation}\label{eq:population risk}
L_{\pi_{cur}}(\bold{w}_{cur}) = \mathbb{E}_{Z_{cur} \sim \pi_{cur}}\left[\ell\left(\bold{w}_{cur},Z_{cur}\right)\right].
\end{equation}


To introduce the empirical risk of model $\bold{w}_{cur}$ for client $k$, we first define some notations. 
Suppose that $N$ is the number of data samples used for model retraining. 
Let $N_{cur}$ and $N_{pre}$ denote the sizes of  $\mathbf{S}_{k,cur}$ and $\mathbf{S}_{k,pre}$, and  \(N_{cur}+N_{pre}=N\).
Further, we define $\alpha \in (0,1]$ as a weight parameter to determine the importance for $\mathbf{S}_{k,cur}$ and $\mathbf{S}_{k,pre}$ following the previous work~\cite{DBLP:conf/isit/WuMAZ20}.
Intuitively, in practical systems, we want the system to perform better on fresher data, so we may take a larger value for $\alpha$ to make the performance metric more accurate. $Z_{nk,cur}$ and $Z_{nk,pre}$ represent the $n_{th}$ data sample of client $k$ in datasets $\mathbf{S}_{k,cur}$ and $\mathbf{S}_{k,pre}$, respectively. 
Then, the empirical risk of global model $\bold{w}_{cur}$ of client $k$ is defined as:
\vspace{-2mm}
\begin{multline}\label{eq:empirical rick}
\!\!\!\!\hat{L}_{\alpha}(\mathbf{w}_{cur},S_{k,cur},S_{k,pre})\triangleq\\ \alpha\left(\frac{1}{N_{cur}}\sum\limits_{n=1}^{N_{cur}}\ell(\mathbf{w}_{cur},Z_{nk,cur})\right)\\
+(1-\alpha)\left(\frac{1}{N_{pre}}\sum\limits_{n=N_{cur}+1}^{ N}\ell(\mathbf{w}_{cur},Z_{nk,pre})\right).
\end{multline}

Next, we define the generalization error for model $\bold{w}_{cur}$. Following the conventional definition in~\cite{DBLP:conf/isit/WuMAZ20,DBLP:journals/jsait/BuZV20}, the generalization error is the difference between the population risk and empirical risk.
Suppose $\bold{w}_{cur}$ has been deployed at the end of time slot $t$. We use $G_1(\pi_{cur},\bold{w}_{cur})$ to denote the generalization error for model $\bold{w}_{cur}$ inferring the data samples drawn from $\pi_{cur}$ at stage I in time slot $t+1$, as shown in Fig. \ref{fig:system-model-time-line}. Let $G_2(\pi_{nxt}, \bold{w}_{cur})$ represent the generalization error for model $\bold{w}_{cur}$ inferring the data samples drawn from the unseen distribution $\pi_{nxt}$  as stage II in time slot $t+1$. Therefore, the generalization error for model $\bold{w}_{cur}$ at stage I and stage II (e.g., in time slot $t+1$ of Fig.~\ref{fig:system-model-time-line}) can be defined as \eqref{eq:stage I-mean generalization error} and \eqref{eq:stage II-mean generalization error}, respectively: 
\vspace{-2mm}
\begin{multline}
\label{eq:stage I-mean generalization error}
\!\!\!\!\!G_1(\pi_{cur}, \bold{w}_{cur})=\\
\!\!\!\!\!\!L_{\pi_{cur}}(\bold{w}_{cur})-\frac{1}{K}\!\sum\limits_{k=1}^{K}\!\hat{L}_{\alpha}(\bold{w}_{cur},\mathbf{S}_{k,cur},\mathbf{S}_{k,pre}),
\end{multline}
\vspace{-5mm}
\begin{multline}
\label{eq:stage II-mean generalization error}
\!\!\!\!\!G_2(\pi_{nxt}, \bold{w}_{cur})=\\
\!\!\!\!\!\!L_{\pi_{nxt}}(\bold{w}_{cur}) -\frac{1}{K}\!\sum\limits_{k=1}^{K}\!\hat{L}_{\alpha}(\bold{w}_{cur},\mathbf{S}_{k,cur},\mathbf{S}_{k,pre}).
\end{multline}
The second terms in Eq.~\eqref{eq:stage I-mean generalization error} and Eq.~\eqref{eq:stage II-mean generalization error} are identical because the empirical risk is computed using the same model parameters \(\bold{w}_{cur}\) and the same data samples \(\{\mathbf{S}_{k,cur}, \mathbf{S}_{k,pre}\}_{k=1}^K\)~\cite{DBLP:conf/isit/WuMAZ20,DBLP:journals/jsait/BuZV20}.

To represent the generalization error of $\bold{w}_{cur}$ in a time slot, we define $\tau_1$ to represent the ratio of the time length of stage I for a time slot, and $\tau_2$ to denote the ratio of the time length of stage II for a time slot. $\tau_1+\tau_2=1$. Although we set $\tau_1$ and $\tau_2$ to remain constant over time, such a setting can still accurately describe scenarios where the time interval between two different data distributions varies. This is because \(\Pi_t\) and \(\Pi_{t+1}\) are random variables that can take the same distribution from the state space \(\mathcal{M}\), effectively modeling situations where no distribution shift occurs between consecutive time slots. Given $\tau_1$ and $\tau_2$, by taking the expectation on $G_1(\pi_{cur}, \bold{w}_{cur})$ and $G_2(\pi_{nxt}, \bold{w}_{cur})$ with respect to model $\bold{w}_{cur}$ and datasets $S_{k,cur}$, $S_{k,pre}$, we define a \emph{weighted generalization error} for model $\bold{w}_{cur}$ as:
\vspace{-2mm}
\begin{multline}\label{eq:WEN}
\!\!\!\!\!\!G(\pi_{cur},\pi_{nxt})\!\triangleq\! \\
\tau_{1} \mathbb{E}_{\bold{wS}}[G_1(\pi_{cur},\bold{w}_{cur})]+\tau_2\mathbb{E}_{\bold{w}S}[G_2(\pi_{nxt}, \bold{w}_{cur})].
\end{multline}
The expectation $\mathbb{E}_{\bold{wS}}[\cdot]$ is taken with respect to the randomness in the learned model $\bold{w}_{cur}$ and the sampling of local datasets $\{\mathbf{S}_{k,cur}, \mathbf{S}_{k,pre}\}_{k=1}^K$.

\subsubsection{The Stationary Generalization Error $\overline{G}$}

Note that in Eq.~\eqref{eq:WEN}, $\bold{w}_{cur}$ is a random variable, as it was trained based on  distributions $\pi_{pre}$ and $\pi_{cur}$ randomly drawn from state space $\mathcal{M}$. Meanwhile, $\pi_{nxt}$ is also a random variable  drawn from state space $\mathcal{M}$.  
Therefore, to depict the characteristic of concept drift over time, we define an Stationary Generalization Error for all realizations of $\pi_{pre},\pi_{cur},\pi_{nxt}$ as our ultimate performance metric. 

\begin{definition}[Stationary Generalization Error]\label{def:System error}
The \emph{Stationary Generalization Error} $\overline{G}$ is determined as:
\begin{equation}
\label{eq: expected generalization error}
\begin{aligned}
\overline{G}&\triangleq \mathbb{E}_{\pi_{pre}\pi_{cur}\pi_{nxt}}\left[G(\pi_{cur},\pi_{nxt})\right]\\
&\triangleq \sum\limits_{\pi}p(\pi_{pre})p(\pi_{cur}|\pi_{pre})p(\pi_{nxt}|\pi_{cur})G(\pi_{cur},\pi_{nxt}),
\end{aligned}
\end{equation}
where $\pi$ is short for $\pi_{pre}\pi_{cur}\pi_{nxt}$.
\end{definition}

The Stationary Generalization Error  \(\overline{G}\) is defined based on the case when the Markov chain reaches steady state. It captures the model’s ability to fit the current (recent) data distribution, while also reflecting its generalization ability to future (unseen) distributions under concept drift. The assumption of a steady state is justified as it reflects a long-term behavior of the system under concept drift. In practical scenarios, the distribution changes introduced by concept drift often stabilize over time, either cyclically (e.g., seasonal patterns) or through convergence to a new dominant distribution. Analyzing the system in a steady state allows for capturing the averaged effects of these changes over time. Additionally, the steady-state assumption simplifies theoretical analysis while maintaining practical relevance, as it provides insights into the system's expected behavior after sufficient adaptation to the drift.
In Section~\ref{sec:Bounding System Error via Mutual Information}, we aim to bound $\overline{G}$ and to quantify the impact of concept drift.

\section{Bounding Stationary Generalization Error via KL Divergence and Mutual Information}\label{sec:Bounding System Error via Mutual Information}

We first give an upper bound of $\overline{G}$ following the information-theoretic framework~\cite{DBLP:journals/jsait/BuZV20,DBLP:conf/isit/WuMAZ20}. Then, based on the upper bound of $\overline{G}$, we conduct case studies on typical concept drift patterns to provide concrete insights.  Finally, motivated by these theoretical findings, we propose a method to enhance long-term performance by regularizing the ERM approach~\cite{DBLP:conf/nips/Vapnik91} with mutual information and KL divergence.


\subsection{Upper Bound on the Stationary Generalization Error} \label{sec:prove upper bound}
In this section, we first provide some assumptions, definitions, and lemmas used for the upper bound derivation. Then, we prove the upper bound of $\overline{G}$ defined in Eq. (\ref{eq: expected generalization error}). 

To characterize the information revealed by a model with respect to a data sample, we introduce Assumption \ref{assumption:cumulant generating function} following \cite{DBLP:journals/jsait/BuZV20,DBLP:conf/isit/WuMAZ20}.
\begin{assumption}\label{assumption:cumulant generating function}
Considering an arbitrary $\bold{w}$ and a data sample $Z$, the cumulant generating function of the random variable $\ell(\bold{w}, Z)-\mathbb{E}_{Z}[\ell(\bold{w}, Z)]$ is bounded by some $\lambda$ and convex function $\psi$ on the interval $[0,b)$:
\begin{equation}\label{eq:cumulant generating function}
\log \mathbb{E}\left[e^{\lambda(\ell(\bold{w}, Z)-\mathbb{E}_{Z}[\ell(\bold{w}, Z)])}\right] \leq \psi(-\lambda).
\end{equation}
\end{assumption}
We further define the Legendre dual for $\psi$ and introduce Lemma \ref{lemma:inverse function of psi}, for capturing the upper bound of $\overline{G}$.

\begin{definition}[Legendre Dual]
    For a convex function $\psi$ defined on the interval $[0,b)$, where $0<b\leq \infty$, its Legendre dual $\psi^{*}$ is defined as 
    \begin{equation}\label{eq:psi}
        \psi^{*}(x) \triangleq \sup _{\lambda \in[0, b)}\left(\lambda x-\psi\left(\lambda\right)\right).
    \end{equation}
\end{definition}

\begin{lemma}[\!\!\cite{DBLP:journals/jsait/BuZV20}, Lemma 2]\label{lemma:inverse function of psi}
Assume that $\psi(0)=\psi^{'}(0)=0$, then the inverse function of Eq. (\ref{eq:psi}) can be written as:
\begin{equation}\label{eq:inverse function of psi}
    \psi^{*-1}(y) \triangleq \inf _{\lambda \in[0, b)}\left(\frac{y+\psi(\lambda)}{\lambda}\right).
\end{equation}        
\end{lemma}

In addition, according to~\cite{DBLP:books/daglib/0035704}, we characterize the difference between two distributions using KL divergence. Specifically, the KL divergence between two distributions $U$ and $V$ defined over sample space $\mathcal{X}$ for some function $f$ is given by:
\begin{equation} \label{eq:KL divergence}
D_{\rm KL}(U\|V)=\sup\limits_f\left[\mathbb{E}_U[f(x)]-\log\mathbb{E}_V[e^{f(x)}]\right].
\end{equation}
Let $P_{\bold{w},Z}$ be the joint distribution of $\bold{w}$ and $Z$. Then the mutual information between $\bold{w}$ and $Z$ is given by:
\begin{equation} \label{eq:mutual information}
I(\bold{w},Z)=D_{\rm KL}(P_{\bold{w},Z}\|P_{\bold{w}} \otimes P_{Z}),
\end{equation}
where \(\otimes\) denotes the \emph{product distribution} of  \(P_{\bold{w}}\) and \(P_Z\).



Now, we are ready to provide the upper bound of $\overline{G}$. The following Theorem \ref{theorem:expected generalization bound} provides an upper bound of $\overline{G}$ with respect to (i) the mutual information $I(\bold{w},Z)$ that quantifies the dependence between
the output of the learning algorithm $\bold{w}$ and the input individual training sample $Z$, and (ii) the KL divergence $D_{\rm KL}(\cdot\|\cdot)$ that quantifies the distance between two
different data distributions in state space $\mathcal{M}$, following the characterization used in~\cite{DBLP:journals/jsait/BuZV20,DBLP:conf/isit/WuMAZ20}. 


\begin{theorem}[Upper Bound of the Stationary Generalization Error]\label{theorem:expected generalization bound}
The Stationary Generalization Error is bounded as:
\vspace{-2mm}
\begin{multline}  \label{eq:system error bound}
\overline{G}\leq \frac{1}{K^2}\sum\limits_{\pi_{pre},\pi_{cur},\pi_{nxt}}p(\pi_{pre})p(\pi_{cur}|\pi_{pre})p(\pi_{nxt}|\pi_{cur})\\
\left(\tau_1\sum\limits_{k=1}^{K}\left(\frac{\alpha}{N_{cur}}\!\!\sum\limits_{n=1}^{N_{cur}}\!\psi^{*-1}(I(\mathbf{w}_{cur},Z_{nk,cur}))\right.+\frac{1-\alpha}{N_{pre}}\right.\\
\left.\sum\limits_{n=N_{cur}+1}^{N}\!\!\!\!\!\psi^{*-1}(I(\mathbf{w}_{cur},Z_{nk,pre})+D_{\rm KL}(\pi_{pre}\|\pi_{cur}))\!\right)\\
\!\!\!\!\!\!\!+\tau_2\sum\limits_{k=1}^{K}\left(\frac{\alpha}{N_{cur}}\sum\limits_{n=1}^{N_{cur}}\psi^{*-1}(I(\mathbf{w}_{cur},Z_{nk,cur})\right.\\
\left.+D_{\rm KL}(\pi_{cur}\|\pi_{nxt}))\left.+\frac{1-\alpha}{N_{pre}}\right.\right.\\
\left.\left.\sum\limits_{n=N_{cur}+1}^{N}\!\!\!\!\!\psi^{*-1}(I(\mathbf{w}_{cur},Z_{nk,pre})\!+\!D_{\rm KL}(\pi_{pre}\|\pi_{nxt}))\!\right)\!\right).\!\!\!\!\!
\end{multline}
\end{theorem}

The proof of Theorem~\ref{theorem:expected generalization bound} is given in Appendix~A. Theorem~\ref{theorem:expected generalization bound} demonstrates that in the presence of concept drift, the performance of an FL system is primarily bounded by two key terms: mutual information \( I(\bold{w}, Z) \) and KL divergence \( D_{\rm KL}(\cdot \| \cdot) \). Specifically, the mutual information quantifies the dependence between each individual training sample and the output model of the learning algorithm. The mutual information measures how much information is shared between model \( \bold{w} \) and individual sample \( Z \), indicating how well the model has learned from a specific data sample. A lower mutual information indicates less dependence of the model on a single data sample \( Z \). This implies that the model \( \bold{w} \) is less affected by perturbations in the data, thereby possessing stronger generalization capabilities. On the other hand, the KL divergence quantifies the distance between distributions. Generally, as more samples are used for training, the mutual information gradually decreases and approaches zero~\cite{DBLP:journals/jmlr/BousquetE02}. Consequently, further increasing the number of data samples has a diminishing effect on $\overline{G}$. In contrast, the KL divergence does not vanish as the number of data samples increases, which suggests that concept drift persists regardless of the sample size.

Additionally, we observe that the function \(  \psi^{*-1}(\cdot) \) is concave. Mathematically, this implies that the growth rate of \( \overline{G} \) decreases as mutual information and KL divergence increase. Specifically, since \( \psi^{*-1}(\cdot)\) is concave, its second derivative satisfies \( (\psi^{*-1})'' \leq 0 \), indicating a diminishing marginal impact of increases in mutual information and KL divergence on the bound of \( \overline{G} \). This property highlights that the influence of further increases in these quantities on the generalization error bound diminishes, reflecting a nonlinear relationship between these quantities and the system performance.

Based on the above insights, we propose an algorithm to achieve an appropriate tradeoff between data fitting and generalization by controlling mutual information and KL divergence. The details of this algorithm are discussed in Section~\ref{sec:algorithm design}.

\subsection{Case Studies: Three Concept Drift Patterns}\label{sec:Examples under 3 Concept Drift Patterns}

\begin{figure*}[htb]
	\centering
	\subfloat[]{
		\includegraphics[width=0.3\linewidth]{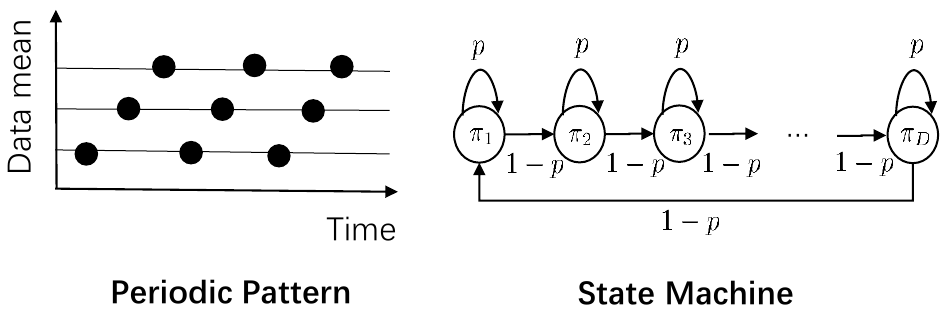}\label{fig:pattern 1}
		}
	\subfloat[]{
		\includegraphics[width=0.3\linewidth]{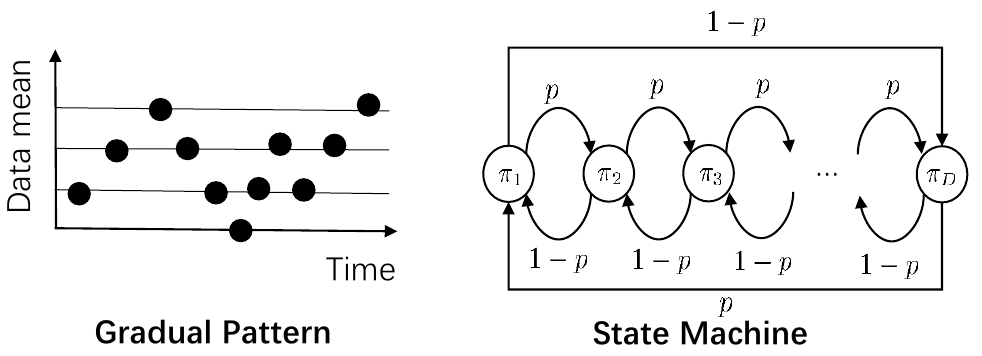}\label{fig:pattern 2}
		}
  	\subfloat[]{
		\includegraphics[width=0.3\linewidth]{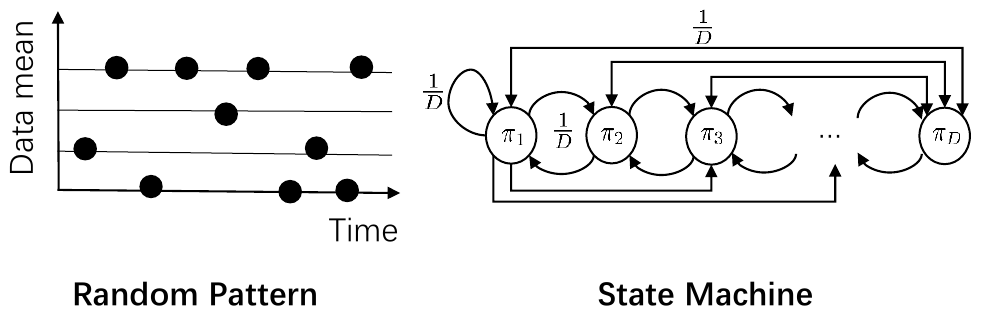}\label{fig:pattern 3}
		}

 \caption{State machines for three patterns. (a) Periodic pattern. (b) Gradual pattern. (c) Random pattern. }
 \label{fig:three pattern ilustration}
\end{figure*}

In this section, we provide specific expressions of the upper bound of $\overline{G}$ under different concept drift patterns and analyze how the bound changes with system factors. 

To compare the difference of the upper bound as the increasing of the number of data samples under different concept drift patterns, we introduce additional assumptions in order to express function $\psi^{*-1}(\cdot)$ in an explicit form.  We suppose loss function $\ell(\bold{w},Z)$ is sub-Gaussian with parameter $r^2$, where $r$ is a constant. This assumption is considered by existing works~\cite{DBLP:conf/isit/WuMAZ20,DBLP:journals/jsait/BuZV20,DBLP:journals/entropy/BarnesDP22} and reasonable because in machine learning, the sub-Gaussian property is often used to analyze the concentration of random variables and to derive concentration inequalities. Note that  parameter $r$ quantifies the tightness of the concentration of loss function around the mean, where a smaller value of $r$ indicates stronger concentration. Based on such an assumption, we have $\psi^{*-1}(x) = \sqrt{2r^2x}$ by Lemma \ref{lemma:inverse function of psi}. 
In addition, we assume that the sample size is large, under which the influence of a data sample $Z$ on the performance of $\bold{w}$ should be small. That is, when $N\rightarrow \infty$, $I(\bold{w}_{cur},Z) \rightarrow 0$~\cite{DBLP:conf/isit/WuMAZ20}. Based on these assumptions, we can rewrite the bound as follows:
\begin{multline}
\label{eq:simplied upper bound}
\overline{G} \leq \frac{1}{K}\mathbb{E}_{\pi}
\left(\tau_1(1-\alpha)\sqrt{r^2 D_{\rm KL}(\pi_{pre}\|\pi_{cur})}\right.\\
\left.+\tau_2\left(\alpha\sqrt{r^2 D_{\rm KL}(\pi_{cur}\|\pi_{nxt})}\right.\right.\\
\left.\left.+(1-\alpha)\sqrt{r^2 D_{\rm KL}(\pi_{pre}\|\pi_{nxt})}\right)\right).
\end{multline}

To quantify the impact of distance between data distributions on the upper bound, we further specify the expression of  $D_{\rm KL}(\cdot\|\cdot)$ in Eq. \eqref{eq:simplied upper bound}. For presentation simplicity, we use notations $\pi_i \in M$ and $\pi_j \in M$ to refer to two different data distributions in the rest of this section. We consider the datasets drawn from multivariate Gaussian distributions and calculate $D_{\rm KL}(\pi_i \| \pi_j)$ as follows:
\begin{definition}
\label{assumption: multivariate Gaussian distribution}
    Consider two multivariate Gaussian distributions $\pi_{i}$ and $\pi_{j}$. $\pi_{i}=\mathcal{N}(\mu_{i},\Sigma_{i})\in M$, $\pi_{j}=\mathcal{N}(\mu_{j},\Sigma_{j})\in M$, where $i, j \in \mathcal{D}$. $\mu_{i},\mu_{j}$ are mean vectors. Suppose $\mu_{i}-\mu_{j}= |i-j|\Delta$, where $\Delta$ is also a vector. Let $\Sigma_{i}=\Sigma_{j}=\Sigma$ be the covariance matrix. Then, the KL divergence between $\pi_{i}$ and $\pi_{j}$ can be calculated as:
    \vspace{-2mm}
    \begin{equation}\label{eq:assumption multivariate Gaussian distribution}
            D_{\rm KL}(\pi_{i}\|\pi_{j})=\frac{(i-j)^2}{2}\Delta^T\Sigma^{-1}\Delta.
    \end{equation}
\end{definition}
Next, we will analyze the periodic, gradual, and random patterns.

\subsubsection{Periodic, Gradual and Random Patterns}  Let $\pi_{i}, \pi_{i-1}, \pi_{i+1}\in M, i\in \mathcal{D}$.
(i) In the periodic pattern, the state machine is shown in Fig. \ref{fig:pattern 1}. State $\pi_{i}$ can be transferred to only the next adjacent state $\pi_{i+1}$ and itself. The transition probabilities are defined as $p(\pi_{i}|\pi_{i})=p$ and $p(\pi_{(i+1)~mod~D}|\pi_{i})=1-p$, where $mod$ is short for modulo. 
(ii) In the gradual pattern, the state machine is shown in Fig. \ref{fig:pattern 2}. State $\pi_{i}$ can only be transferred to adjacent states $\pi_{(i-1)~mod~D}$ and $\pi_{(i+1)~mod~D}$. The transition probabilities are defined as $p(\pi_{(i+1)~mod~D}|\pi_{i})=p$ and $p(\pi_{(i-1)~mod~D}|\pi_{i})=1-p$.
(iii) In the random pattern, the state machine is shown in Fig.~\ref{fig:pattern 3}. Each state can be transferred to any other state with probability $1/D$. 

We can derive the detailed expression of the upper bound on $\overline{G}$ for these patterns:
\begin{corollary}[the Upper Bound on $\overline{G}$ for Three Patterns]\label{corollary: case study}
In periodic pattern, the upper bound on $\overline{G}$ is given by:
\vspace{-2mm}
\begin{multline} \label{eq: pattern 1}
\!\!\!\!\!\!\overline{G}\!\leq\!  \frac{r\sqrt{\Delta^T\Sigma^{-1}\Delta}}{KD}\bigg(\!\big((4\alpha+1)\tau_2-\alpha+1\big)p^2\\
-\big((8\alpha+2D)\tau_2+2D-2\alpha D\big)p\\
+(2D+4\alpha-1)\tau_2+2D-2\alpha D-1 + \alpha\bigg).
\end{multline}
In gradual pattern, the upper bound on $\overline{G}$ is given by:
\begin{multline} \label{eq: pattern 2}
\overline{G}\!\leq\! \frac{r\sqrt{\Delta^T\Sigma^{-1}\Delta}}{KD}\bigg(8\tau_2(D-2+2\alpha-\alpha D)(p^2-p)\\
+(2D+4\alpha-6)\tau_2+2D-2\alpha D-2 + 2\alpha\bigg).
\end{multline}
In the random pattern, the upper bound on $\overline{G}$ is given by:
\begin{equation} \label{eq: pattern 3}
\begin{aligned}
\overline{G}&\leq \frac{r\sqrt{\Delta^T\Sigma^{-1}\Delta}}{K}\left(\tau_2\alpha-\alpha+1\right)\left(\frac{5}{3}D+\frac{1}{3D}-2\right).
\end{aligned}
\end{equation}
\end{corollary}
Corollary~\ref{corollary: case study} is derived by first computing the stationary distribution of the Markov chain for each drift pattern, and then substituting it into the upper bound expression of $\overline{G}$  given in Eq.~\eqref{eq:simplied upper bound}.

\subsubsection{Comparison between Patterns}\label{sec:Comparing Three Patterns}

We denote functions $B_{p}(p)$, $B_{g}(p)$, and $B_{r}(p)$ in terms of $p$ as the upper bounds of $\overline{G}$ for periodic pattern, gradual pattern, and random pattern, respectively.

\begin{proposition}[Threshold]\label{proposition:Threshold}
    There exists a threshold $p_{th} \leq 0.5$, such that $B_{r}(p)>B_{p}(p)>B_{g}(p)$ for any $p<p_{th}$.
\end{proposition}
\begin{figure}[t]
	\centering
	\subfloat[]{
                \includegraphics[width=0.45\linewidth]{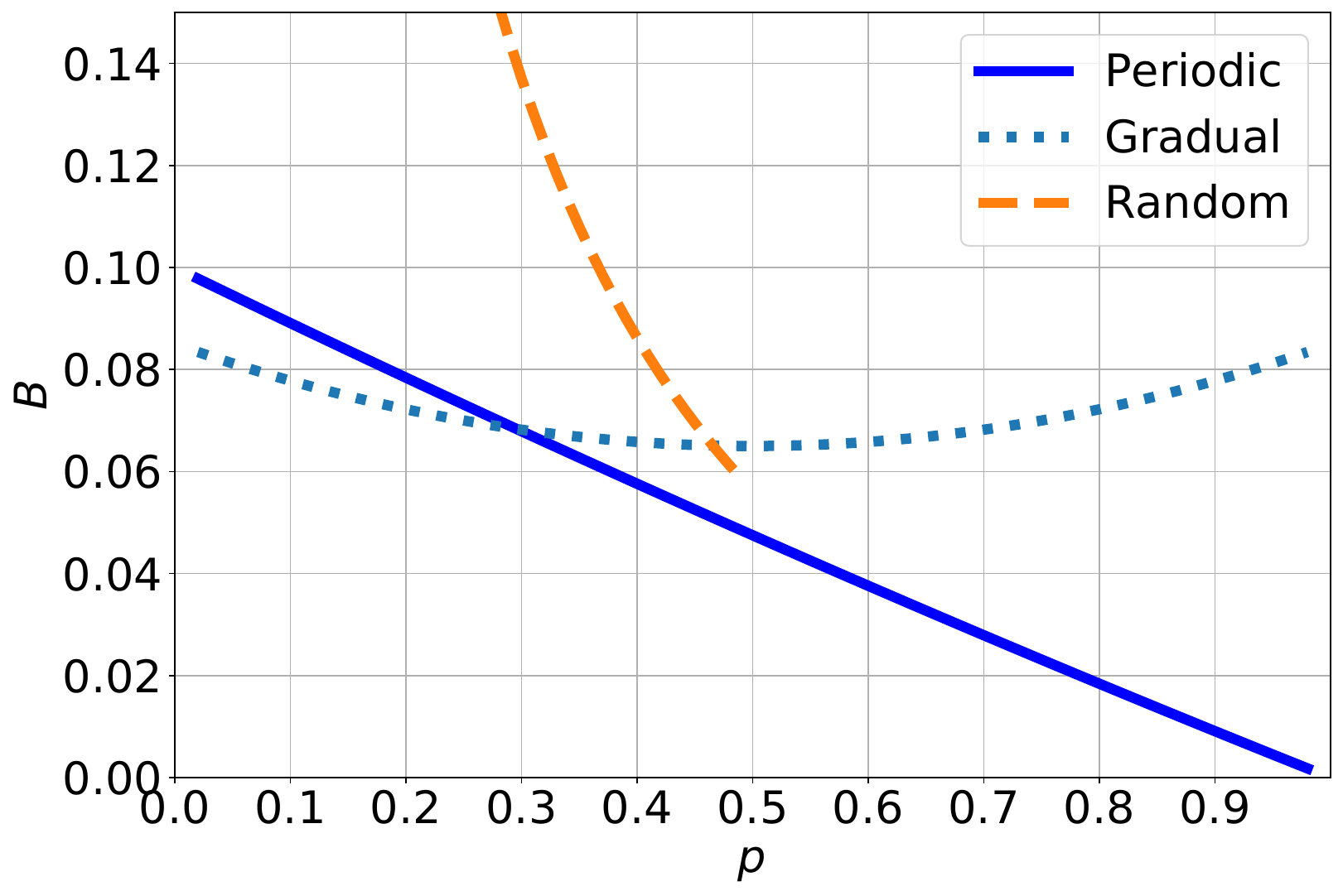}\label{fig:compare_3_pattern_bound_p_1}
                }
	\subfloat[]{
		\includegraphics[width=0.45\linewidth]{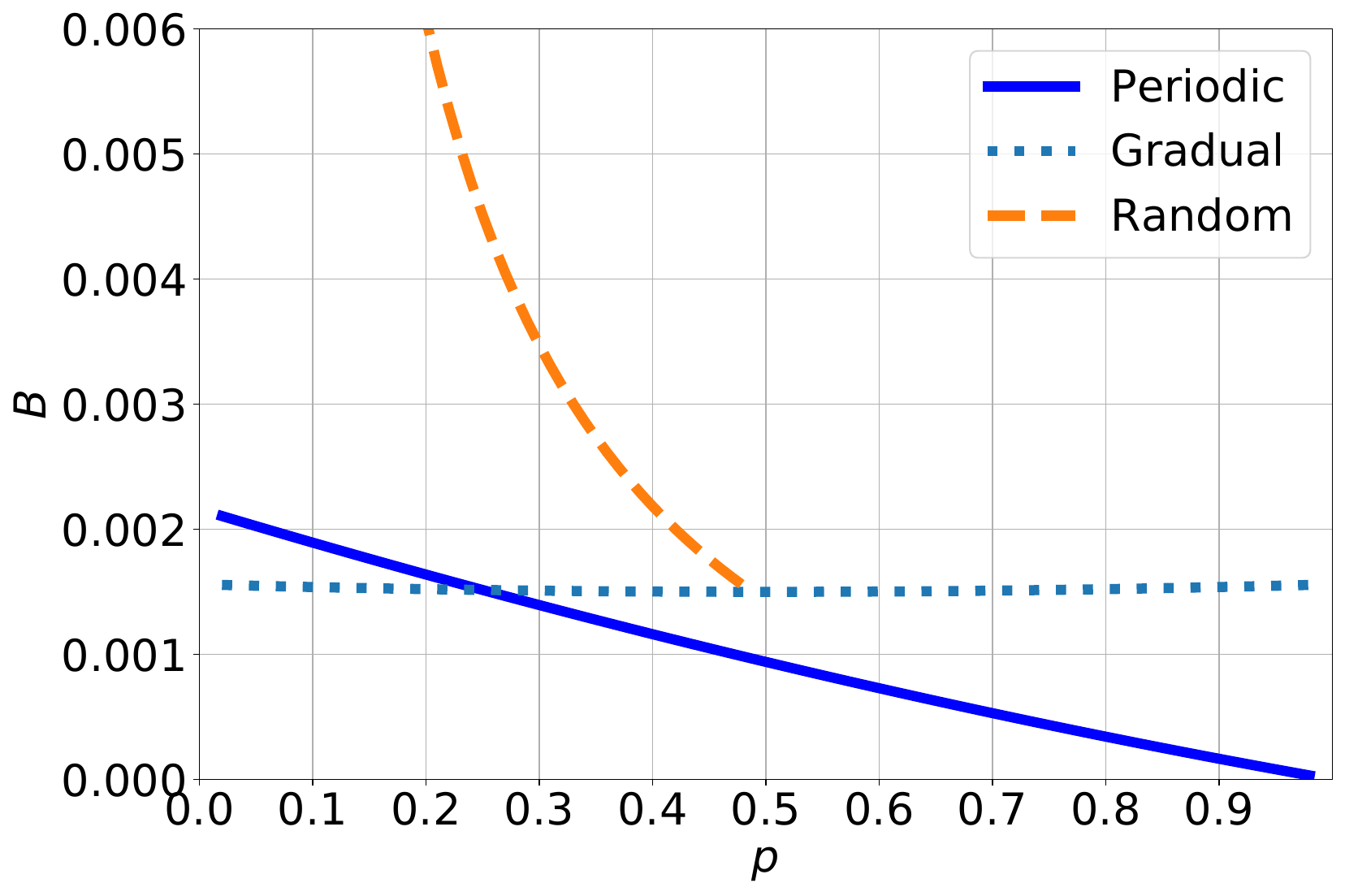}\label{fig:compare_3_pattern_bound_p_2}}
 \caption{The upper bound of $\overline{G}$ changes over $p$. (a) $D=10$, $K=10$, $\alpha=0.5$, $\tau_{2}=0.5$. (b) $D=5$, $K=100$, $\alpha=0.9$, $\tau_{2}=0.1$.}
 \label{fig:compare_3_pattern_bound_p}
\end{figure}

The proof of Proposition~\ref{proposition:Threshold} is given in Appendix~B. Resulting from Proposition~\ref{proposition:Threshold}, as the data distribution tends to change drastically (i.e., when $p$ is small), the FL performance degradation in the random pattern scenario is more severe than that in the periodic pattern, while it is relatively minor in the gradual pattern. This highlights the necessity of carefully designed FL algorithms to effectively address the concept drift phenomenon, particularly in the context of the random pattern. 
Such an analytical result is validated in Fig.~\ref{fig:compare_3_pattern_bound_p_1} with $D=10$, $K=10$, $\alpha=0.5$, $\tau_1=0.5$, $\tau_2=0.5$, $r=0.5$, $\sqrt{\Delta^T\Sigma^{-1}\Delta}=1$.


\begin{proposition}[The Effects of Parameters]\label{proposition:The Effects of Parameters on the Threshold} 
The bounds $B_g(p)$, $B_r(p)$, and $B_p(p)$, as well as the threshold $p_{th}$, decrease under the following conditions: (i) $\alpha$ increases, (ii) $K$ increases, or (iii) $\tau_2$ decreases.
\end{proposition}
Proposition~\ref{proposition:The Effects of Parameters on the Threshold} can be established by examining the analytical expressions of the bounds and the threshold, and by evaluating their partial derivatives with respect to each parameter.  
From Proposition~\ref{proposition:The Effects of Parameters on the Threshold}, we could improve FL system performance by increasing the number of clients, assigning higher weights on new data, or accelerating the training process ($\tau_2$ decreases). Roughly speaking, improving FL performance via these parameter adjustments, the FL performance becomes more robust to the periodic pattern (i.e., has a smaller bound over a wider range of $p$) when compared with the gradual and random patterns.


\subsection{Algorithm Design}\label{sec:algorithm design}


Inspired by Theorem~\ref{theorem:expected generalization bound}, we can balance between data fitting and generalization by controlling the KL divergence between different distributions and the mutual information between each training sample and the output of the learning algorithm.
Based on this idea, we propose an algorithm that regularizes
the ERM algorithm~\cite{DBLP:conf/nips/Vapnik91} with KL divergence and mutual information. The proposed method is detailed in Algorithm~\ref{alo:our method}. 

For each client $k \in \mathcal{K}$, we define the client local training loss function of time slot $t$, which serves as the foundation for optimizing the FL model on individual clients. To reduce the KL divergence between different data distributions, we add KL divergence regularization terms between the current and previous data distributions to the client local training loss function. To reduce the mutual information between each individual training sample and the output hypothesis, we follow~\cite{DBLP:conf/nips/XuR17} and add an independent random noise term into the client local training loss function. Specifically, we formulate the KL regularization terms as $D_{\rm KL}(\mathcal{P}_{cur}^{\mathbf{w_k}}(\hat{Z}_{k,cur})\|\mathcal{P}_{pre}^{\mathbf{w}_k}(\hat{Z}_{k,pre}))$ and $D_{\rm KL}(\mathcal{P}_{pre}^{\mathbf{w}_k}(\hat{Z}_{k,pre})\|\mathcal{P}_{cur}^{\mathbf{w}_k}(\hat{Z}_{k,cur}))$ in the representation space, which also can be interpreted as matching the distribution of the learned representations from current and previous datasets~\cite{DBLP:conf/icml/ShenBW23,DBLP:journals/tit/DongGCSZL25,DBLP:conf/iclr/NguyenTGTB22}. Here, the local model $\mathbf{w}_k$ of client $k$ maps input data $Z_{k}$ to representation $\hat{Z}_k=h(Z_{k};\mathbf{w}_k)$, which corresponds to the output of the last hidden layer of $\mathbf{w}_k$. Note that $\mathbf{w}_{cur}=\frac{1}{K}\sum_{k=1}^{K}\mathbf{w}_k$. $\mathcal{P}_{cur}^{\mathbf{w}_k}(\hat{Z}_{k,cur})$ and $\mathcal{P}_{pre}^{\mathbf{w}_k}(\hat{Z}_{k,pre})$ denote the distributions over learned representations from the current and previous local datasets, respectively. In practice, we model these posterior distributions of representations using Gaussian distributions. To regulate the mutual information, we introduce an independent Gaussian noise $\epsilon \sim \mathcal{N}(\mu,\Sigma)$ to the local training loss function.
The local training loss function for a client $k$ is given by: 
\begin{multline}
f(\bold{w}_{k}) = \frac{1}{ N_{cur}}\sum\limits_{n=1}^{ N_{cur}}\ell(\bold{w}_{k},Z_{nk,cur}) \\
+\gamma D_{\rm KL}(\mathcal{P}_{cur}^{\mathbf{w}_k}(\hat{Z}_{k,cur})\|\mathcal{P}_{pre}^{\mathbf{w}_k}(\hat{Z}_{k,pre}))\\
+ \gamma D_{\rm KL}(\mathcal{P}_{pre}^{\mathbf{w}_k}(\hat{Z}_{k,pre})\|\mathcal{P}_{cur}^{\mathbf{w}_k}(\hat{Z}_{k,cur}))+\epsilon, \label{eq:objective function}
\end{multline}
where $\gamma$ is a hyper-parameter. Here, the first term focuses on minimizing the empirical risk over the current dataset, while the KL divergence terms ensure consistency between the learned representations distributions of the current dataset and the previous dataset and the noise term regulates the mutual information. Our goal is to minimize Eq.~\eqref{eq:objective function}.

It is difficult to obtain $D_{\rm KL}(\mathcal{P}_{cur}^{\mathbf{w}_k}(\hat{Z}_{k,cur}) \| \mathcal{P}_{pre}^{\mathbf{w}_k}(\hat{Z}_{k,pre}))$ and $D_{\rm KL}(\mathcal{P}_{pre}^{\mathbf{w}_k}(\hat{Z}_{k,pre}) \| \mathcal{P}_{cur}^{\mathbf{w}_k}(\hat{Z}_{k,cur}))$, therefore in practice we estimate these KL divergence terms using mini-batches.
Let  \((Z_{ik,cur})_{i=1}^B\) and \((Z_{ik,pre})_{i=1}^B\) be mini-batches drawn from the current dataset $\mathbf{S}_{cur}$ and the previous dataset $\mathbf{S}_{pre}$, where $B$ is the batch size. Their corresponding representations, denoted by \((\hat{Z}_{ik,cur})_{i=1}^B\) and \((\hat{Z}_{ik,pre})_{i=1}^B\), are obtained via the local model $\mathbf{w}_k$. 
We approximate the KL divergence terms using the empirical average over the mini-batch samples:
\begin{multline}
D_{\rm KL}(\mathcal{P}_{cur}^{\mathbf{w}_k}(\hat{Z}_{k,cur}) \| \mathcal{P}_{pre}^{\mathbf{w}_k}(\hat{Z}_{k,pre})) \approx\\
\frac{1}{B} \sum_{i=1}^B \left[ \log \mathcal{P}_{cur}^{\mathbf{w}_k}(\hat{Z}_{ik,cur})\right. 
\left.- \log \mathcal{P}_{pre}^{\mathbf{w}_k}(\hat{Z}_{ik,pre}) \right], \label{eq:estimate KL1}
\end{multline}
\begin{multline}
D_{\rm KL}(\mathcal{P}_{pre}^{\mathbf{w}_k}(\hat{Z}_{k,pre}) \| \mathcal{P}_{cur}^{\mathbf{w}_k}(\hat{Z}_{k,cur})) \approx\\
\frac{1}{B} \sum_{i=1}^B \left[ \log \mathcal{P}_{pre}^{\mathbf{w}_k}(\hat{Z}_{ik,pre})\right.
\left.- \log \mathcal{P}_{cur}^{\mathbf{w}_k}(\hat{Z}_{ik,cur}) \right]. \label{eq:estimate KL2}
\end{multline}
For each client $k$, after approximating the KL divergence terms using mini-batches as shown in \eqref{eq:estimate KL1} and \eqref{eq:estimate KL2}, the local training loss function in Eq.~\eqref{eq:objective function} is optimized using stochastic gradient descent. The gradients are computed with respect to the local model $\mathbf{w}_k$, and $\mathbf{w}_k$ is updated iteratively to minimize the local training loss. Once the local training is complete, the server aggregates the updated local models $\mathbf{w}_{k}$ from all participating clients to produce the global model $\mathbf{w}_{cur}$.
\begin{algorithm}[t]
\caption{Kullback-Leibler and Mutual Information-FedAVG (KLMI-FedAVG).}
\begin{algorithmic}[1]
\REQUIRE number of clients $K$, learning rate $\eta$ and initialized global model $\bold{w}_{0}$.
\STATE \textbf{Procedure} \textsc{ServerExecution}($\bold{w}_{0}$)
    \FOR{each communication round $r = 1, 2, \dots$}
        \FOR{each client $k$ \textbf{in parallel}}
            \STATE The local model $\bold{w}_{k} \gets$ \textsc{ClientUpdate}($k, \bold{w}_{r-1}$);
        \ENDFOR
        \STATE Global model $\bold{w}_{r} \gets \frac{1}{K} \sum_{k=1}^K \bold{w}_{k}$;
    \ENDFOR
\STATE \textbf{return} Global model $\bold{w}_{cur} \gets \bold{w}_{r}$
\STATE \textbf{end Procedure}

\STATE \textbf{Procedure} \textsc{ClientUpdate}($k, \bold{w}$)
    \STATE $\mathcal{B} \gets$ (split dataset into batches of size $B$);
    \FOR{each local epoch $i$ from 1 to $E$}
        \FOR{each batch $b \in \mathcal{B}$}
            \STATE Estimate the KL divergences in Eq.~\eqref{eq:objective function} using Eq.~\eqref{eq:estimate KL1} and Eq.~\eqref{eq:estimate KL2};
            \STATE Minimize the local training loss function $f(\bold{w})$ in Eq.~\eqref{eq:objective function} by gradient descent;
            \STATE $\bold{w} \gets \bold{w} - \eta \nabla f(\bold{w})$;
        \ENDFOR
    \ENDFOR
    \STATE \textbf{return} $\bold{w}$ to server
\STATE \textbf{end Procedure}
\end{algorithmic}
\label{alo:our method}
\end{algorithm}

\section{Tradeoffs between System Performance and Cost}\label{sec:Trade off between system error and cost}
In the previous section, we give an upper bound on $\overline{G}$. However, in practical systems, we cannot pursue the minimization of $\overline{G}$ indiscriminately, as reducing the error is often at the expense of additional investment, training, and communication costs in FL systems. 
In this section, we aim to characterize the tradeoff between $\overline{G}$ and system cost, under which we show the Pareto front of the FL system. The boundary indicates the best performance of the FL system. We consider the tradeoff resulting from the processing capacities of the clients. Specifically, a higher processing capacity induces a higher system cost, while it leads to a faster adaptation to real-time data distribution (i.e., a smaller $\tau_2$) and hence a smaller upper bound on $\overline{G}$. 

In the following, we first formulate a Stationary Generalization Error minimization problem subject to system cost constraint. Then, we characterize the performance-cost tradeoff in closed form and hence show the Pareto front despite the challenge of nonconvex programming. Our analysis focuses on general synchronous FL algorithms, including Algorithm~\ref{alo:our method}, FedAvg~\cite{DBLP:conf/aistats/McMahanMRHA17}, and its variants (e.g., FedProx~\cite{DBLP:conf/mlsys/LiSZSTS20}).



\subsection{Stationary Generalization Error Minimization Problem}
We consider the processing capacity of the clients for their local training, denoted by $f_k$ (in FLOPS) for $k\in\mathcal{K}$, as the decision variables to understand the performance-cost tradeoff. In the following, we first present the expression of the upper bound of $\overline{G}$ and cost function concerning $f_k$. Then, we formulate the optimization problem. 

\subsubsection{The Upper Bound on $\overline{G}$}
Let $\boldsymbol{f}=(f_k, k\in\mathcal{K})$ denote the processing capacity vector. Let $T^{UL}_k$ and $T^{DL}_k$ (in seconds) denote the time that client $k$ is required for uploading and downloading the model updates in each training round, respectively. $J_k$ is the number of floating point operations required by client $k$ to process one data unit. Each local training performs $E$ epochs. $N_{k}$ is the sample size. The duration of each training round~\cite{DBLP:conf/infocom/TangW21} is 
\begin{equation}\label{eq:tau}
\phi(\boldsymbol{f}) = \max_{k\in \{1,2,..,K\}}\left\{ \frac{N_kJ_k E}{f_k} + T^{UL}_k + T^{DL}_k\right\}.
\end{equation}
Suppose the FL algorithm continues for $R$ training rounds before its convergence.
Then, the duration needed for adaptation (i.e., the duration of stage II), denoted by $\tau_2(\boldsymbol{f})$, is 
\begin{equation}
    \tau_2(\boldsymbol{f})=R\cdot\phi(\boldsymbol{f}).
\end{equation}
According to Theorem~\ref{theorem:expected generalization bound}, the upper bound on $\overline{G}$ is a linear function of $\tau_2(\boldsymbol{f})$, which we denote as $B(\tau_2(\boldsymbol{f}))$.

\subsubsection{System Cost}
Given the processing capacity $f_k$, the cost of a client $k\in\mathcal{K}$ is defined as follows~\cite{DBLP:conf/infocom/TangW21}:
\begin{multline}\label{eq:C}
C_k(f_k)=(C^{UL}_k+C_k^{DL})R+C^{invt}_k f_k\\
+C^{comp}_k (f_k)^2N_kJ_kER.
\end{multline}
Parameters $C^{UL}_k$ and $C^{DL}_k$ correspond to the model uploading and downloading costs in each training round, respectively. Term $C^{invt}_kf_k$ is the investment cost per processing capacity. Term $C^{comp}_k (f_k)^2$ represents the operating cost that client $k$ needs to pay for conducting local training during one training round~\cite{DBLP:conf/infocom/TranBZMH19}.

\subsubsection{Problem Formulation}
To characterize the performance and cost tradeoff, we  introduce a budget for system cost, denoted by $\overline{C}$. Then, the Stationary Generalization Error minimization problem is formulated as:
\begin{subequations}\label{eq:multiobjective optimization}
\vspace{-2mm}
	\begin{align}
	\mathop{\text{minimize}}\limits_{\boldsymbol{f}} & ~~\textstyle  B(\tau_2(\boldsymbol{f}))\\
	\textrm{subject to} &~~ \sum\limits_kC_k(f_k)\leq \overline{C}.\label{eq:constraint}
 \vspace{-1mm}	
 \end{align}

\end{subequations}
Problem \eqref{eq:multiobjective optimization} is a nonconvex problem. This is because $\phi(\boldsymbol{f})$ contains a max operator, which is nonconvex. Despite this, we derive the closed-form relation between cost budget $\overline{C}$ and the optimal objective value, i.e., $B^* \triangleq B(\tau_2(\boldsymbol{f}^*))$, where $\boldsymbol{f}^*$ optimizes problem~\eqref{eq:multiobjective optimization}.

\subsection{Performance-Cost Tradeoff}
We first characterize some features of the optimal solution $\boldsymbol{f}^*$. Then, we transform problem \eqref{eq:multiobjective optimization} to an equivalent problem, with which we can obtain the tradeoff. The following two lemmas present the features of the optimal solution $\boldsymbol{f}^*$ with proofs given in Appendices~C and~D. First, a solution is optimal only if the equality holds for \eqref{eq:constraint}. That is, 
\begin{lemma}\label{lemma:optimal solution is equal}
    Processing capacity vector $\boldsymbol{f}^*$ is the optimal solution to problem \eqref{eq:multiobjective optimization} only if $\sum_kC_k(f_k^*)= \overline{C}$.
\end{lemma}
Second, the $\max$ operator in \eqref{eq:tau} introduces non-convexity, which makes the problem \eqref{eq:multiobjective optimization} difficult to solve. Lemma~\ref{lemma:remove max} addresses this issue by handling the $\max$ operator.

\begin{lemma}\label{lemma:remove max}
    Let $\phi_k(f_k)=\frac{S_kD_kE}{f_k}+T^{UL}_k+T^{DL}_k$. Processing capacity vector $\boldsymbol{f}^*$ is the optimal solution only if $\phi_k(f_k^*)$ are identical for $k\in\mathcal{K}$. 
\end{lemma}


According to Lemmas \ref{lemma:optimal solution is equal} and \ref{lemma:remove max}, we can transform problem \eqref{eq:multiobjective optimization} to an equivalent problem. In the equivalent problem, $\tau_2$ is set to be the decision variable. That is, 
\begin{subequations}\label{eq:multiobjective optimization-2}
	\begin{align}
	\mathop{\text{minimize}}\limits_{\tau_2} & ~~\textstyle  B(\tau_2)\\
	\textrm{subject to} &~~ \sum\limits_kC_k(\phi_k^{-1}(\frac{\tau_2}{R}))= \overline{C}.\label{eq:constraint2}
	\end{align}
\end{subequations}
where  $\phi_k^{-1}(\cdot)$ is the inverse function of $\phi_k(\cdot)$. Problems \eqref{eq:multiobjective optimization} and \eqref{eq:multiobjective optimization-2} are equivalent in the following sense. 
\begin{proposition}\label{proposition:tradeoff}
If $\tau_2^*$ is the optimal solution to problem~\eqref{eq:multiobjective optimization-2}, then $f_k^*\triangleq\phi_k^{-1}(\frac{\tau_2^*}{R})$ for $k\in\mathcal{K}$ optimizes problem \eqref{eq:multiobjective optimization}. 
\end{proposition}


The proof of Proposition~\ref{proposition:tradeoff} can be found in Appendix~E.
Now, we focus on solving \eqref{eq:multiobjective optimization-2}. For presentation simplicity, we represent 
 $B(\tau_2)=a\tau_2+b$, because  $B(\tau_2)$ is linear in $\tau_2$ based on Theorem \ref{theorem:expected generalization bound}. As a result, we represent the cost constraint \eqref{eq:constraint2} in the following form: 
 \begin{equation}\label{eq:constraint3}
     \sum\limits_k\frac{c_k}{(\tau_2-g_k)^2}+\frac{d_k}{\tau_2-h_k}+e_k= \overline{C}, 
 \end{equation}
where $c_k$, $g_k$, $d_k$, $h_k$, and $e_k$ are constant values and can be determined based on \eqref{eq:tau} and \eqref{eq:C}. Recall that $B^*$ denotes the optimal objective value of problem \eqref{eq:multiobjective optimization-2}, we have the following theorem.  

\begin{theorem}[Performance-Cost Tradeoff]\label{theorem:pareto-frontier}
The optimal value $B^*$ and system cost $\overline{C}$ satisfy
    \begin{equation}\label{eq:relation between B and cost}
\sum\limits_{k}\frac{c_k}{(\frac{B^*-b}{a}-g_k)^2}+\frac{d_k}{\frac{B^*-b}{a}-h_k}+e_k = \overline{C}.
\end{equation}
\end{theorem}This is proven by substituting $\tau_2$ and $B^* = a\tau_2+b$ in Eq.~\eqref{eq:constraint3}. Based on Theorem~\ref{theorem:pareto-frontier}, we can plot Pareto frontier for the upper bound $B$ and FL system cost $\overline{C}$. In addition, according to Eq.~\eqref{eq:relation between B and cost}, Pareto frontier is convex, suggesting that such a frontier can be easily obtained by existing multi-objective optimization approaches in practical FL systems. 
The impact of parameters on the Pareto frontiers will be discussed in Section \ref{sec:experiment 2}.

\begin{figure}[t]
\vspace{-3mm}
\centering
\setlength{\abovecaptionskip}{-0.2pt}
\includegraphics[width=1\linewidth]{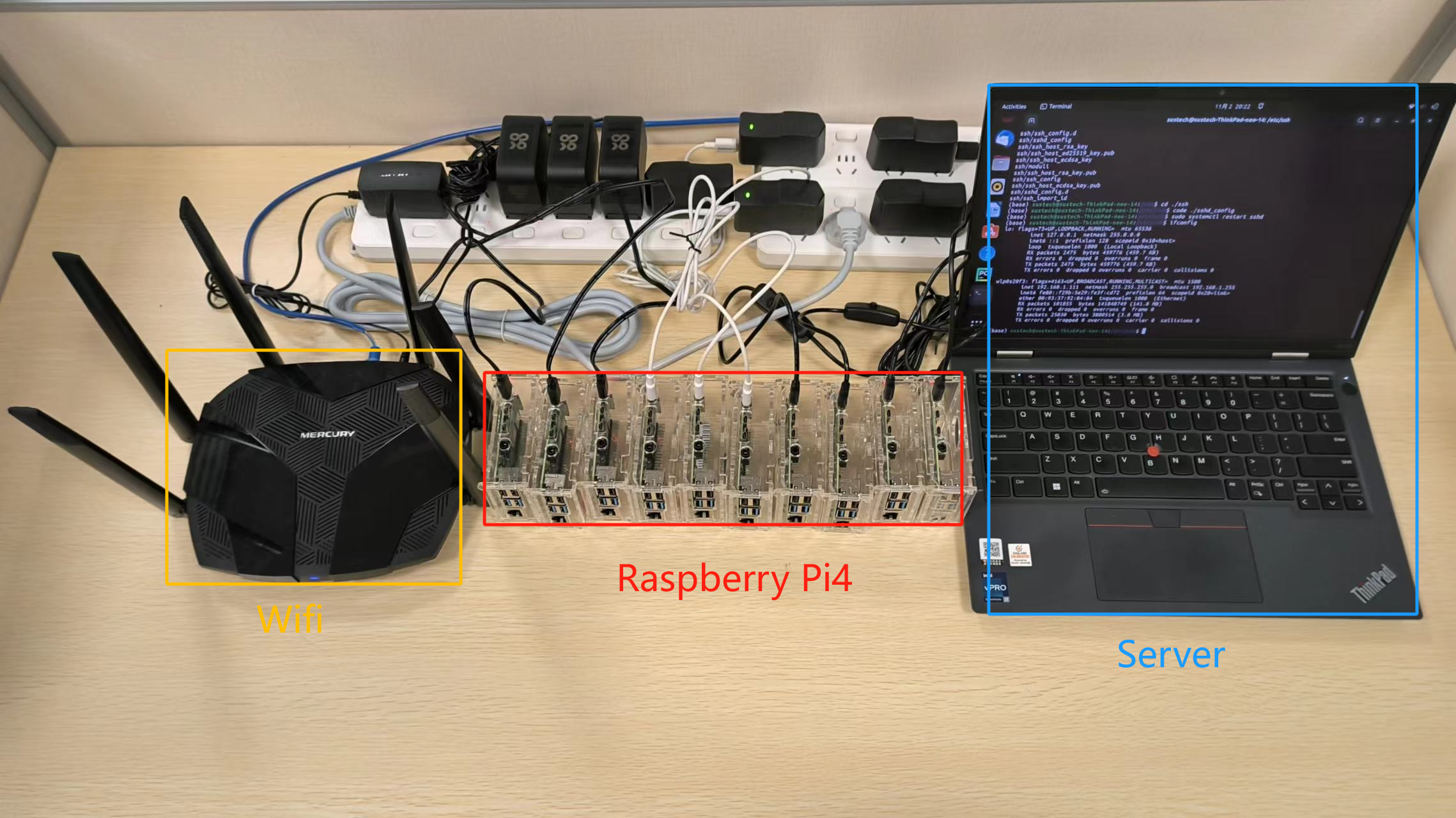}
\caption{Testbed consisting of Raspberry Pi4 devices, a laptop, and a Wi-Fi router. The Raspberry Pi4 devices are connected via the Wi-Fi network, with the laptop serving as the server for model aggregation.}\label{fig:testbed}
\vspace{-3mm}
\end{figure}

\section{Experiments}\label{sec:Experiments}
In Section~\ref{sec:experiment 1}, we built an FL testbed using Raspberry Pi4 devices to verify the upper bound of $\overline{G}$ (Theorem~\ref{theorem:expected generalization bound}). In Section~\ref{sec:experiment 3}, we compare our proposed Algorithm~\ref{alo:our method} with  baselines. In Section~\ref{sec:experiment 2}, we present Pareto frontiers according to Theorem~\ref{theorem:pareto-frontier}. Please refer to Appendix~G for additional experimental results.

\subsection{Upper Bound on the Stationary Generalization Error}\label{sec:experiment 1}
\subsubsection{Experimental Setup}\label{subsec:setup}
We built an FL system using Raspberry Pi4 devices. Our testbed is shown in Fig.~\ref{fig:testbed}. We use FedAvg algorithm~\cite{DBLP:conf/aistats/McMahanMRHA17} to aggregate the global model. We use datasets CIRCLE~\cite{DBLP:conf/sbia/GamaMCR04} and MNIST~\cite{mnist} with a 2-layer CNN model and CIFAR-10~\cite{krizhevsky2009learning} with a ResNet18 model to conduct experiments. We test CIRCLE and MNIST on our testbed, and CIFAR-10 on a server with a GPU NVIDIA A100. Specifically, CIRCLE is a synthetic datasets commonly used in concept drift experiments~\cite{DBLP:journals/tkde/LuLDGGZ19,DBLP:journals/inffus/CriadoCIRB22}. Samples are classified to two classes by the circle boundary for each concept. Different concepts are generated by moving the centre of circles. When applying MNIST, we generate different concepts by rotating the images, with a total of four concepts (i.e., 0, 90, 180, and 270 degrees). When using CIFAR-10, we use dirichlet distribution~\cite{DBLP:journals/corr/abs-1909-06335} to divide the dataset into different concepts, and the corresponding dirichlet parameter is set to 1. 

\subsubsection{Performance under Different Patterns}

\begin{figure}[t]
	\centering

	\subfloat[IID]{
		\includegraphics[width=0.47\linewidth]{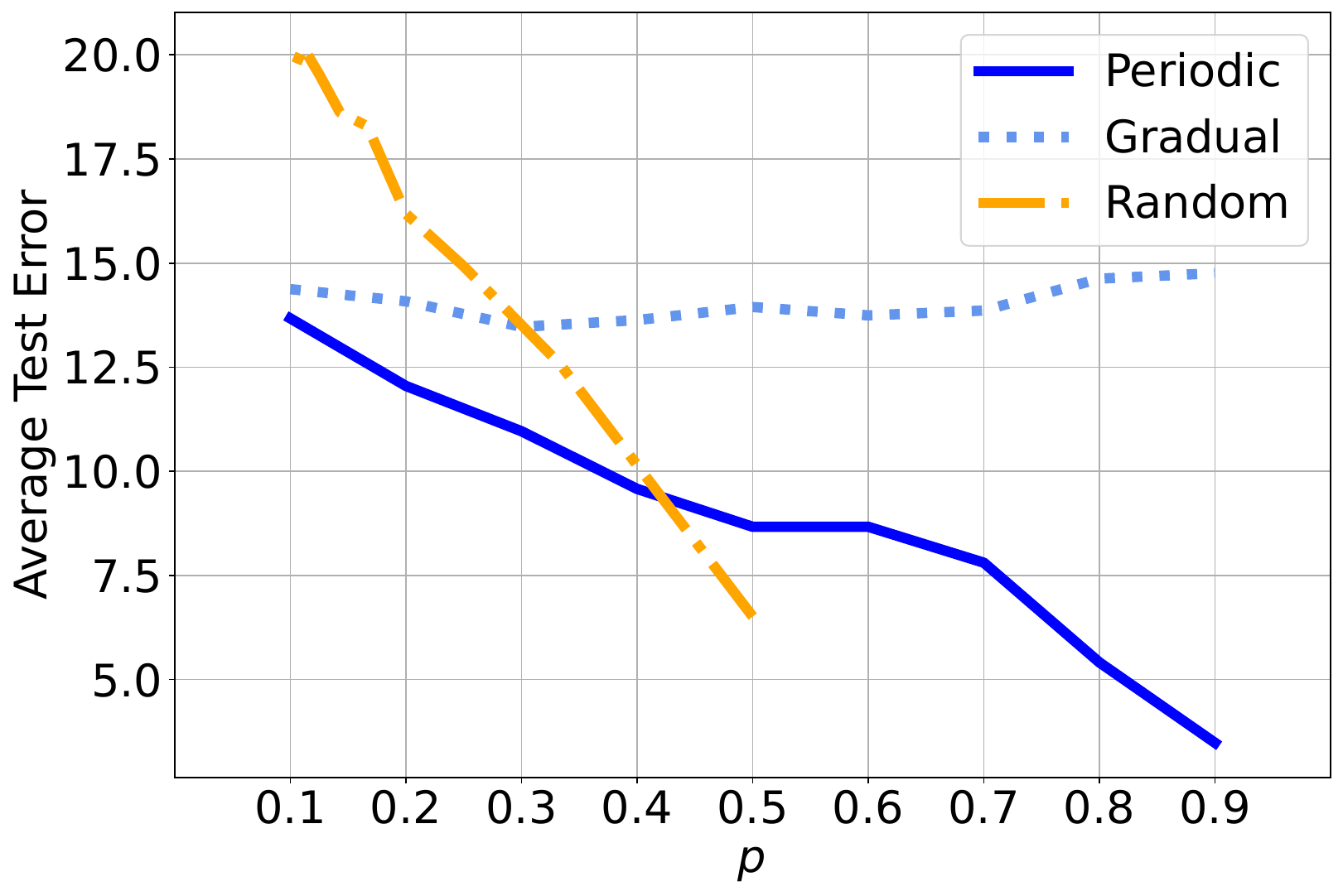}
		\label{fig:circle_iid}}
	\subfloat[Non-IID]{
		\includegraphics[width=0.47\linewidth]{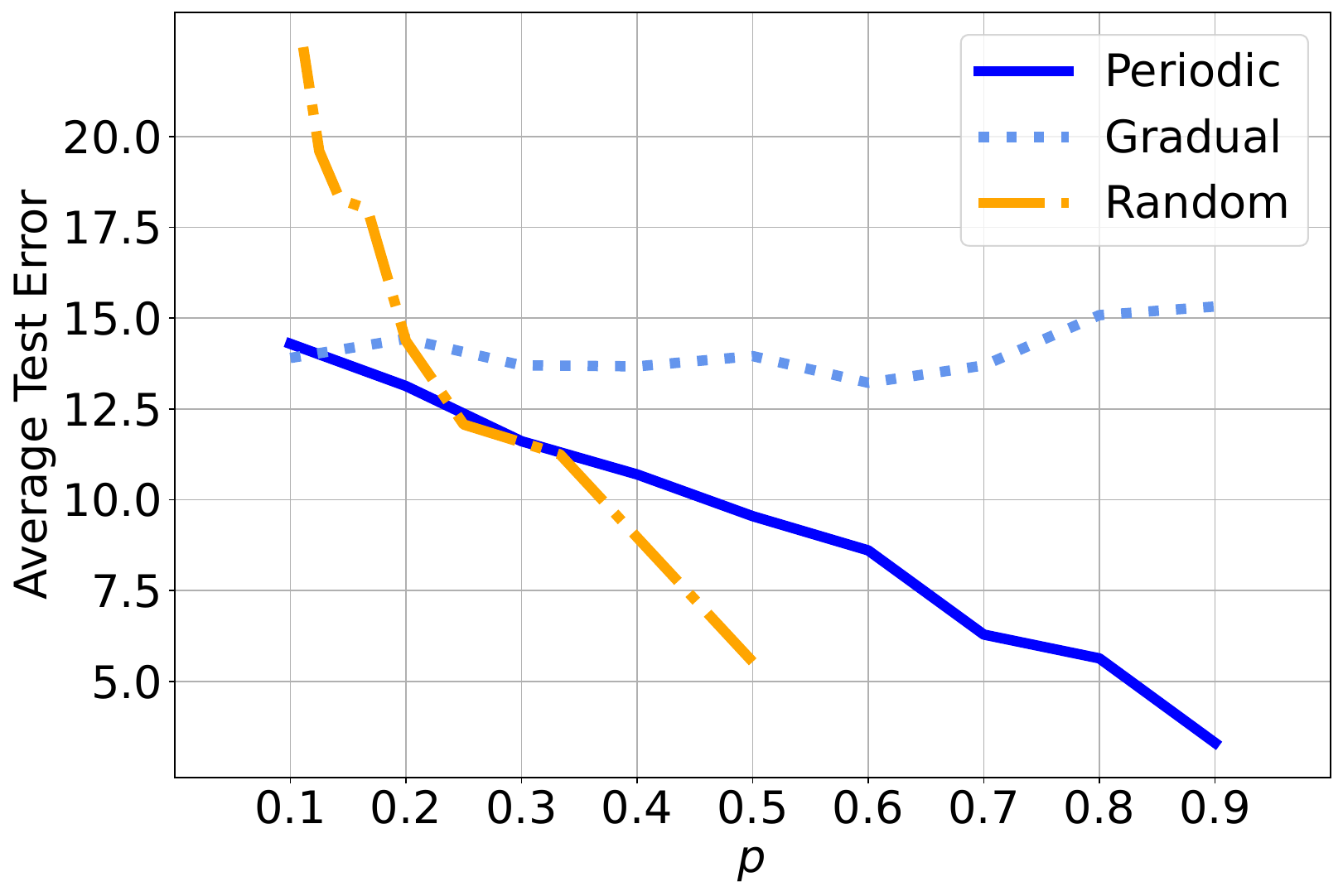}
		\label{fig:circle_non_iid}}

 \caption{CIRCLE: Impact on average test error with $D=5$.}
\label{fig:circle_simulation_compare_3_pattern_bound_p_advance_D=4}
\end{figure}

\begin{figure}[t]
	\centering
	\subfloat[IID]{
		\includegraphics[width=0.47\linewidth]{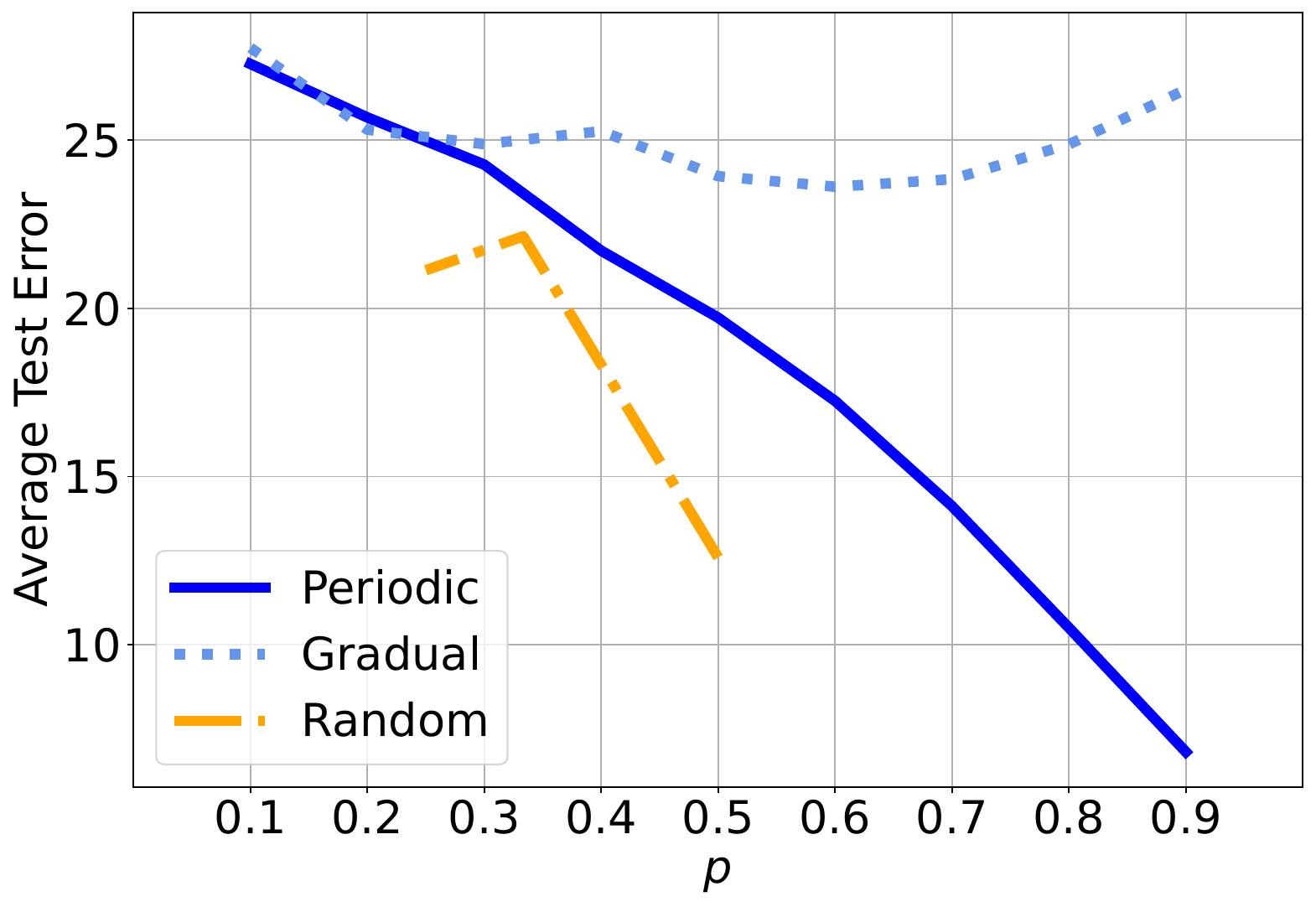}
		\label{fig:mnist_iid}}
	\subfloat[Non-IID]{
		\includegraphics[width=0.47\linewidth]{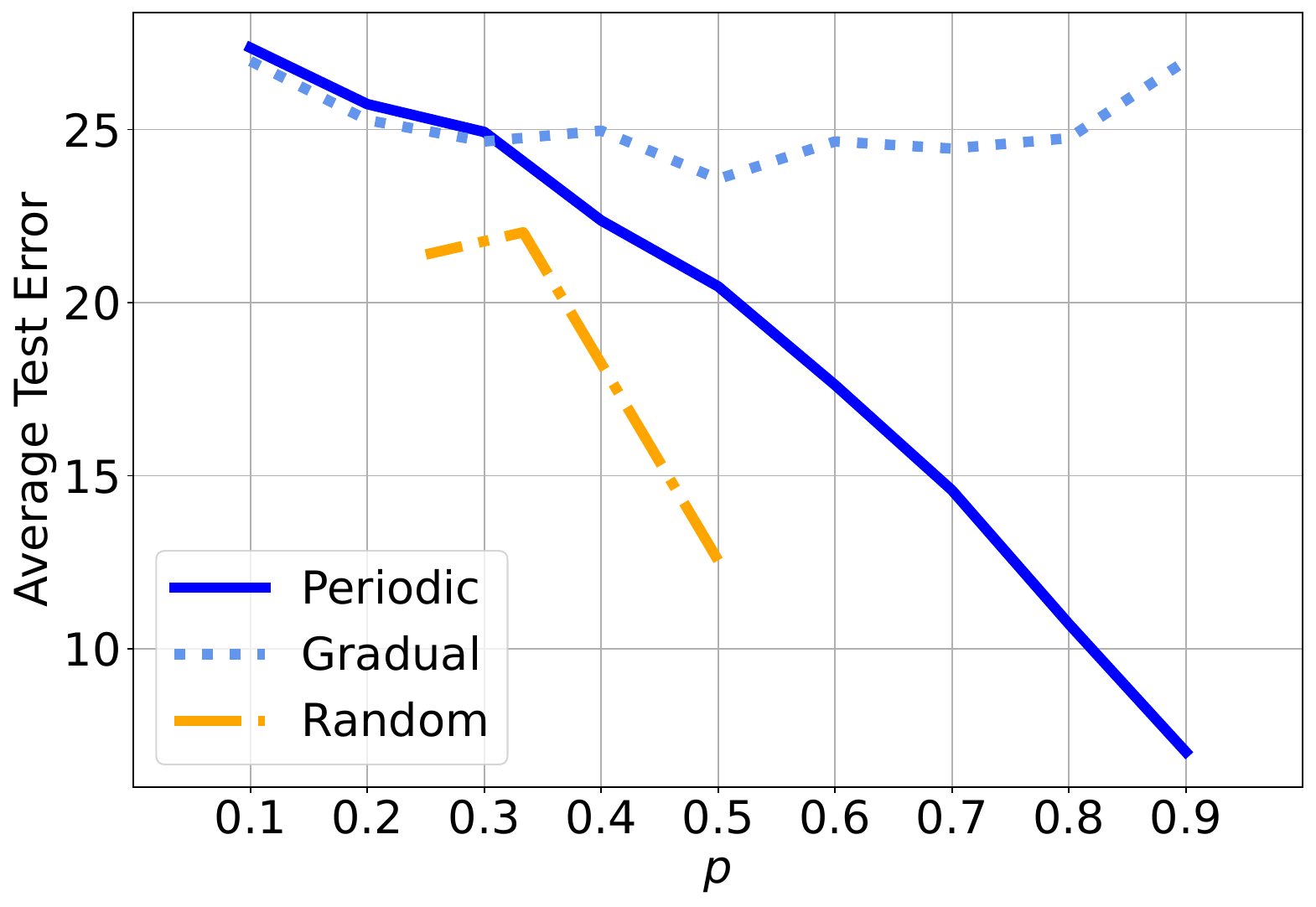}
		\label{fig:mnist_non_iid}}
 \caption{MNIST: Impact on average test error  with $D=4$ (i.e., images rotated by 0, 90, 180, 270 degrees).}
 \label{fig:mnist_simulation_compare_3_pattern_bound_p_advance_D=4}
\end{figure}

\begin{figure}[t]
	\centering
	\subfloat[IID]{
		\includegraphics[width=0.47\linewidth]{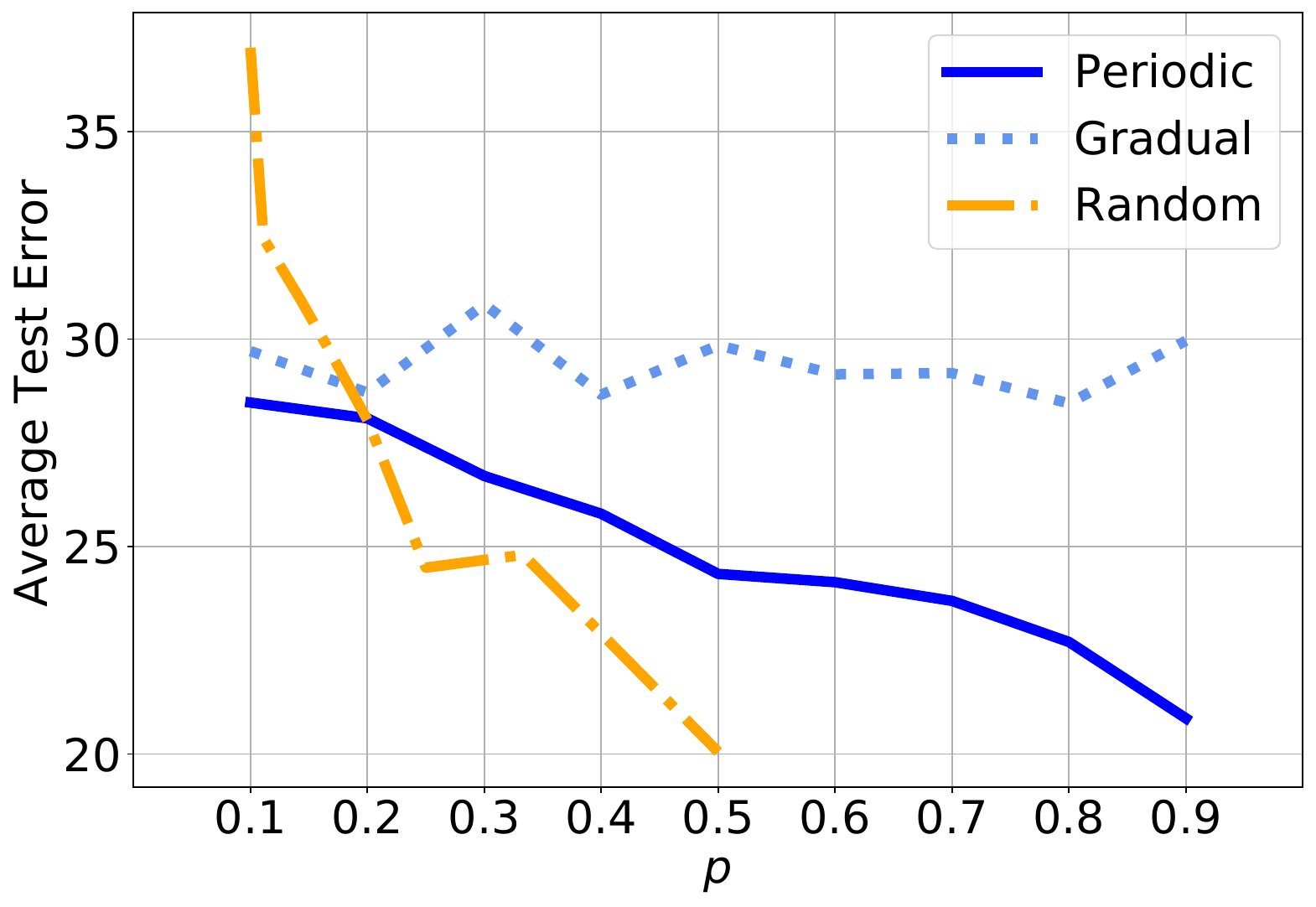}
		\label{fig:cifar_iid_D=5}}
	\subfloat[Non-IID]{
		\includegraphics[width=0.47\linewidth]{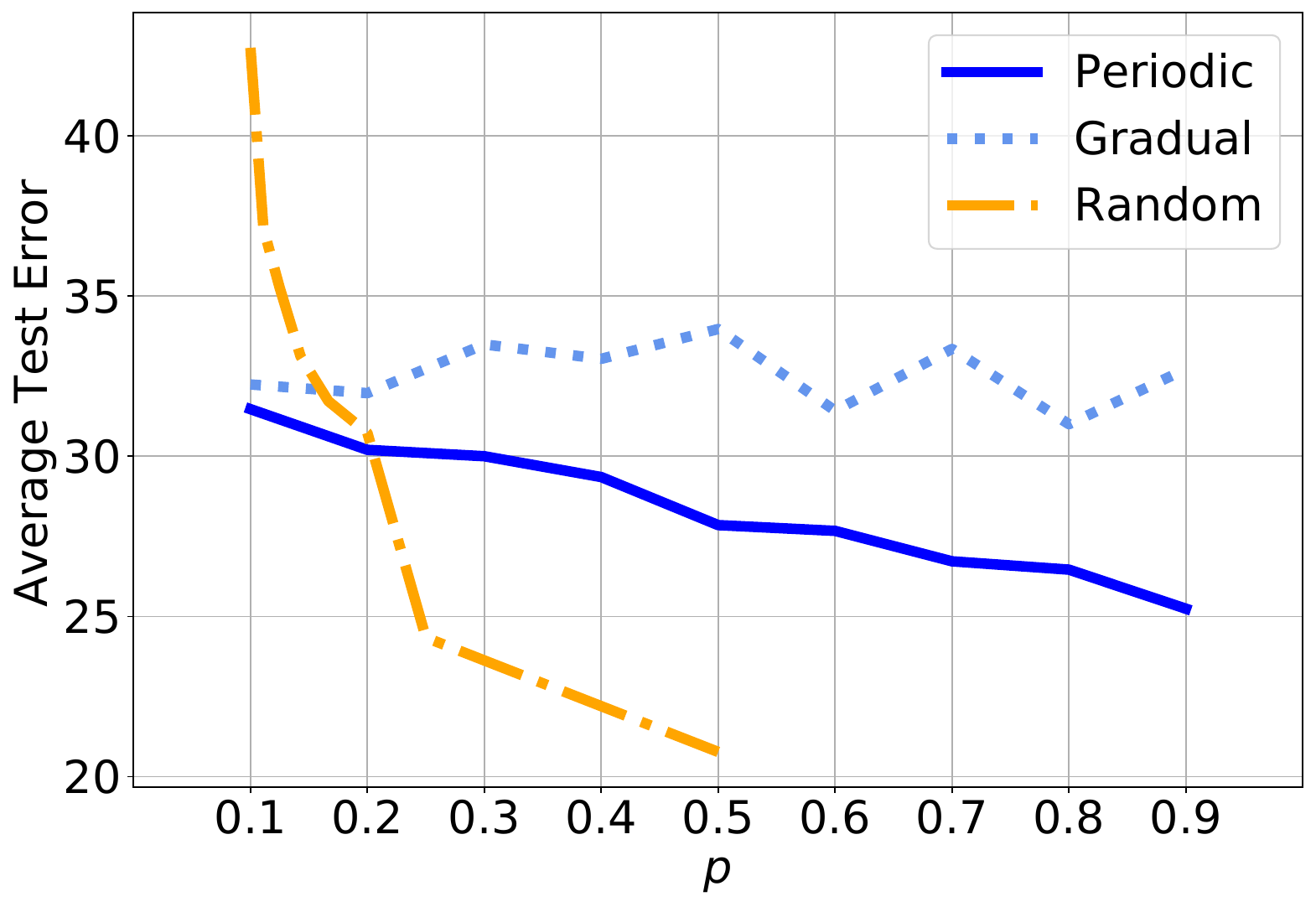}
		\label{fig:cifar_non_iid_D=5}}
 \caption{CIFAR-10: Impact on average test error with $D = 5$.} \label{fig:cifar_simulation_compare_3_pattern_bound_p_advance_D=5}
\end{figure}

\begin{figure}[t]
	\centering
	\subfloat[IID]{
		\includegraphics[width=0.47\linewidth]{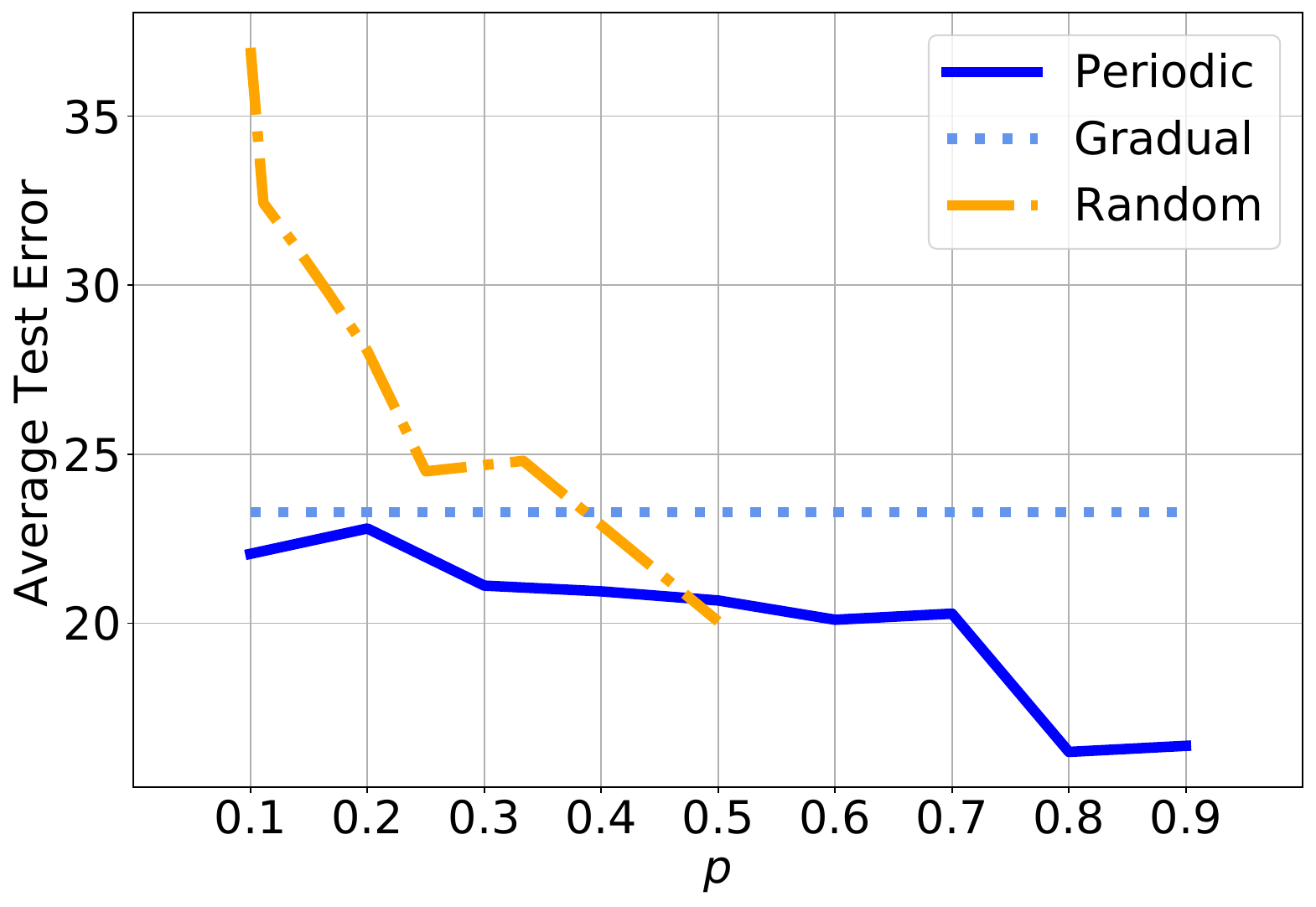}
		\label{fig:cifar_iid_D=2}}
	\subfloat[Non-IID]{
		\includegraphics[width=0.47\linewidth]{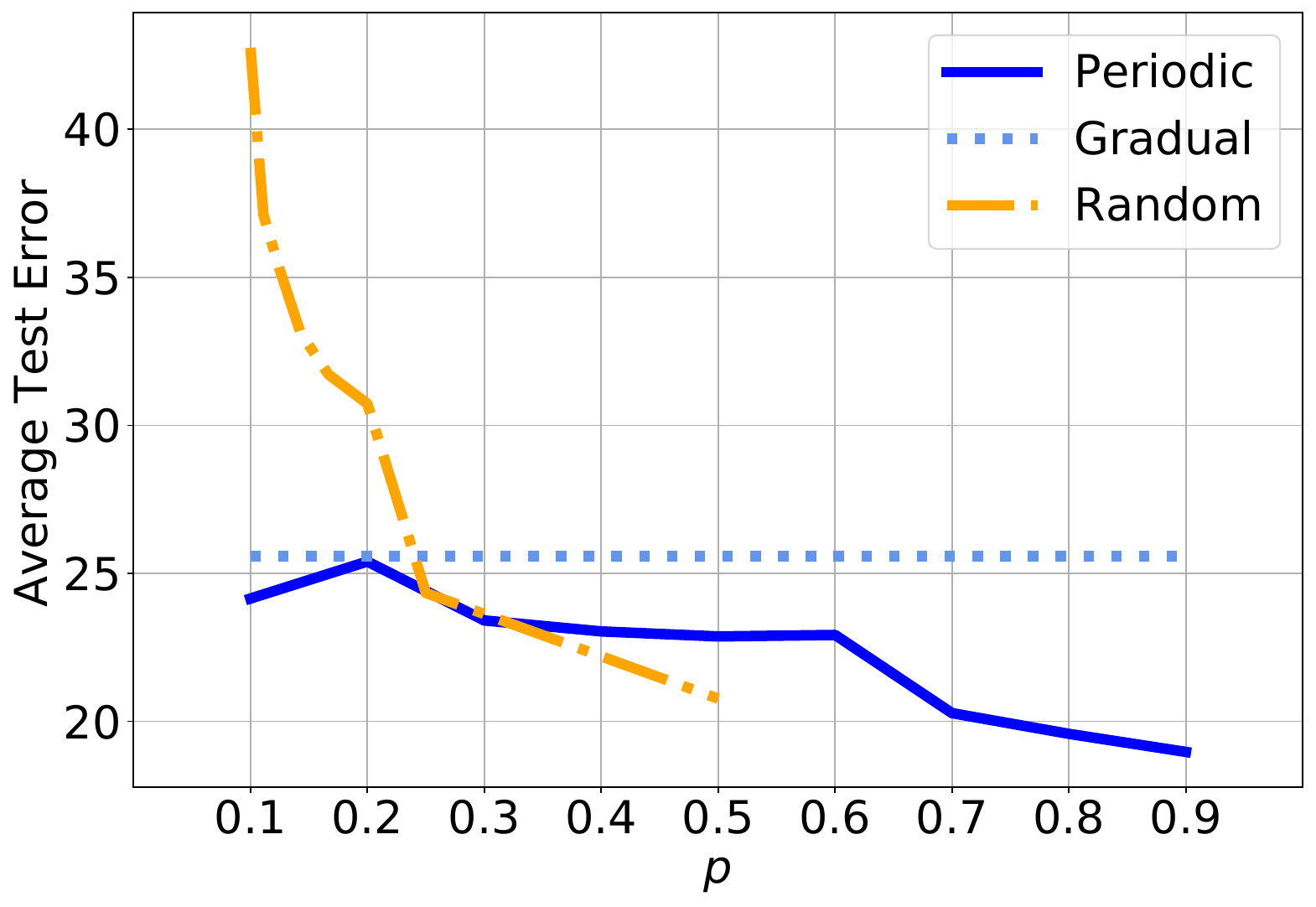}
		\label{fig:cifar_non_iid_D=2}}
 \caption{CIFAR-10: Impact on average test error  with $D=2$.}
\label{fig:cifar_simulation_compare_3_pattern_bound_p_advance_D=2}
\end{figure}



Figs.~\ref{fig:circle_simulation_compare_3_pattern_bound_p_advance_D=4}
--\ref{fig:cifar_simulation_compare_3_pattern_bound_p_advance_D=2}  present the changes in average test error over the transition probability $p$ for three patterns. In time slot \(t\), we denote \(E_{\text{test}}^{t-1,t-1}\) as the test error for \(\mathbf{w}_{t-1}\) on \(\mathcal{S}_{t-1}\) at Stage I, and \(E_{\text{test}}^{t-1,t}\) as the test error for \(\mathbf{w}_{t-1}\) on \(\mathcal{S}_t\). The average test error over \(T\) time slots is defined as \(\frac{1}{2T} \sum_{t=1}^T \left( E_{\text{test}}^{t-1, t-1} + E_{\text{test}}^{t-1, t} \right) \).


We consider 20 time slots for each experimental run. We present the results for independent and identically distributed (IID) data and non-IID data. For non-IID results, we use dirichlet distribution to divide datasets\cite{DBLP:journals/corr/abs-1909-06335}. Figs.~\ref{fig:circle_simulation_compare_3_pattern_bound_p_advance_D=4}--\ref{fig:cifar_simulation_compare_3_pattern_bound_p_advance_D=2} verify our theoretical results (see Fig.~\ref{fig:compare_3_pattern_bound_p}) on how FL system performance changes with $p$. Recall that a smaller $p$ implies a more severe concept drift phenomenon under random and periodic patterns. In contrast, changing $p$ does not affect the degree of  concept drift under gradual pattern. 
\begin{observation}[Threshold: Random and Periodic]\label{ob:threshold}
 As proven in Propositions~\ref{proposition:Threshold} and~\ref{proposition:The Effects of Parameters on the Threshold}, there exists a threshold  $p_{th} \leq 0.5$, such that $B_{r}(p)>B_{p}(p)$ for $p<p_{th}$ (see Figs.~\ref{fig:circle_simulation_compare_3_pattern_bound_p_advance_D=4},~\ref{fig:cifar_simulation_compare_3_pattern_bound_p_advance_D=5}, and~\ref{fig:cifar_simulation_compare_3_pattern_bound_p_advance_D=2}).\footnote{In MNIST dataset, there are at most four concepts. Thus, for random pattern, the transition probability $p$ cannot be lower than $1/4=0.25$ in experiments. In Fig.~\ref{fig:mnist_simulation_compare_3_pattern_bound_p_advance_D=4}, based on the trend of the random pattern, there might exists a threshold $p_{th}<0.25$.} As $D$ decreases, the threshold $p_{th}$ increases (see Figs.~\ref{fig:cifar_simulation_compare_3_pattern_bound_p_advance_D=5} and~\ref{fig:cifar_simulation_compare_3_pattern_bound_p_advance_D=2}).
\end{observation}


\begin{observation}[Performance Degradation: Random and Periodic]\label{observation:The Difference Among Patterns} In Figs.~\ref{fig:circle_simulation_compare_3_pattern_bound_p_advance_D=4}--\ref{fig:cifar_simulation_compare_3_pattern_bound_p_advance_D=2}, when $p$ is small (e.g., $p<1/D$), the FL performance of the random pattern suffers from more severe degradation than the periodic pattern; otherwise, the situation reverses. 
\end{observation} 
Observations~\ref{ob:threshold} and~\ref{observation:The Difference Among Patterns}  also hold in Fig.~\ref{fig:compare_3_pattern_bound_p}. These results suggest that when the concept drift phenomenon is significant (e.g., $p<1/D$), random pattern is more harmful to the system than periodic pattern. Perhaps counter-intuitive, when the concept drift phenomenon is relatively minor (e.g., $p>1/D$), the periodic pattern reduces the system performance at a higher degree than random pattern, requiring effective algorithms for retraining and adaptation.   

\begin{observation}[Threshold: Gradual and Periodic Patterns]\label{ob:gvsp}
   In contrast to the analytical results, there may not exist a threshold $p_{th}$ such that $B_{p}(p)>B_{g}(p)$ for $p<p_{th}$ (see~\ref{fig:cifar_simulation_compare_3_pattern_bound_p_advance_D=5} and Fig.~\ref{fig:cifar_simulation_compare_3_pattern_bound_p_advance_D=2}).
\end{observation}

Observation~\ref{ob:gvsp} holds because in experiments, the distance between data distributions might be asymmetry, i.e., $D_{\rm KL}(\pi_{i}\|\pi_{j})\neq D_{\rm KL}(\pi_{j}\|\pi_{i})$, which makes the upper bound for gradual pattern looser. Note that we do not compare the performance degradation between gradual and periodic patterns, because $p$ does not reflect the degree of concept drift under gradual  pattern. 

\begin{table}[t]\caption{Parameter settings}\label{table:para}
\resizebox{\linewidth}{!}{
\begin{tabular}{ll|ll}
\toprule
\centering
Param. & Value & Param. & Value\\
\midrule
$r$ &  0.5 &$\Delta^T\Sigma^{-1}\Delta $ & 1\\
$R$ & 150 &DL speed &78.26 Mbps~\cite{speed}\\
$E$ & 5 &UL speed & 42.06 Mbps~\cite{speed}\\
$T$& 600 seconds&Invest. cost& \$0.22 per GHz per hour~\cite{cloud}\\
$J_k$& 0.01 gigacycles&DL energy&3 joules per Mbit~\cite{DBLP:conf/imc/BalasubramanianBV09}\\
$N_k$& 600 samples&UL energy&3 joules per Mbit~\cite{DBLP:conf/imc/BalasubramanianBV09}\\
    Model size &  0.16 Mbits &Elec. rate & \$0.174 per kWh~\cite{electricity}\\
\bottomrule
\end{tabular}
}
\end{table}

\subsection{Algorithm Comparison}\label{sec:experiment 3}
This set of experiments aims to demonstrate the performance of our proposed Algorithm~\ref{alo:our method} when compared with baselines in mitigating performance degradation due to concept drift.
\subsubsection{Experimental Setup}
To fairly compare different methods, we use a slightly different setting from Section~\ref{sec:experiment 1}.
We use the Fashion-MNIST and CIFAR-10 datasets to test our method. For the Fashion-MNIST experiments, we use a simple convolutional neural network with four 3×3 convolutional layers (followed by an average pooling layer). For CIFAR-10 experiments, we use a ResNet18 model. 
In the concept drift scenarios, we generate 4 different data distributions by rotating the images. We divide a dataset into 4 sub-datasets. The images in each sub-dataset are rotated counter clockwise by $0, 90, 180, 270$ degrees respectively. The data distribution will switch randomly to one of the predefined distributions ($0$, $90$, $180$, or $270$ degrees) every 100 training rounds. The experiment consists of a total of 1200 training rounds.

We consider three FL Concept Drift Algorithms as baselines, including \textbf{AdaptiveFedAvg}\cite{DBLP:conf/ijcnn/CanonacoBMR21} and \textbf{FLASH}\cite{DBLP:conf/icml/PanchalCMMSMG23}, and one KL Guided domain adaptation algorithm \textbf{KLDA}\cite{DBLP:conf/iclr/NguyenTGTB22}. We also consider \textbf{ERM} algorithm~\cite{DBLP:conf/nips/Vapnik91} as a baseline without any drift adaptation mechanism. When new data distributions emerge every 100 training rounds, we fine-tune the model using these methods to adapt to the new data. For detailed algorithm parameter settings, please refer to Appendix~F.

\subsubsection{Results}
Table~\ref{table:Test accuracy comparison} shows the results of the Fashion-MNIST and CIFAR-10 experiments for the periodic pattern. 
\begin{observation}[The Average Test Accuracy]\label{observation:The Average Test Accuracy}
     The average test accuracy of the baselines is relatively similar across methods. In contrast, our approach outperforms the baselines significantly.
\end{observation}

The validity of Observation~\ref{observation:The Average Test Accuracy} can be attributed to the following reasons. ERM is a method that fine-tunes the model without employing advanced techniques to address performance degradation caused by concept drift. AdaptiveFedAvg and FLASH are methods designed to address the concept drift problem in FL. They detect concept drift based on the gradient update size and adaptively increase the learning rate to accelerate the model's adaptation to new concepts. However, these two adaptive learning rate methods struggle to mitigate the model's performance degradation. KLDA is a domain adaptation method based on KL divergence for centralized machine learning. 
Although this method also uses the KL divergence between representation distributions to guide training, it does not exploit the joint distribution of samples and labels. As a result, its average test accuracy is lower than that of our method.



We provide further details about the changes in test accuracy on Fashion-MNIST and CIFAR-10 datasets during training for the periodic pattern, as illustrated in Fig.~\ref{fig:example16-big} and Fig.~\ref{fig:example19-big}. Specifically, Fig.~\ref{fig:example16} and Fig.~\ref{fig:example19} illustrate training curves, while Fig.~\ref{fig:example16-1} and Fig.~\ref{fig:example19-1} highlight the moments when the model is fine-tuned and when new data appears (concept drift). To improve the clarity and compactness of Fig.~\ref{fig:example16-1} and Fig.~\ref{fig:example19-1}, ``Fine tuned" is abbreviated as ``FT", and ``Drift" is abbreviated as ``D".
\begin{observation}[Training Curves]\label{observation:Fashion-MNIST traning curve}
    Figs.~\ref{fig:example16-big} and~\ref{fig:example19-big} illustrate that when drift occurs, our method results in a smaller drop in model accuracy, showing that it effectively addresses concept drift in FL.
\end{observation}

\begin{table}[t]
\centering
\caption{Average test accuracy for Fashion-MNIST and CIFAR-10 (periodic pattern).}
\begin{tabular}{lcc}
\toprule
Algorithms   & Fashion-MNIST&CIFAR-10\\ 
\midrule
ERM    &53.02\% & 66.93\% \\ 
AdaptiveFedAvg   &49.26\% & 64.91\% \\ 
FLASH   &51.72\% & 59.38\%\\ 
KLDA   & 51.84\% & 65.78\% \\ 
Ours   & \textbf{61.23}\% & \textbf{72.58}\%\\ 
\bottomrule
\end{tabular}
\label{table:Test accuracy comparison}
\end{table}

\begin{figure}[t]
	\centering
	\subfloat[]{
		\includegraphics[width=1\linewidth]{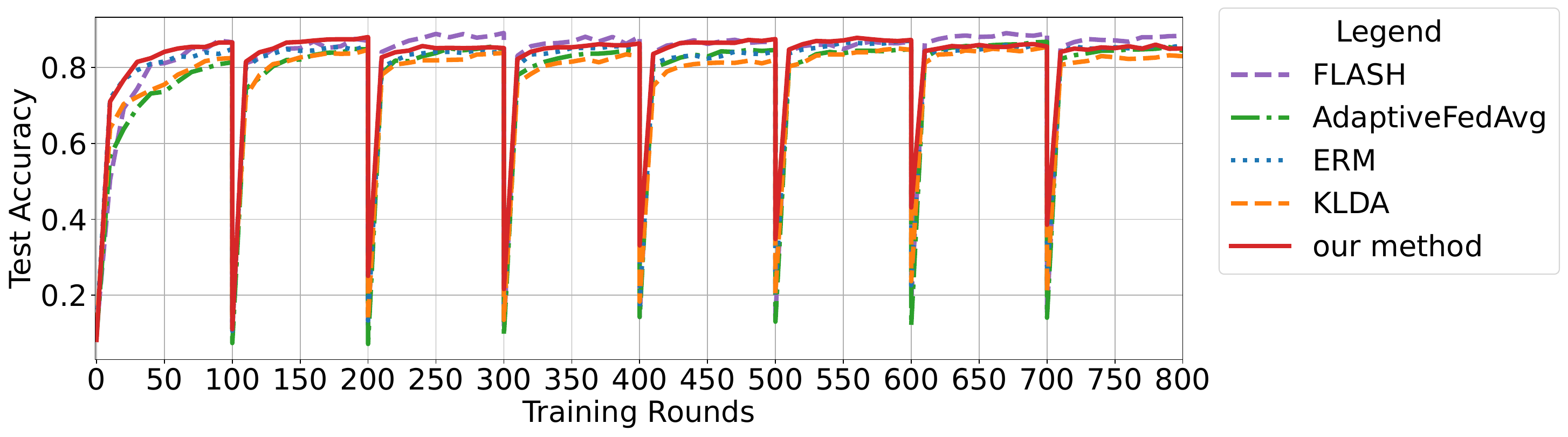}
		\label{fig:example16}}\\
        \centering
	\subfloat[]{
		\includegraphics[width=1\linewidth]{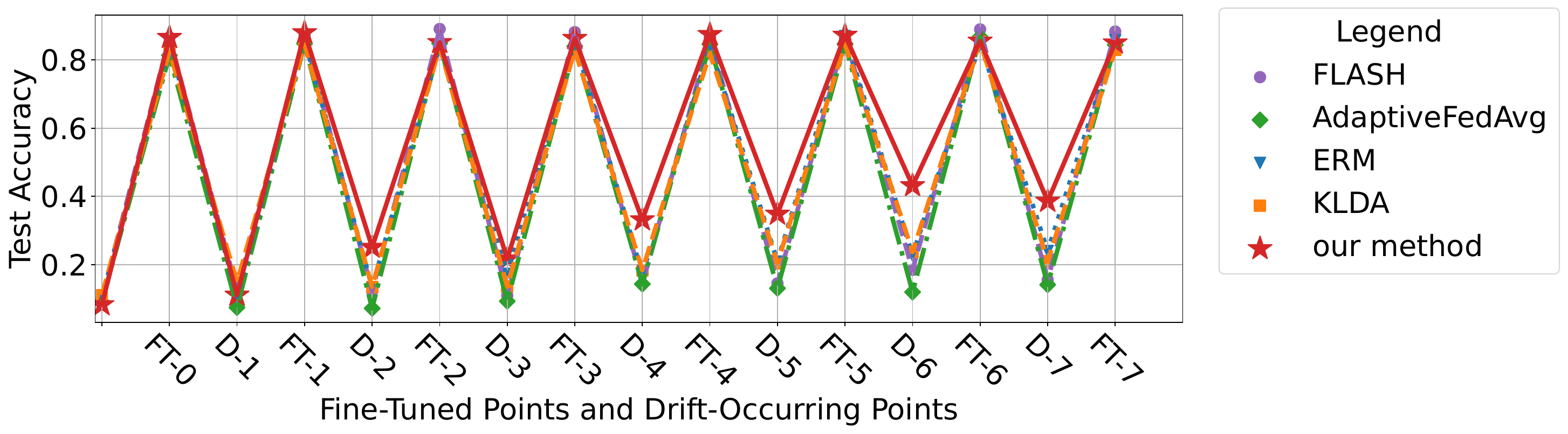}
		\label{fig:example16-1}}
 \caption{Accuracy curves for Fashion-MNIST dataset (periodic pattern): (a)~Training curve; (b) test accuracy of retrain and drift.}
 \label{fig:example16-big}

\end{figure}

\begin{figure}[t]
	\centering
	\subfloat[]{
		\includegraphics[width=1\linewidth]{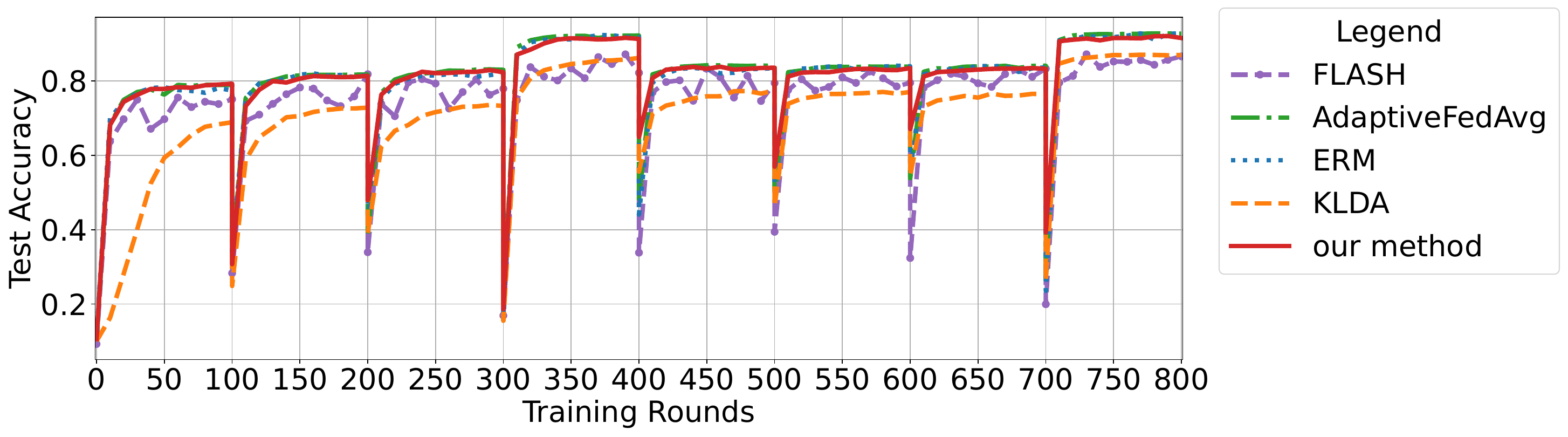}
		\label{fig:example19}}\\
        \centering
	\subfloat[]{
		\includegraphics[width=1\linewidth]{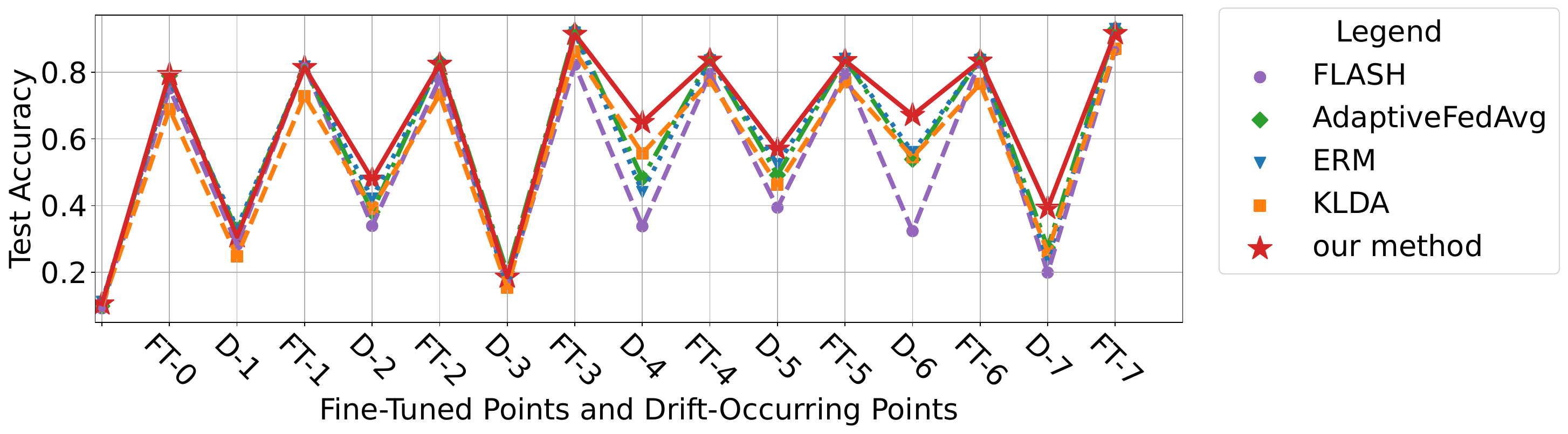}
		\label{fig:example19-1}}
 \caption{Accuracy curves for CIFAR-10 dataset (periodic pattern): (a) Training curve; (b) test accuracy of retrain and drift.}
 \label{fig:example19-big}
\end{figure}


\subsection{Tradeoff between Performance and Cost}\label{sec:experiment 2}
\begin{figure}[t]
	\centering
	\subfloat[$K$]{
		\includegraphics[width=0.47\linewidth]{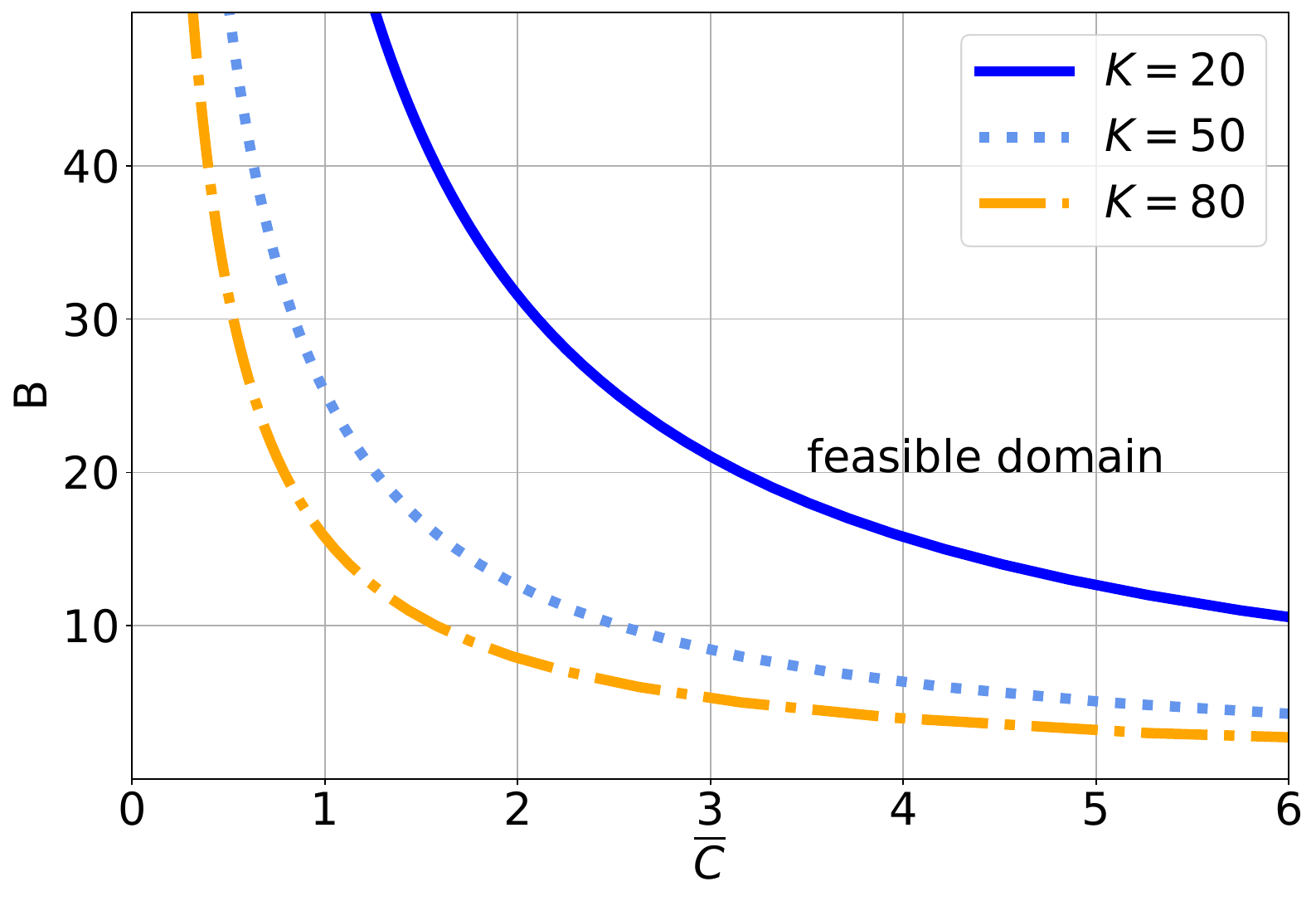}
		\label{fig:tradeoff_B_C_periodic_K}}
  	\subfloat[$D$]{  
		\includegraphics[width=0.47\linewidth]{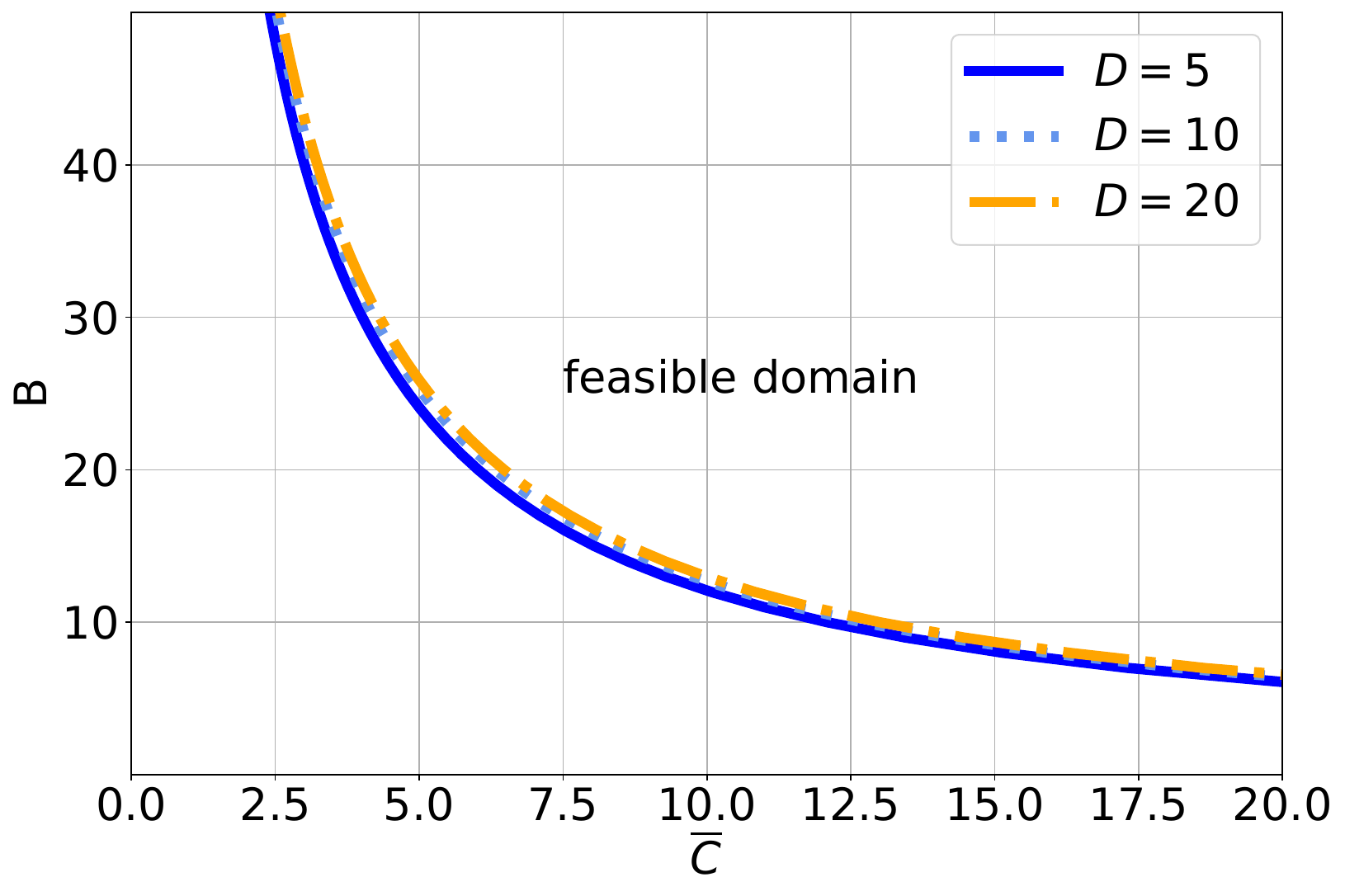}
		\label{fig:tradeoff_B_C_periodic_D}}
 \caption{Pareto frontiers for periodic pattern.}
 \label{fig:tradeoff_B_C_periodic}
\end{figure}

\begin{figure}[t]
	\centering
	\subfloat[$K$]{
		\includegraphics[width=0.47\linewidth]{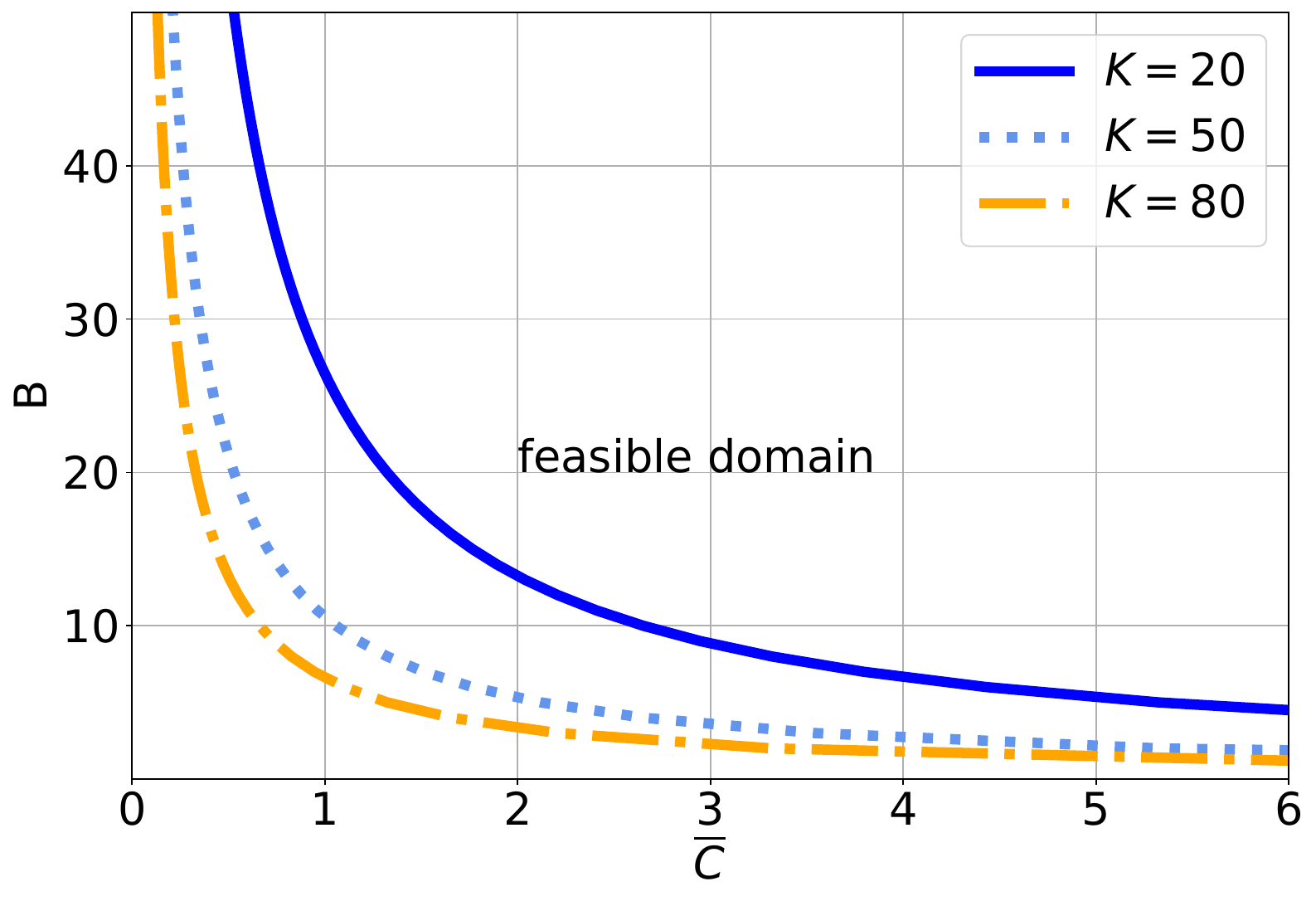}
		\label{tradeoff_B_C_gradual_K}}
  	\subfloat[$D$]{
		\includegraphics[width=0.47\linewidth]{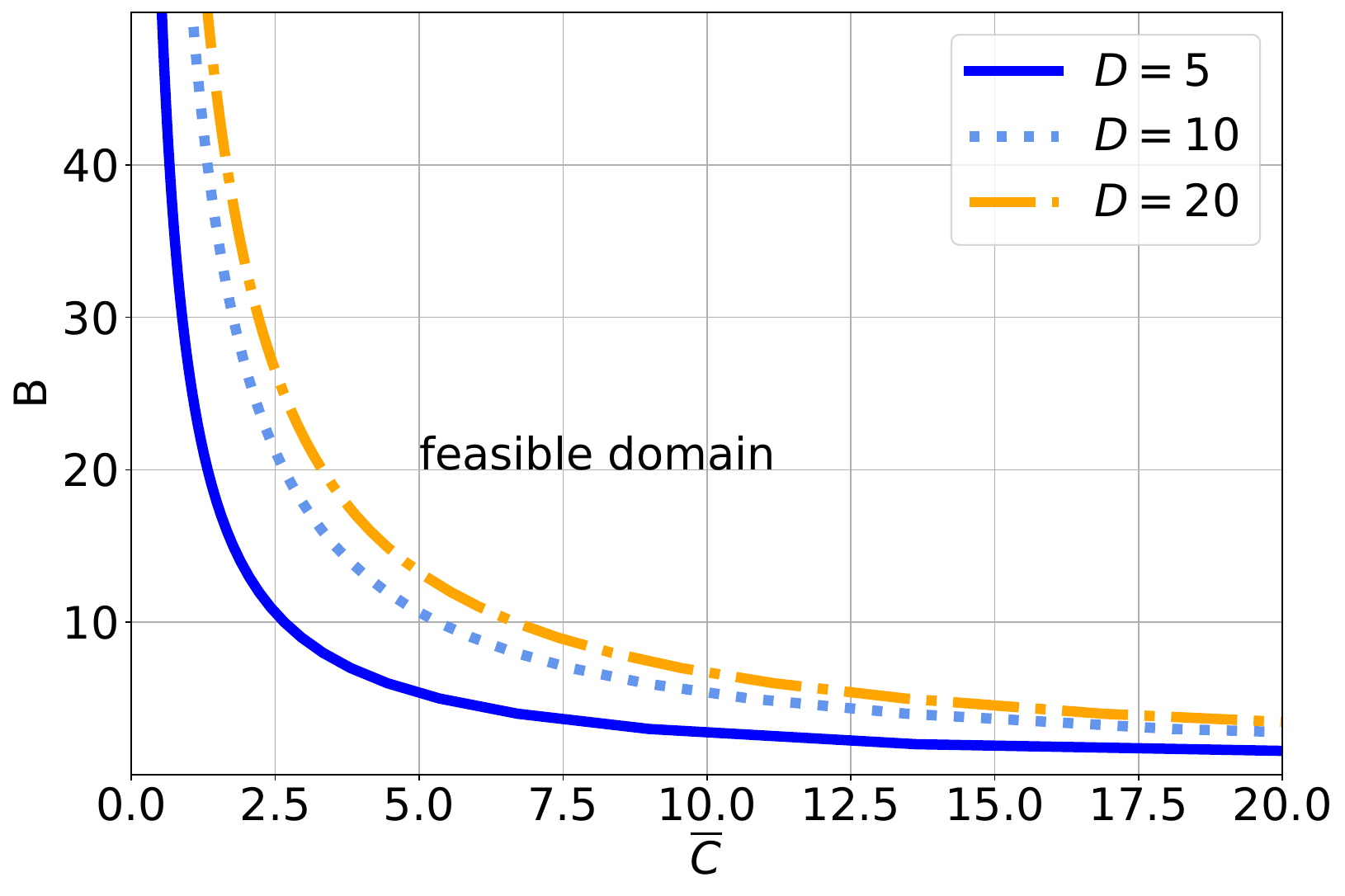}
		\label{tradeoff_B_C_gradual_D}}
 \caption{Pareto frontiers for gradual pattern.}
 \label{fig:tradeoff_B_C_gradual}
\end{figure}

\begin{figure}[t]
	\centering
	\subfloat[$K$]{
		\includegraphics[width=0.47\linewidth]{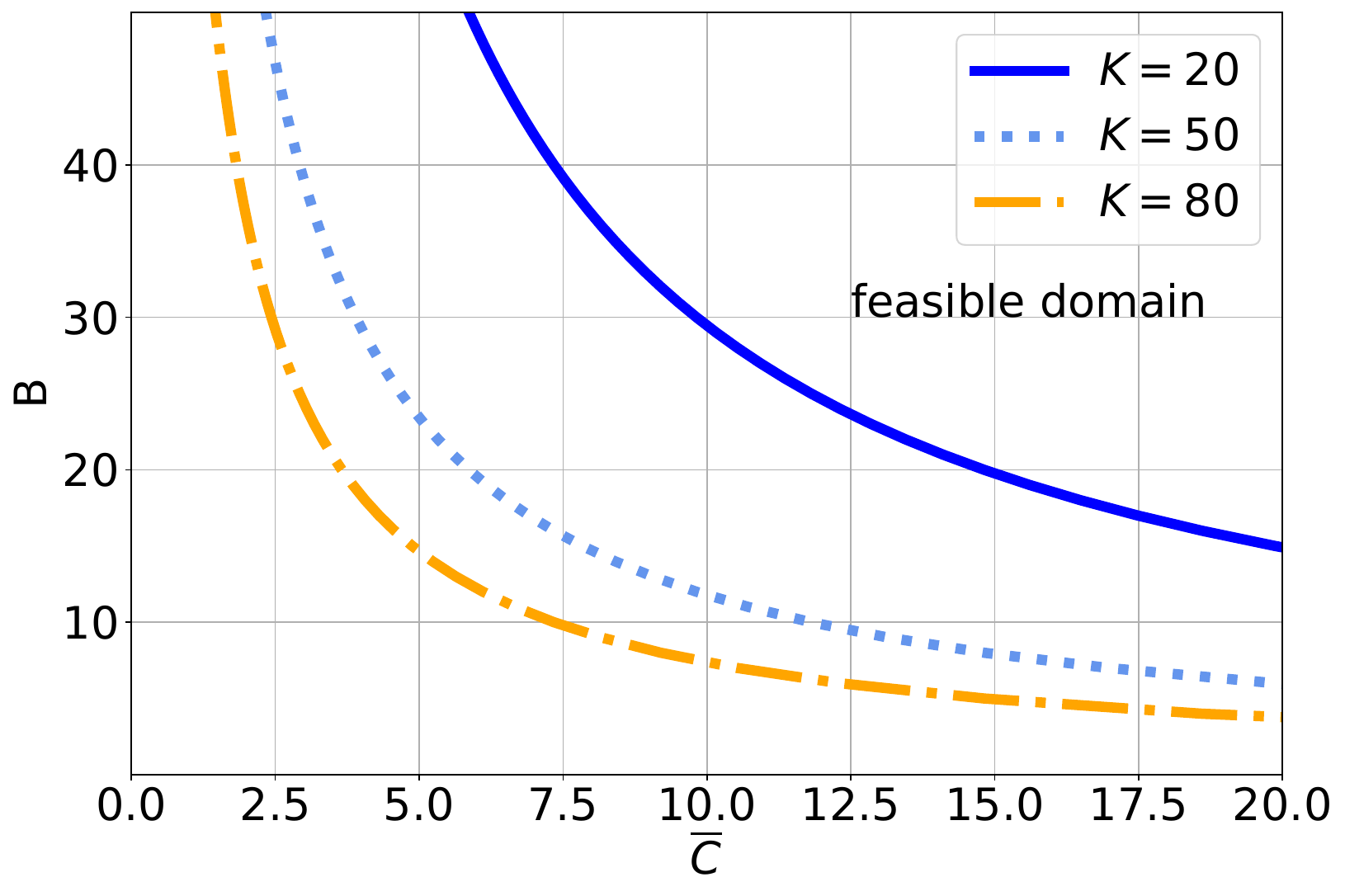}
		\label{tradeoff_B_C_random_K}}
	\subfloat[$D$]{
		\includegraphics[width=0.47\linewidth]{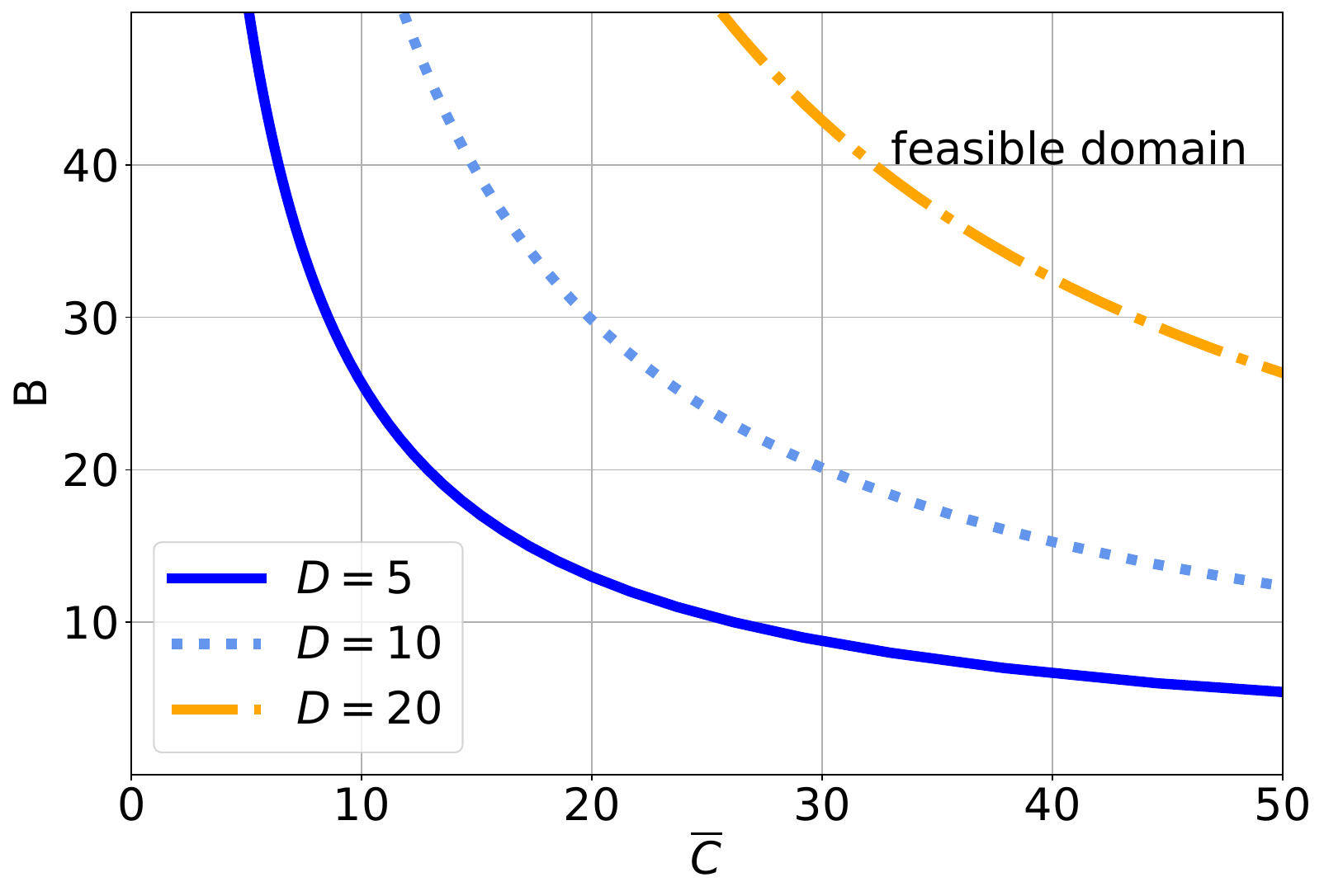}
		\label{tradeoff_B_C_random_D}}
 \caption{Pareto frontiers for random pattern.}
 \label{fig:tradeoff_B_C_random}
\end{figure}

We present the results on the performance-cost tradeoff  (i.e., the tradeoff between the Stationary Generalization Error and system cost) according to Theorem~\ref{theorem:pareto-frontier} with the parameter settings given in Table~\ref{table:para}. Figs.~\ref{fig:tradeoff_B_C_periodic},~\ref{fig:tradeoff_B_C_gradual}, and~\ref{fig:tradeoff_B_C_random} display the Pareto frontiers for the periodic, gradual, and random patterns, respectively. 

\begin{observation}[The impact of Parameters on the Pareto Frontiers]\label{observation:The impact of Parameters on the Pareto Frontiers}
When $K$ increases or $D$ decreases, the performance-cost tradeoff becomes better, i.e., the tradeoff curves move towards the origin such that a lower error can be achieved under a given cost budget.
\end{observation}
Observation~\ref{observation:The impact of Parameters on the Pareto Frontiers} shows that when the number of clients increases or when the system contains a smaller set of possible drifts, we could use lower cost budget to achieve better FL system performance.

\begin{observation}[Impact of Patterns]\label{observation:The Significance of Impact of Parameters on the Pareto Frontiers}
Increasing $K$ has the strongest impact on improving the tradeoff under random pattern, followed by periodic pattern, and then gradual pattern. 
\end{observation}
Considering addressing the concept drift phenomenon by increasing the number of clients. From Observation~\ref{observation:The Significance of Impact of Parameters on the Pareto Frontiers}, such an approach is more effective under random pattern, followed by periodic pattern and then gradual pattern. This result further suggests the necessity of understanding the concept drift patterns for retraining and adaptation algorithm design. 


\section{Conclusion}\label{sec:conclusion}
In this study, we proposed a theoretical framework for FL under concept drift. We modeled changes in data distribution and introduced $\overline{G}$ to quantify system performance. Furthermore, we established an upper bound on $\overline{G}$ in terms of mutual information and KL divergence. To address performance degradation caused by concept drift, we developed an algorithm to mitigate its effects. To explore the performance-cost tradeoff, we formulated and solved the upper bound minimization problem, deriving the Pareto frontiers for performance and cost.
Our theoretical and experimental results demonstrate that drift patterns significantly impact system performance, with random patterns proving particularly detrimental. Additionally, empirical evaluations show that our method consistently outperforms relevant baselines. These findings suggest that KL divergence and mutual information can serve as effective guides for training FL systems under concept drift.
Future research could focus on developing pattern-aware FL retraining or adaptation algorithms to further enhance the performance and adaptability of FL systems. For instance, leveraging the statistical characteristics of data to estimate concept drift patterns could facilitate informed retraining decisions that balance cost and system performance.

\appendices

\appendices

\section{Proof of Theorem 1}\label{sec:appendix Proof of Theorem 1}


\setcounter{theorem}{0} 
\begin{theorem}[Upper Bound of the Stationary Generalization Error]
The Stationary Generalization Error is bounded as:
\vspace{-2mm}
\begin{multline}  \label{eq:system error bound}
\overline{G}\leq \frac{1}{K^2}\sum\limits_{\pi_{pre},\pi_{cur},\pi_{nxt}}p(\pi_{pre})p(\pi_{cur}|\pi_{pre})p(\pi_{nxt}|\pi_{cur})\\
\left(\tau_1\sum\limits_{k=1}^{K}\left(\frac{\alpha}{N_{cur}}\!\!\sum\limits_{n=1}^{N_{cur}}\!\psi^{*-1}(I(\mathbf{w}_{cur},Z_{nk,cur}))\right.+\frac{1-\alpha}{N_{pre}}\right.\\
\left.\sum\limits_{n=N_{cur}+1}^{N}\!\!\!\!\!\psi^{*-1}(I(\mathbf{w}_{cur},Z_{nk,pre})+D_{\rm KL}(\pi_{pre}\|\pi_{cur}))\!\right)\\
\!\!\!\!\!\!\!+\tau_2\sum\limits_{k=1}^{K}\left(\frac{\alpha}{N_{cur}}\sum\limits_{n=1}^{N_{cur}}\psi^{*-1}(I(\mathbf{w}_{cur},Z_{nk,cur})\right.\\
\left.+D_{\rm KL}(\pi_{cur}\|\pi_{nxt}))\left.+\frac{1-\alpha}{N_{pre}}\right.\right.\\
\left.\left.\sum\limits_{n=N_{cur}+1}^{N}\!\!\!\!\!\psi^{*-1}(I(\mathbf{w}_{cur},Z_{nk,pre})\!+\!D_{\rm KL}(\pi_{pre}\|\pi_{nxt}))\!\right)\!\right).\!\!\!\!\!
\end{multline}
\end{theorem}
\begin{proof}

To bound the value of $\overline{G}$ in Eq.~(6), 
we need to bound $G(\pi_{cur},\pi_{nxt})$ in Eq.~(5). In particular, we need to bound $\mathbb{E}_{\bold{w}S}[G_1(\pi_{cur},\bold{w}_{cur})]$ and $\mathbb{E}_{\bold{w}S}[G_2(\pi_{nxt},\bold{w}_{cur})]$, respectively. By substituting these bounds into Eq.~(5) and then Eq.~(6), Theorem~1 can be proven.

We first bound $\mathbb{E}_{\bold{w}S}[G_1(\pi_{cur},\bold{w}_{cur})]$. For simplicity, we employ notation $Z_{n}$, $ n=1,2,...,N$. For $n=1,2,...,N_{cur}$, $Z_n\sim \pi_{cur}$. For $n=N_{cur}+1, ..., N$, $Z_n\sim \pi_{pre}$. By substituting Eq.~(1) and Eq.~(2) to Eq.~(3) and according to the properties of expectation, $\mathbb{E}_{\bold{w}S}[G_1(\pi_{cur},\bold{w}_{cur})]$ can be represented as:
\vspace{-2mm}
\begin{multline}
\label{eq:prove stage I-2}
\frac{1}{NK}\sum\limits_{k=1}^{K}\bigg(\sum\limits_{n=1}^{N_{cur}}\frac{\alpha N}{N_{cur} }\mathbb{E}_{\mathbf{w}_{cur}Z_{n}}\left[\mathbb{E}_{Z_{cur}}\!\!\left[\ell(\mathbf{w}_{cur},Z_{cur})\right]\right.\\
\left.-\ell(\mathbf{w}_{cur},Z_{nk,cur})\right]\\
+\sum\limits_{n=N_{cur}+1}^{ N}\frac{(1-\alpha)N}{N_{pre}}
\mathbb{E}_{\mathbf{w}_{cur}Z_{n}}\left[\mathbb{E}_{Z_{cur}}\!\!\left[
\ell(\mathbf{w}_{cur},Z_{cur})\right]\right.\\
\left.-\ell(\mathbf{w}_{cur},Z_{nk,pre})\right]\bigg).
\end{multline}

To bound the first term of Eq.~\eqref{eq:prove stage I-2}, we link it with the definition of KL divergence in Eq.~(10).
Let $U=P_{\bold{w}_{cur}Z_n}$, $V=P_{\bold{w}_{cur}}\otimes \pi_{cur}$ in Eq.~(10), and define $f\triangleq\lambda \ell\left(\bold{w}_{cur}, Z_n\right)$ for some $\lambda$. The joint probability between an instance $Z_n$ and $\bold{w}_{cur}$ is $P_{\bold{w}_{cur}Z_n}(\bold{w}_{cur},Z_n)$. Then, we have 
\begin{multline}\label{eq:prove stage I-3}
\mathbb{E}_{\bold{w}_{cur} Z_{n}}\left[\lambda \ell\left(\bold{w}_{cur}, Z_{n}\right)\right] \\
\leq\! D_{\rm KL}\!\left(P_{\bold{w}_{cur} Z_{n}} \| P_{\bold{w}_{cur}} \!\otimes\! \pi_{cur}\right)\\
+\log \mathbb{E}_{w_{cur}\otimes \pi_{cur}}\!\left[e^{\lambda \ell\left(\bold{w}_{cur}, Z_{n}\right)}\!\right].
\end{multline}
We can bound the first term of Eq. \eqref{eq:prove stage I-2} using Eq. (\ref{eq:prove stage I-3}) and Eq.~(7) for $n=1, \ldots, N_{cur}$:
\vspace{-2mm}
\begin{equation} \label{eq:prove stage I-4}
\begin{aligned}
&\mathbb{E}_{\bold{w}_{cur} Z_{n}}\left[\mathbb{E}_{Z_{cur}}[\ell(\bold{w}_{cur},Z_{cur})]-\ell\left(\bold{w}_{cur}, Z_{nk,cur}\right)\right] \\
&\leq-\frac{1}{\lambda}\left(D_{\rm KL}\left(P_{\bold{w}_{cur} Z_{nk,cur}} \| P_{\bold{w}_{cur}} \otimes \pi_{cur}\right)+\psi(\lambda)\right)\\
&=-\frac{1}{\lambda}\left(I\left(\bold{w}_{cur} ; Z_{nk,cur}\right)+D_{\rm KL}\left(P_{Z_{n}} \| \pi_{cur}\right)+\psi(\lambda)\right)\\
&\stackrel{(a)}=-\frac{1}{\lambda}\left(I\left(\bold{w}_{cur} ; Z_{nk,cur}\right)+\psi(\lambda)\right).
\end{aligned}
\end{equation}
Here, (a) holds because $P_{Z_n}=\pi_{cur}$ when $Z_n\sim \pi_{cur}$ in Stage I for $n=1, \ldots, N_{cur}$.
Further, by minimizing the right hand side (RHS) of Eq. \eqref{eq:prove stage I-4} and substituting Eq.~(9), Eq.~\eqref{eq:prove stage I-4} can be bounded by:
\begin{equation}\label{eq:minimizing the RHS}
\begin{aligned}
&\mathbb{E}_{\bold{w}_{cur} Z_{n}}\left[\mathbb{E}_{Z_{cur}}\left[\ell(\bold{w}_{cur},Z_{cur})\right]-\ell\left(\bold{w}_{cur}, Z_{nk,cur}\right)\right]\\
&\leq \min _{\lambda \in\left[0,b\right]} \frac{1}{\lambda}\left(I\left(\bold{w}_{cur} ; Z_{nk,cur}\right)+\psi(-\lambda)\right)\\
&=\psi^{*-1}\left(I\left(\bold{w}_{cur} ; Z_{nk,cur}\right)\right).
\end{aligned}
\end{equation}

To prove the bound of the second term in Eq. \eqref{eq:prove stage I-2}, for $n=N_{cur}+1, \ldots, N$ and $P_{Z_n}=\pi_{pre}$, we follow the same proof path as above and obtain:
\begin{equation}\label{eq:minimizing the RHS-2}
\begin{aligned}
&\mathbb{E}_{\bold{w}_{cur} Z_{n}}\left[\mathbb{E}_{Z_{cur}}\left[\ell(\bold{w}_{cur},Z_{cur})\right]-\ell(\bold{w}_{cur}, Z_{nk,pre})\right]\\
&\leq \psi^{*-1}\left(I(\bold{w}_{cur} ; Z_{nk,pre})+D_{\rm KL}(\pi_{pre} \| \pi_{cur})\right).
\end{aligned}
\end{equation}

Note that since $\pi_{pre}$ and $\pi_{cur}$ may be different due to the transition between states, then the KL divergence $D_{\rm KL}(\pi_{pre}\| \pi_{cur})\neq 0$, i.e., $D_{\rm KL}(\pi_{pre}\| \pi_{cur})$ cannot be canceled out in Eq. \eqref{eq:prove stage I-4}.
Now we bound $\mathbb{E}_{\bold{w}S}[G_1(\pi_{cur}, \bold{w}_{cur})]$ by summing over $n$ and $k$ using the upper bounds in Eq. \eqref{eq:minimizing the RHS} and Eq. \eqref{eq:minimizing the RHS-2} based on~[29, Theorem 4]. Following~[29, Theorem 4], we assume that the loss function is in the form of Bregman divergence.  And the upper bound of $\mathbb{E}_{\bold{w}S}[G_1(\pi_{cur}, \bold{w}_{cur})]$ can be represented as:



\begin{multline}
\label{eq:prove stage I-5}
\sum\limits_{k=1}^{K}\Bigg(\frac{\alpha}{N_{cur}K^2}\sum\limits_{n=1}^{N_{cur}}\psi^{*-1}\left(I\left(\mathbf{w}_{cur},Z_{nk,cur}\right)\right)\\
+\left.\frac{1-\alpha}{N_{pre}K^2}\sum\limits_{n=N_{cur}+1}^{N}\psi^{*-1}\!\left(I(\mathbf{w}_{cur},Z_{nk,pre})\right.\right.\\
\left.+D_{\rm KL}(\pi_{pre}\|\pi_{cur})\right)\Bigg).
\end{multline}

Following the same proof idea used in proving~\eqref{eq:prove stage I-5}, we can derive the upper bound of $\mathbb{E}_{\bold{w}S}[G_2(\pi_{nxt},\bold{w}_{cur})]$ as follows:

\begin{multline}
\label{eq:prove stage II-1}
\!\!\!\!\!\sum\limits_{k=1}^{K}\!\Bigg(\!\frac{\alpha}{N_{cur}K^2}\!\sum\limits_{n=1}^{N_{cur}}\psi^{*-1}\!\left(I(\mathbf{w}_{cur},\!Z_{nk,cur})\!+\!D_{\rm KL}(\pi_{cur}\|\pi_{nxt})\right)\\
+\left.\frac{1-\alpha}{N_{pre}K^2}\sum\limits_{n=N_{cur}+1}^{N}\psi^{*-1}\left(I(\mathbf{w}_{cur},Z_{nk,pre})\right.\right.\\
\left.+D_{\rm KL}(\pi_{pre}\|\pi_{nxt})\right)\Bigg).
\end{multline}

By substituting the upper bound of $\mathbb{E}_{\bold{w}S}[G_1(\pi_{cur}, \bold{w}_{cur})]$ in~\eqref{eq:prove stage I-5} and the upper bound of $\mathbb{E}_{\bold{w}S}[G_2(\pi_{nxt},\bold{w}_{cur})]$ in~\eqref{eq:prove stage II-1} to Eq.~(6), we can derive the upper bound of $\overline{G}$ in Eq.~(12).
\end{proof}

\section{Proof of Proposition 1}\label{sec: appendix Proof of Proposition 1}
\setcounter{proposition}{0} 
\begin{proposition}[Threshold]\label{proposition:Threshold}
    There exists a threshold $p_{th} \leq 0.5$, such that $B_{r}(p)>B_{p}(p)>B_{g}(p)$ for any $p<p_{th}$.
\end{proposition}
\begin{proof}
    First, we prove that there exists a unique intersection point of $B_{p}(p)$ and $B_{g}(p)$ within $(0,1)$. \textbf{(i)} We establish the existence by showing that $\lim\limits_{p\rightarrow0^{+}}B_{p}(p)>\lim\limits_{p\rightarrow0^{+}}B_{g}(p)$ and $\lim\limits_{p\rightarrow1^{-}}B_{g}(p)>\lim\limits_{p\rightarrow1^{-}}B_{p}(p)$. \textbf{(ii)} We prove uniqueness. When $0<p<1$, $B_{p}(p)$ is monotonically decreasing. When $D>2$ and $0<p\leq 0.5$, $B_{g}(p)$ is monotonically decreasing. When $D>2$, $0.5<p<1$, $B_{g}(p)$ is monotonically increasing. When $D=2$, $B_{g}(p)$ is a constant function. Due to (i) and (ii), $B_{p}(p)$ and $B_{g}(p)$ have a unique intersection point within $(0,1)$. Let $p_{pg}$ be the x-coordinate of the intersection point of $B_{p}(p)$ and $B_{g}(p)$.
    
    Similarly, we can prove that if $B_{r}(0.5)\leq B_{p}(0.5)$, then $B_{p}(p)$ and $B_{r}(p)$ have a unique intersection point within $(0,0.5]$. Let $p_{pr}$ be the x-coordinate of the intersection point of $B_{p}(p)$ and $B_{r}(p)$.

    Second, we prove that $p_{th} = \min\{p_{pg} ,p_{pr} ,0.5\}$ by the graph of $B_{p}(p)$, $B_{g}(p)$ and $B_{r}(p)$.  $B_{p}(p)$ is a quadratic function opening upwards defined on $(0,1)$, and the axis of symmetry of $B_{p}(p)$ is larger than 1. when $D>2$, $B_{g}(p)$ is a quadratic function opening upwards defined on $(0,1)$, and the axis of symmetry of $B_{g}(p)$ is $0.5$; when $D=2$, $B_{g}(p)$ is a constant function defined on $(0,1)$. $B_{r}(p)$ is a function composed of a linear term $p$ and a reciprocal term $1/p$, defined on $(0,0.5]$. Let $B^{'}_{p}(p)$, $B^{'}_{g}(p)$, and $B^{'}_{r}(p)$ denote the derivative of $B_{p}(p)$, $B_{g}(p)$, and $B_{r}(p)$, respectively. When $0 < p \leq 0.5$, $0 \geq B^{'}_{g}(p) > B^{'}_{p}(p) > B^{'}_{r}(p)$. When $0.5<p<1$, $B^{'}_{g}(p) > 0 > B^{'}_{p}(p)$. 
    After drawing the graph of $B_{p}(p)$, $B_{g}(p)$ and $B_{r}(p)$, we can prove the following two cases. \textbf{(i)} $p_{pg}\leq 0.5$. If $B_{r}(0.5)\geq B_{p}(0.5)$, then $p_{th}=p_{pg}$; if $B_{r}(0.5) < B_{p}(0.5)$, then $p_{th}=\min\{p_{pg},p_{pr}\}$. \textbf{(ii)} $p_{pg} > 0.5$. If $B_{r}(0.5)\geq B_{p}(0.5)$, then $p_{th}=0.5$; if $B_{r}(0.5) < B_{p}(0.5)$, then $p_{th}=p_{pr}$. Thus, $B_{r}(p)>B_{p}(p)>B_{g}(p)$ for $p<p_{th}$.
\end{proof}

\begin{table}[t]
\centering
\caption{Average test accuracy for Fashion-MNIST and CIFAR-10 (random pattern).}
\begin{tabular}{lcc}
\toprule
Algorithms   & Fashion-MNIST&CIFAR-10\\ 
\midrule
ERM    &55.25\% & 71.75\% \\ 
AdaptiveFedAvg   &53.48\% & 72.79\% \\ 
FLASH   &55.20\% & 63.18\%\\ 
KLDA   & 53.21\% & 69.60\% \\ 
Ours   & \textbf{67.24}\% & \textbf{76.83}\%\\ 
\bottomrule
\end{tabular}
\label{table:random pattern Test accuracy comparison}
\end{table}

\begin{table}[t]
\centering
\caption{Average test accuracy for Fashion-MNIST and CIFAR-10 (gradual pattern).}
\begin{tabular}{lcc}
\toprule
Algorithms   & Fashion-MNIST&CIFAR-10\\ 
\midrule
ERM    &53.46\% & 64.61\% \\ 
AdaptiveFedAvg   &49.40\% & 65.37\% \\ 
FLASH   &50.84\% & 56.76\%\\ 
KLDA   & 52.06\% & 64.38\% \\ 
Ours   & \textbf{61.05}\% & \textbf{68.69}\%\\ 
\bottomrule
\end{tabular}
\label{table:gradual pattern Test accuracy comparison}
\end{table}

\begin{figure}[t]
	\centering
	\subfloat[]{
		\includegraphics[width=1\linewidth]{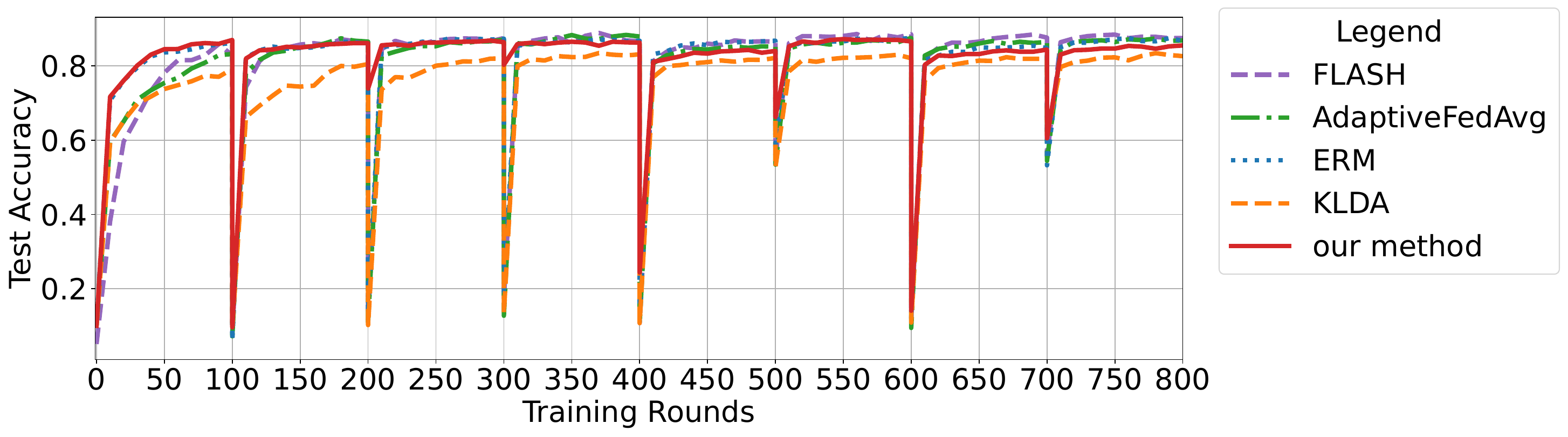}
		\label{fig:example16-randompattern}}\\
        \centering
	\subfloat[]{
		\includegraphics[width=1\linewidth]{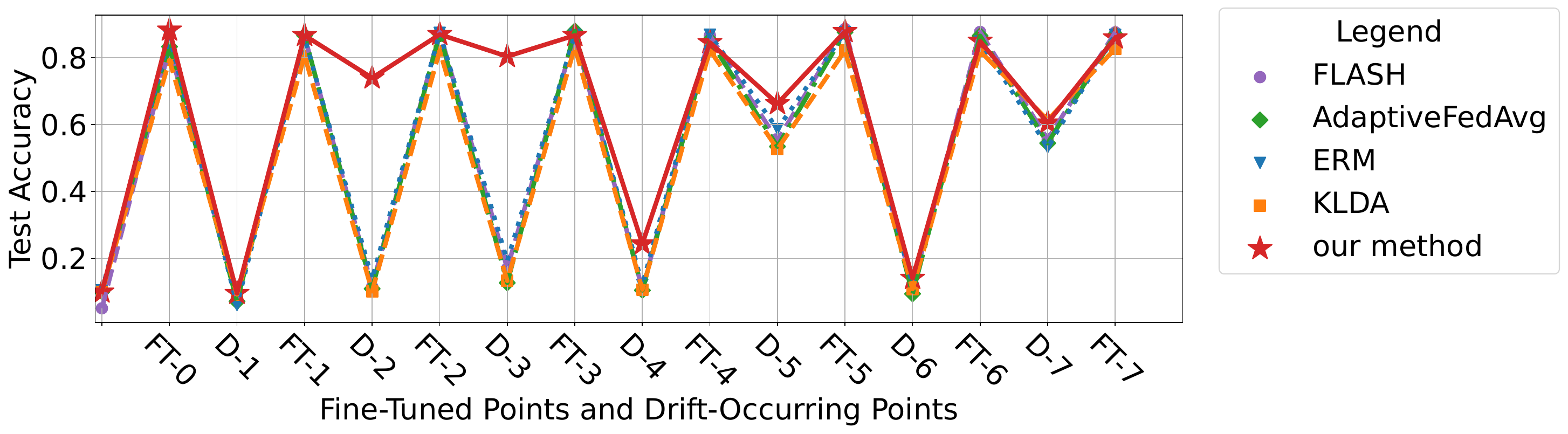}
		\label{fig:example16-1-random-pattern}}
 \caption{Accuracy curves for Fasion-MNIST dataset (random pattern): (a) Training curve; (b) test accuracy of retrain and drift.}
 \label{fig:example16-big-random-pattern}
\end{figure}

\begin{figure}[t]
	\centering
	\subfloat[]{
		\includegraphics[width=1\linewidth]{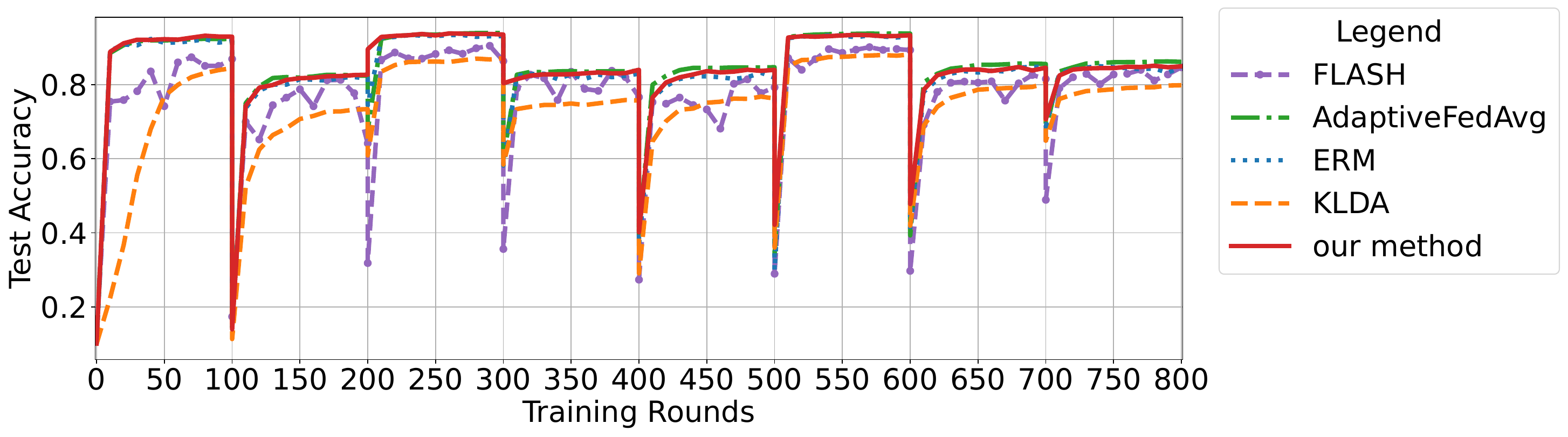}
		\label{fig:example19-randompattern}}\\
        \centering
	\subfloat[]{
		\includegraphics[width=1\linewidth]{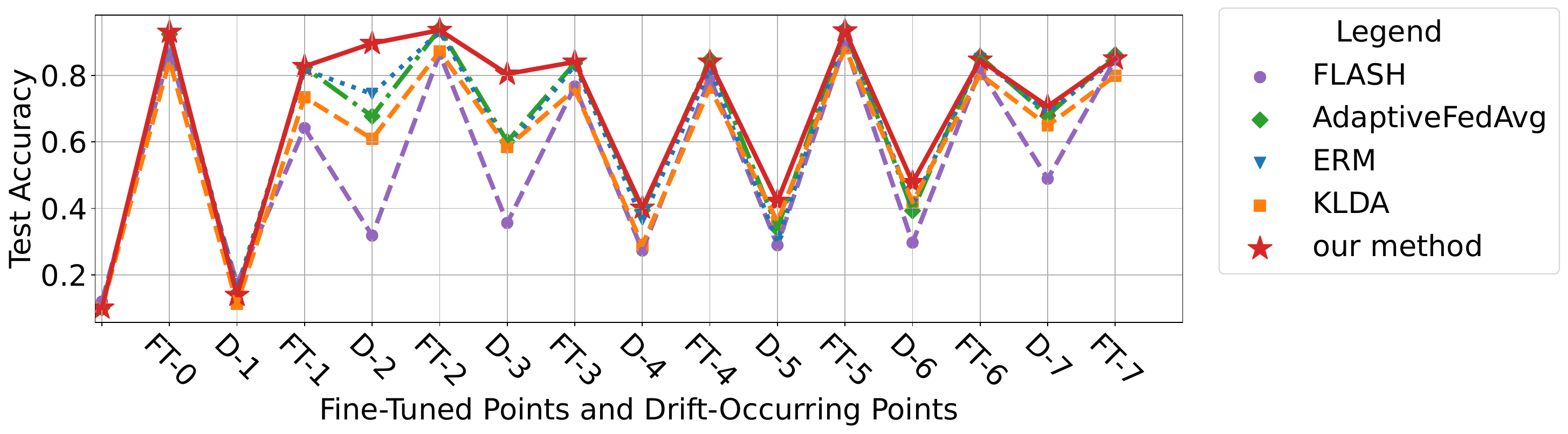}
		\label{fig:example19-1-randompattern}}
 \caption{Accuracy curves for CIFAR-10 dataset (random pattern):  (a) Training curve; (b) test accuracy of retrain and drift. }
 \label{fig:example19-big-random-pattern}
\end{figure}

\begin{figure}[t]
	\centering
	\subfloat[]{
		\includegraphics[width=1\linewidth]{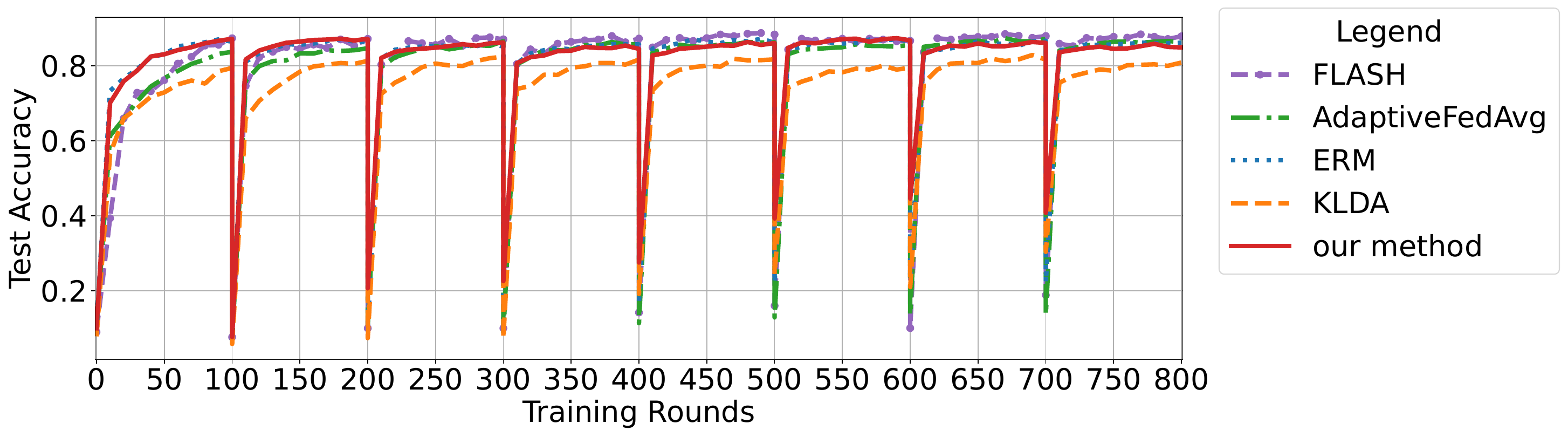}
		\label{fig:example16-gradualpattern}}\\
        \centering
	\subfloat[]{
		\includegraphics[width=1\linewidth]{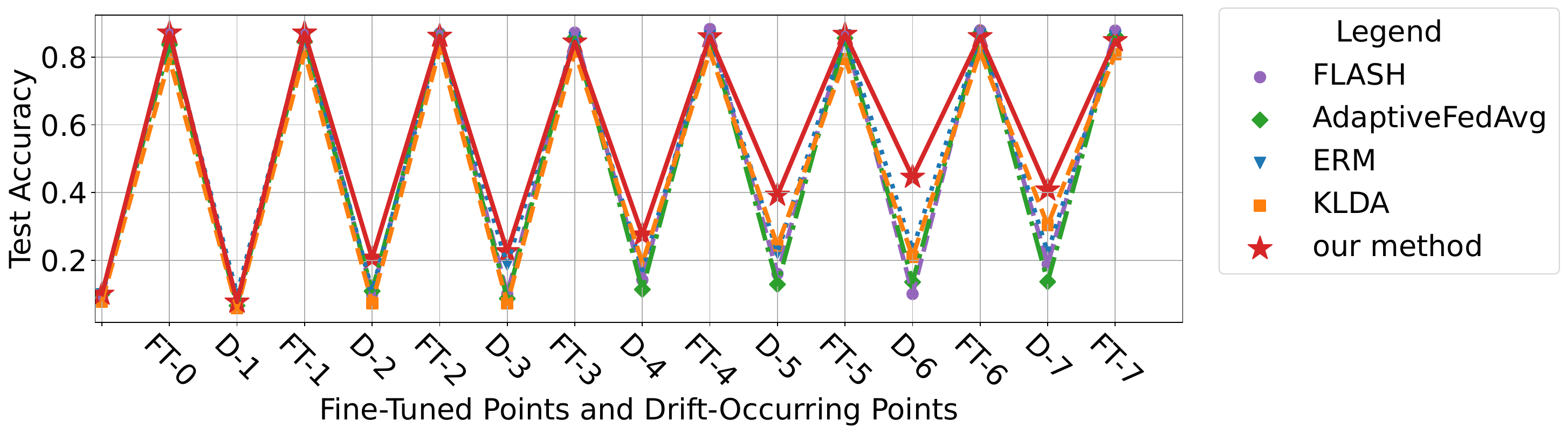}
		\label{fig:example16-1-gradual-pattern}}
 \caption{Accuracy curves for Fasion-MNIST dataset (gradual pattern): (a) Training curve; (b) test accuracy of retrain and drift.}
 \label{fig:example16-big-gradual-pattern}
\end{figure}

\begin{figure}[t]
	\centering
	\subfloat[]{
		\includegraphics[width=1\linewidth]{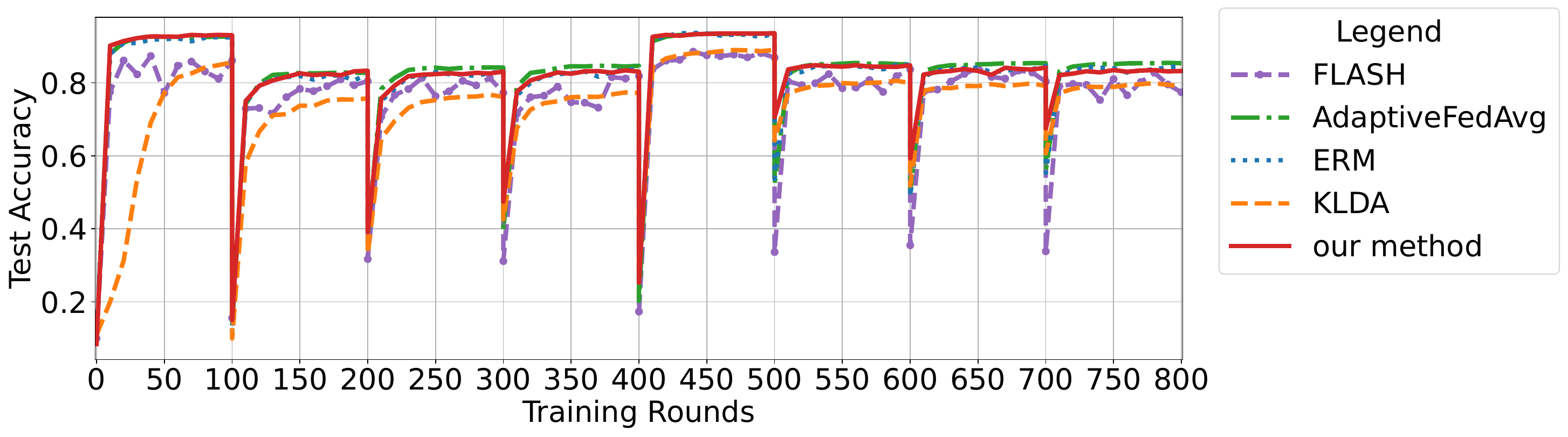}
		\label{fig:example19-randompattern}}\\
        \centering
	\subfloat[]{
		\includegraphics[width=1\linewidth]{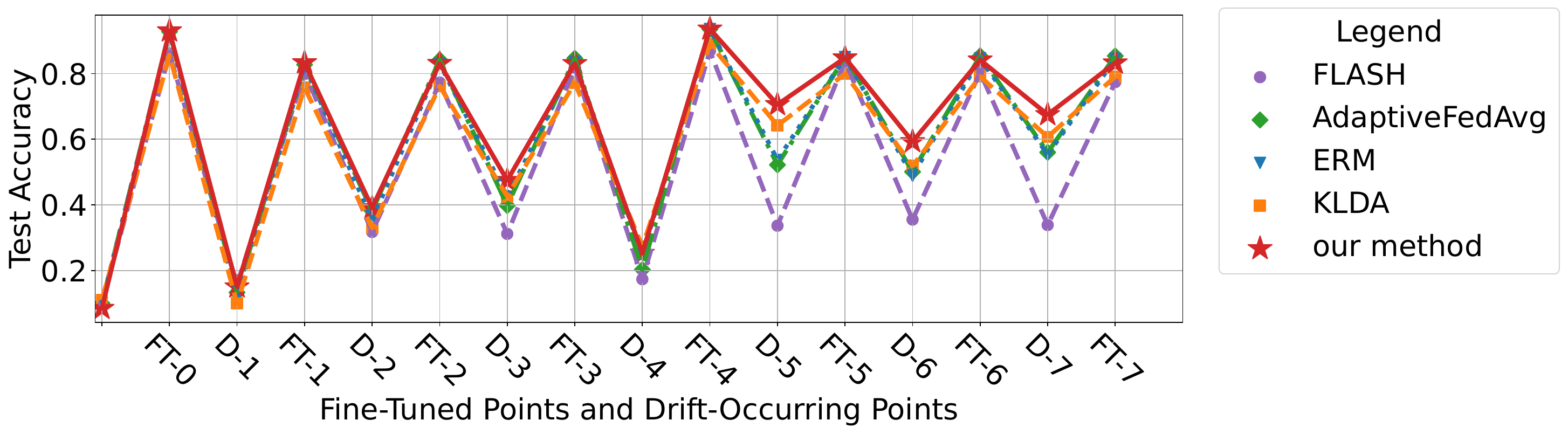}
		\label{fig:example19-1-gradualpattern}}
 \caption{Accuracy curves for CIFAR-10 dataset (gradual pattern):  (a) Training curve; (b) test accuracy of retrain and drift. }
 \label{fig:example19-big-gradual-pattern}
\end{figure}

\vspace{-5mm}

\section{Proof of Lemma 2}\label{sec:appendix Proof of Lemma 2}
\setcounter{lemma}{1} 
\begin{lemma}
    Processing capacity vector $\boldsymbol{f}^*$ is the optimal solution to problem~(24) only if $\sum_kC_k(f_k^*)= \overline{C}$.
\end{lemma}

\begin{proof}
    If $\sum_kC_k(f_k^*)< \overline{C}$, then there always exists a set of small value $\epsilon_k$ for each $k$, $\boldsymbol{\epsilon}=\{\epsilon_k,k\in\mathcal{K}\}$, such that
    $ \sum_kC_k(f_k^*+\epsilon_k)\leq \overline{C}$  and $\phi(\boldsymbol{f}^*+\boldsymbol{\epsilon})<\phi(\boldsymbol{f}^*)$. The later inequality leads to the fact that $B(R\phi(\boldsymbol{f}^*+\boldsymbol{\epsilon}))< B(R\phi(\boldsymbol{f}^*))$. 
    This implies a contradiction, so $\boldsymbol{f}^*$ cannot be the optimal solution to problem~(24).
\end{proof}

\section{Proof of Lemma 3} \label{sec:appendix Proof of Lemma 3}
\begin{lemma}
    Let $\phi_k(f_k)=\frac{S_kD_kE}{f_k}+T^{UL}_k+T^{DL}_k$. Processing capacity vector $\boldsymbol{f}^*$ is the optimal solution only if $\phi_k(f_k^*)$ are identical for $k\in\mathcal{K}$. 
\end{lemma}

\begin{proof}
    Suppose $\boldsymbol{f}^*$ is the optimal solution and $\phi_k(f_k^*)$ are not identical for $k \in \mathcal{K}$. Then, there exists an $k^-$ such that $\phi_{k^-}(f_{k^-}^*)$ is the smallest, and another $k^+$ such that $\phi_{k^+}(f_{k^+}^*)$ is the largest.
    Let $f_k^{'}<f_k^*$, $\phi_k(f_k^{'})<\phi_{k^+}(f_{k^+}^*)$ and $\boldsymbol{f}^{'}=(f_1^*,...,f_{k-1}^*,f_k^{'},f_{k+1}^*,...,f_K^*)$. We can show that $\phi(\boldsymbol{f}^{'})=\phi(\boldsymbol{f}^*)$, according to Eq.~(21). Thus, $B(R\phi(\boldsymbol{f}^{'}))=B(R\phi(\boldsymbol{f}^*))$. Meanwhile, since $f_k^{'}<f_k^*$ and $C_k(f_k^{'})<C_k(f_k^*)$, $C(\boldsymbol{f}^{'})<C(\boldsymbol{f}^*)\leq \overline{C}$.
    According to Lemma~2, $\boldsymbol{f}^{'}$ cannot be the optimal solution to problem~(24). Further, since $\boldsymbol{f}^{'}$ and $\boldsymbol{f}^*$ lead to the same objective value, i.e.,  $B(R\phi(\boldsymbol{f}^{'}))=B(R\phi(\boldsymbol{f}^*))$, $\boldsymbol{f}^*$ cannot be the optimal solution. This leads to a contradiction. Therefore, Lemma~3 is true.   
\end{proof}

\section{Proof of Proposition 3} \label{sec:appendix Proof of Proposition 3}
\setcounter{proposition}{2}  
\begin{proposition}\label{proposition:final solution}
If $\tau_2^*$ is the optimal solution to problem~(25), then $f_k^*\triangleq\phi_k^{-1}(\frac{\tau_2^*}{R})$ for $k\in\mathcal{K}$ optimizes problem~(24). 
\end{proposition}
\begin{proof}
    Suppose $\tau_2^*$ is the optimal solution to~(25), but $f_k^*\triangleq\phi_k^{-1}(\frac{\tau_2^*}{R})$ is not the optimal solution to~(24), where  $\phi_k^{-1}(\cdot)$ is the inverse function of $\phi_k(\cdot)$.  We aim to show contradiction under the two possible cases. 

    Case 1. There exists a $k$ such that $f_k^*$ is not a feasible solution to problem~(24). However, this cannot be the case, because $\sum_kC_k(\phi_k^{-1}(\frac{\tau_2^*}{R}))= \overline{C}$. That is, $\sum_kC_k(f_k^*) = \overline{C}$ satisfies the constraint.

    Case 2. There exists a $k$ such that  $f_k^*$ is feasible but cannot minimize the objective function. Then, there exists another $\boldsymbol{f}^{'}$ such that $B(R\phi(\boldsymbol{f}^{'}))<B(R\phi(\boldsymbol{f}^{*}))$. And $\sum_kC_k(f_k^{'}) = \overline{C}$ according to Lemma~2.
    Thus, there exists a $\tau_2^{'}=R\phi(\boldsymbol{f}^{'})$ such that $B(\tau_2^{'})<B(\tau_2^{*})$ and $\sum_kC_k(\phi_k^{-1}(\frac{\tau_2^{'}}{R}))= \overline{C}$. Thus, $\tau_2^*$ cannot be the optimal solution. This leads to a contradiction. Therefore, Proposition~\ref{proposition:final solution} is true.
\end{proof}

\section{Detailed Experimental Settings}\label{sec:appendix Detailed Experimental Settings}
The training settings for various methods are as follows:

\textbf{ERM:} learning rate: $10^{-4}$,  batch size: 128, representation dimension: 16.

\textbf{FLASH:} initial learning rate: $10^{-4}$,  batch size: 128, representation dimension: 16.

\textbf{AdaptiveFedAvg:} initial learning rate: $10^{-4}$, batch size: 128, representation dimension: 16.

\textbf{KLDA:} learning rate: $10^{-4}$,  batch size: 128, representation dimension: 16, $\beta$: $0.3$, $\beta_{aug}$: $0.1$.

\textbf{Ours:} learning rate: $10^{-4}$,  batch size: 128, representation dimension: 16,  \(\gamma\): 0.001.

\section{Additional Experiment Results}\label{sec:appendix Additional Experiment Results}

\subsection{Performance Variation with Distance Between Distributions}

\begin{figure}[t]
	\centering
	\subfloat[CIRCLE]{
		\includegraphics[width=0.8\linewidth]{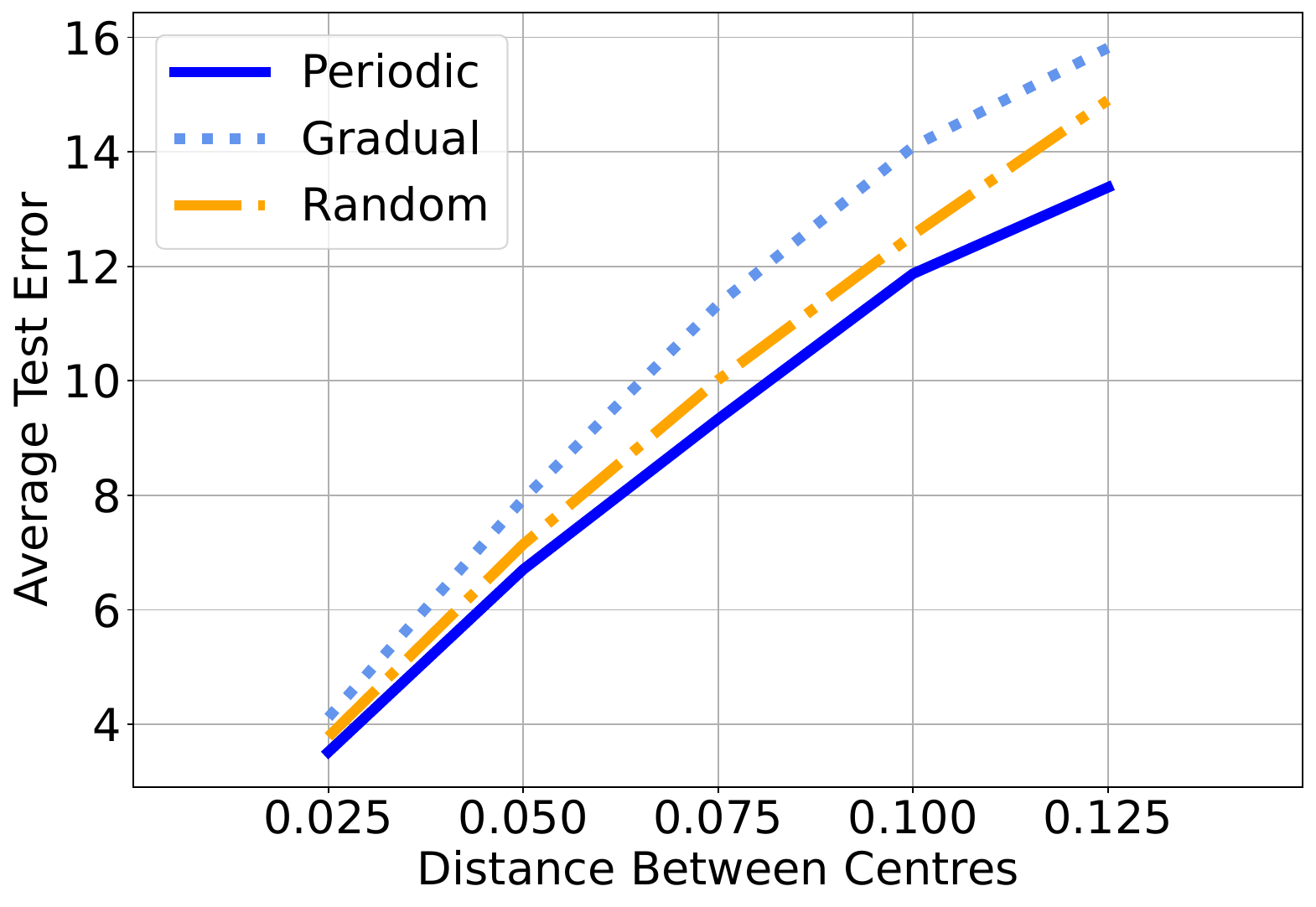}
		\label{fig:circle_distance}}
	\\
	\subfloat[MNIST]{
		\includegraphics[width=0.8\linewidth]{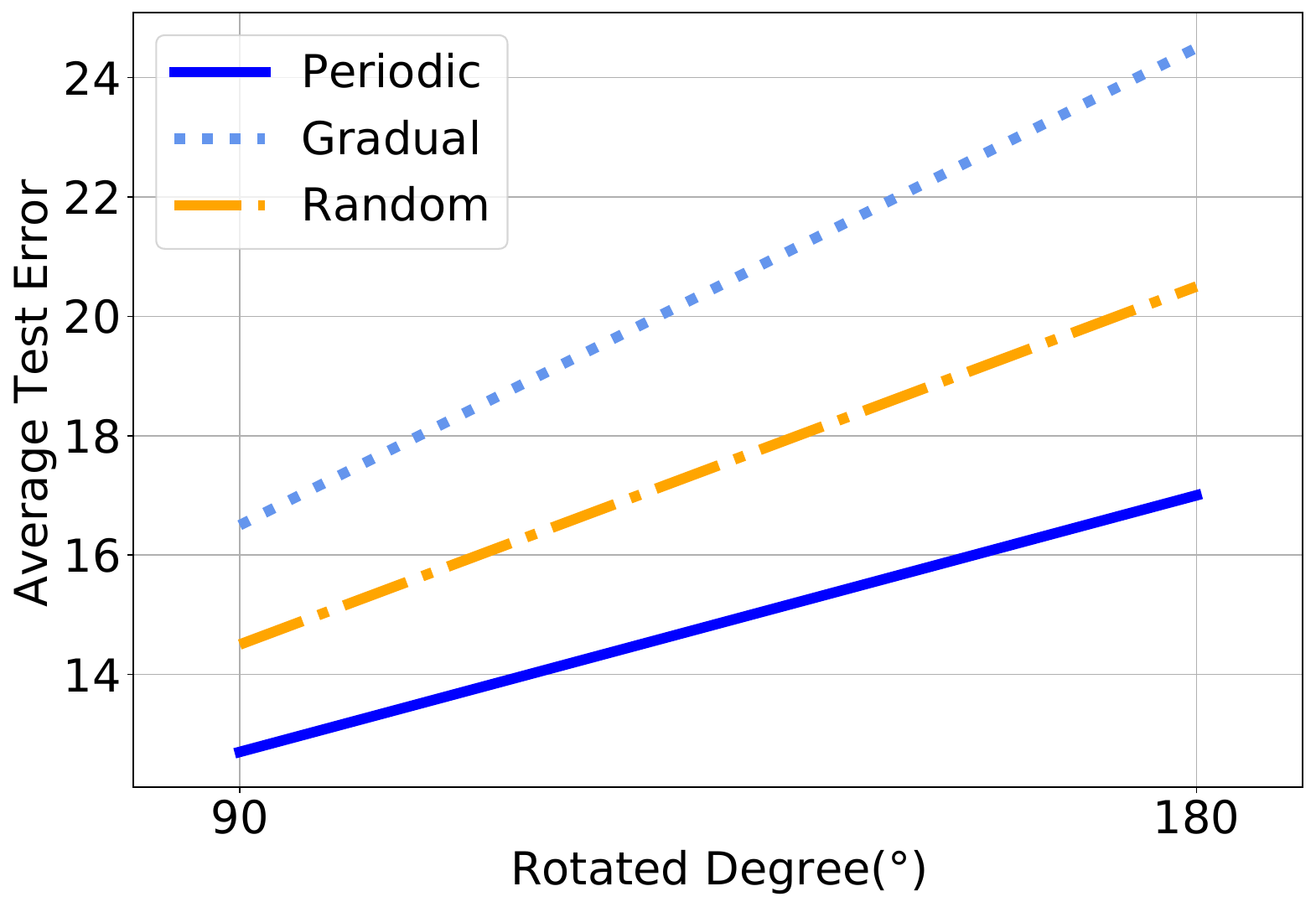}
		\label{fig:mnist_distance}}
	\\
	\subfloat[CIFAR10]{
		\includegraphics[width=0.8\linewidth]{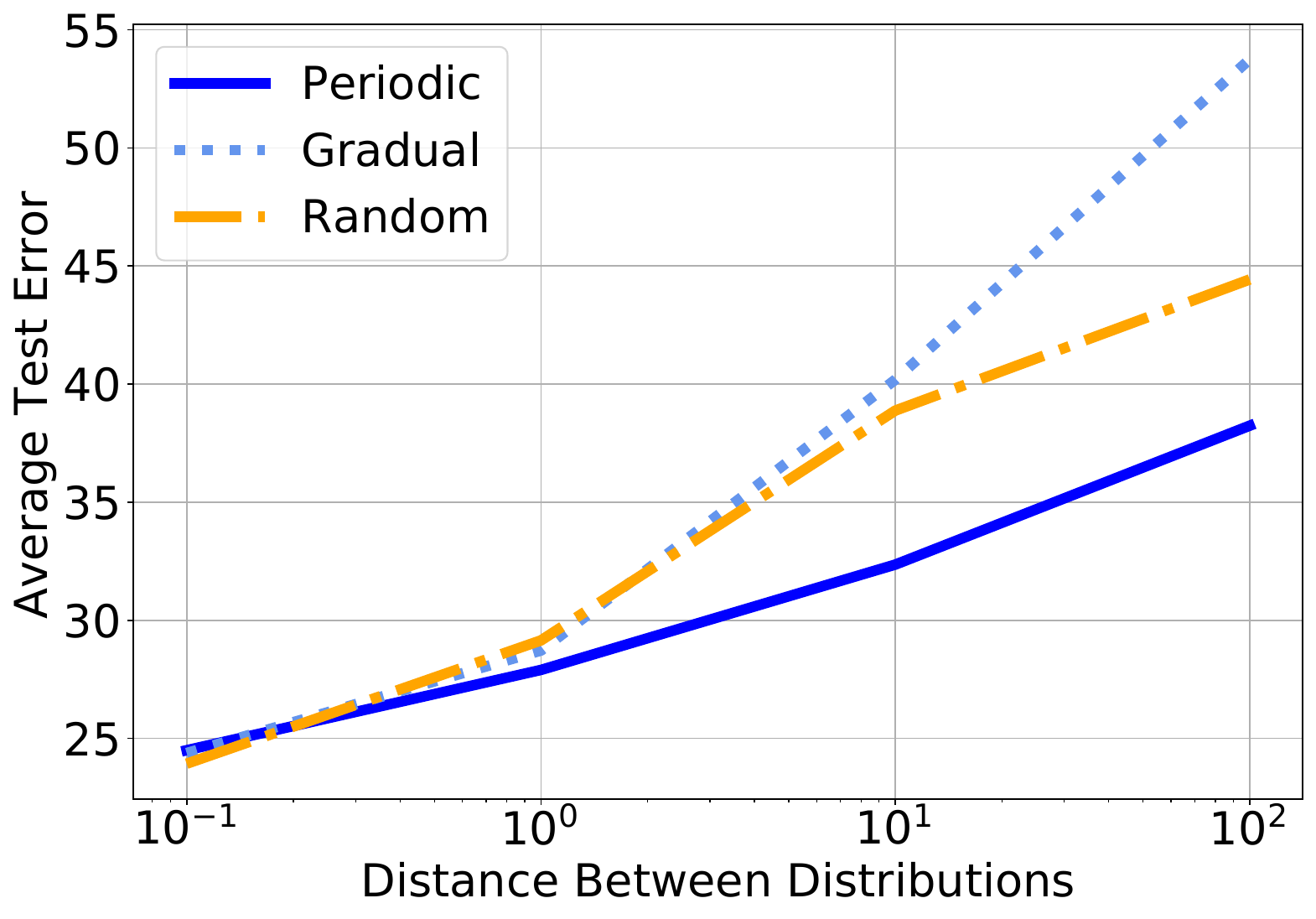}
		\label{fig:cifar_distance}}
	\caption{Average test error changes over the distance between distributions.}
	\label{fig:Average test error changes over the distance between distributions.}
\end{figure}

In Fig.~\ref{fig:Average test error changes over the distance between distributions.}, we present the average test error changes over distance between distributions under three patterns. In this case, let clients have non-IID dataset. Recall that under different datasets, the concept drift was generated with different approaches. Thus, in Figs.~\ref{fig:circle_distance},~\ref{fig:mnist_distance}, and~\ref{fig:cifar_distance}, the x-axis represents different physical meanings, i.e., the distance between the centers of the circles, the rotated degree for MNIST images, and the distance between data distributions (i.e. the reciprocal of the dirichlet distribution concentration parameter), respectively.  
\setcounter{observation}{7}
\begin{observation}[Impact of Distance]\label{observation:Error Grows with Distance}
    The average test error grows as the distance between data distributions increases.
\end{observation}
Observation~\ref{observation:Error Grows with Distance} verifies our analytical results on $\overline{G}$ in Theorem 1, i.e., when the difference between recent and future unseen datasets is larger, the performance degradation is more significant.
\begin{observation}[Performance Difference]\label{observation: The Difference of Patterns' Error Varies with Distance}
As the distance between data distributions increases, the performance gap (i.e.,  the difference of the average test error) between different patterns increases.
\end{observation}
Observation~\ref{observation: The Difference of Patterns' Error Varies with Distance} suggests that when the difference between recent data and future unseen data is more significant, it is more  necessary to understand the on-going concept drift pattern and to design the retraining or  adaptation algorithm according to the pattern.

\subsection{Algorithm Comparison}

Table~\ref{table:random pattern Test accuracy comparison} and Table~\ref{table:gradual pattern Test accuracy comparison} present the results for the random and gradual patterns in the Fashion-MNIST and CIFAR-10 experiments. Each experiment for each pattern consists of a total of 1200 training rounds. 

Fig.~\ref{fig:example16-big-random-pattern} and Fig.~\ref{fig:example19-big-random-pattern} illustrate the changes in test accuracy for Fashion-MNIST and CIFAR-10 during training under the random pattern, respectively. Fig.~\ref{fig:example16-big-gradual-pattern} and Fig.~\ref{fig:example19-big-gradual-pattern} display the changes in test accuracy for Fashion-MNIST and CIFAR-10 during training under the gradual pattern, respectively.

\end{document}